\documentclass{article}

\usepackage[final,main]{neurips_2026}

\usepackage[utf8]{inputenc} % allow utf-8 input
\usepackage[T1]{fontenc}    % use 8-bit T1 fonts
\usepackage{hyperref}       % hyperlinks
\usepackage{url}            % simple URL typesetting
\usepackage{booktabs}       % professional-quality tables
\usepackage{amsfonts}       % blackboard math symbols
\usepackage{nicefrac}       % compact symbols for 1/2, etc.
\usepackage{microtype}      % microtypography
\usepackage[dvipsnames]{xcolor} % colors (with dvipsnames for AutoML definitions)

\usepackage{amsmath}
\usepackage{amsthm}
\usepackage{thmtools}      % <-- new
\usepackage{thm-restate}   % <-- new
\usepackage{mathtools}
\usepackage{bm}

\usepackage{algorithm}
\usepackage[noend]{algcompatible}
\usepackage{verbatim}

\usepackage{graphicx}
\usepackage{wrapfig}
\usepackage{subcaption}
\usepackage{csquotes}
\usepackage{pifont} % for cmark/xmark

\usepackage{tikz}
\usepackage{pgfplots}
\usetikzlibrary{shapes.geometric, arrows, arrows.meta, positioning, colorbrewer, patterns, decorations.pathreplacing, fit, backgrounds, shapes, calc, shadows}

\tikzstyle{action} = [rectangle, rounded corners, minimum width=3cm, minimum height=1cm,text centered, draw=sta-orange, fill=sta-orange!30]
\tikzstyle{instance} = [rectangle, rounded corners, minimum width=3cm, minimum height=1.5cm,text centered, draw=sta-purple, fill=sta-purple!30]
\tikzstyle{arrow} = [thick,->,>=stealth, auto, line width=0.5mm]
\tikzstyle{line} = [draw, -latex']

\newcommand{\SumNode}{\mathsf{S}}
\newcommand{\ProductNode}{\mathsf{P}}
\newcommand{\Node}{\mathsf{N}}
\newcommand{\Leaf}{\mathsf{L}}
\newcommand{\graph}{\mathcal{G}}
\newcommand{\ch}[1]{\operatorname{ch}(#1)}

\newcommand{\solidline}[1]{\protect \tikz[baseline=-0.5ex] \protect \draw[very thick, #1] (0,0) -- (0.4,0);}
\newcommand{\dashedline}[1]{\protect \tikz[baseline=-0.5ex] \protect \draw[very thick, dashed, #1] (0,0) -- (0.4,0);}

\newcommand{\dashdottedline}[1]{\protect \tikz[baseline=-0.5ex] \protect \draw[very thick, dash pattern=on 4pt off 2pt on 1pt off 2pt, #1] (0,0) -- (0.5,0);}

\definecolor{grey}{HTML}{424241}
\definecolor{boTPE}{HTML}{56048C}
\definecolor{boRF}{HTML}{AB0392}
\definecolor{DisCoMBO}{HTML}{931147}
\definecolor{DisCoMBOearly}{HTML}{B6175A}
\definecolor{DisCoMBOmany}{HTML}{E2AEC8}
\definecolor{DisCoMBOrecover}{HTML}{D687AC}
\definecolor{IBOearly}{HTML}{0D47A1}
\definecolor{IBO}{HTML}{0B3D89}
\definecolor{dynabo}{HTML}{16501B}
\definecolor{PiBOcol}{HTML}{014d08}
\definecolor{GPintro}{HTML}{474747}
\definecolor{mygray}{RGB}{190,190,190}

\newtheorem{defin}{Definition}
\newtheorem{prop}{Proposition}

\newtheorem{lem}{Lemma}

\title{DisCoMBO: Steering Expert-in-the-Loop Black-Box Optimization via Distributional Conformance}

\author{%
  Jonas Seng\thanks{Equal contribution.} \\
  Computer Science Department\\
  TU Darmstadt \\
  \texttt{jonas.seng@tu-darmstadt.de} \\
  \And
  Bennet Wittelsbach\footnotemark[1] \\
  Computer Science Department\\
  TU Darmstadt \\
  \And
  Kristian Kersting \\
  TU Darmstadt, hessian.AI, DFKI \\
}

\hypersetup{
    pdftitle={DisCoMBO: Steering Expert-in-the-Loop Black-Box Optimization via Distributional Conformance},
    pdfauthor={Jonas Seng, Bennet Wittelsbach, Kristian Kersting}, % Für den Review-Prozess anonym belassen
    pdfkeywords={Probabilistic Methods, Optimization, Human-in-the-loop},
}

\begin{document}

\maketitle

\begin{abstract}
Sequential Model-Based Optimization (SMBO) traditionally relies on Bayesian or ensembling surrogates for uncertainty quantification. While historically treated as fully data-driven, SMBO increasingly integrates external domain expertise to accelerate discovery. To overcome the opaque guidance and diminished integration fidelity of standard acquisition re-weighting, Probabilistic Circuits (PCs) have emerged as a generative surrogate alternative, enabling direct knowledge injection via conditional sampling. However, these generative routines lack the formal exploration-exploitation semantics required for rigorous optimization. We introduce the \textbf{Dis}tributional \textbf{Co}nformance Score (DisCo), a novel metric that unifies the flexibility and efficiency of PCs with a formal uncertainty framework. DisCo provides a bounded, $[0, 1]$-normalized measure of model "surprise" that (1) recovers properties comparable to kernel-based uncertainty known from, e.g., Gaussian Processes, while maintaining linear-time inference, and (2) enables accurate assessment of conformance of external knowledge w.r.t. model evidence. We then present DisCoMBO, a framework leveraging these properties for robust, knowledge-aware optimization. We prove that DisCoMBO is a zero-regret algorithm and demonstrate its effectiveness across diverse benchmarks from AutoML, material optimization, and wind park optimization.
\end{abstract}

\section{Introduction} \label{sec:Intro}
Black-Box Optimization (BBO) is a cornerstone of modern scientific discovery and serves as an optimization engine throughout various disciplines and industries. Examples are material discovery~\citep{kusne2020bomatsci}, wind park layout optimization~\citep{ju2019windoptim}, and hyperparameter tuning in machine learning~\citep{schede_2022}. Given the high cost of evaluating these objective functions, Sequential Model-Based Optimization (SMBO) has emerged as the dominant paradigm. By alternating between learning a surrogate model and acquiring the next candidate configuration, SMBO seeks to find the global optimum while navigating the search space with maximum sample efficiency.

\begin{figure*}[t]
\centering
    \begin{subfigure}[c]{0.48\textwidth}
        \centering
        \includegraphics[width=\textwidth]{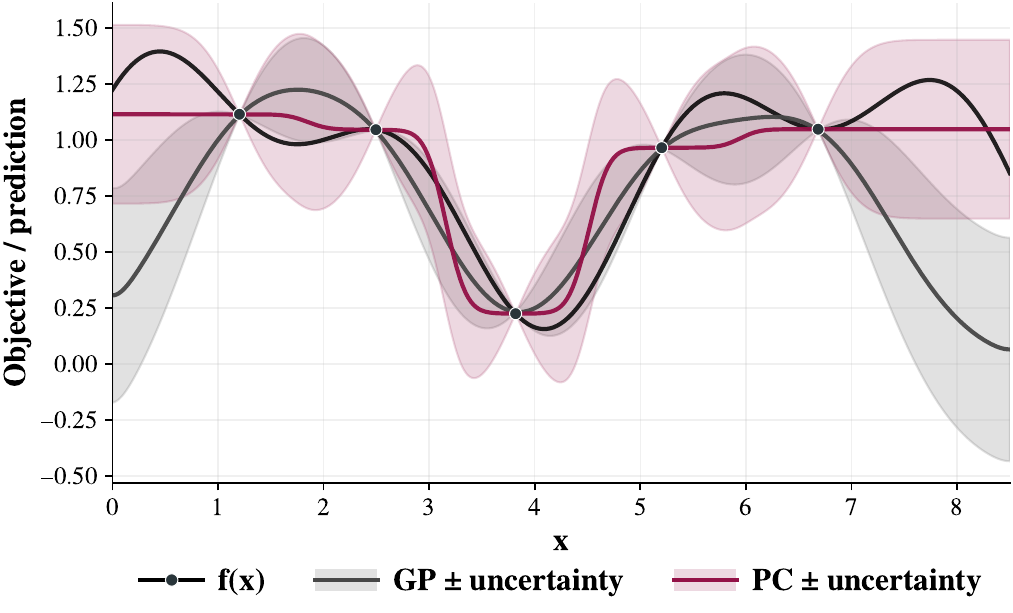}
        \caption{}
    \end{subfigure}%
    \hfill
    \begin{subfigure}[c]{0.48\textwidth}
        \centering
        \includegraphics[width=\textwidth]{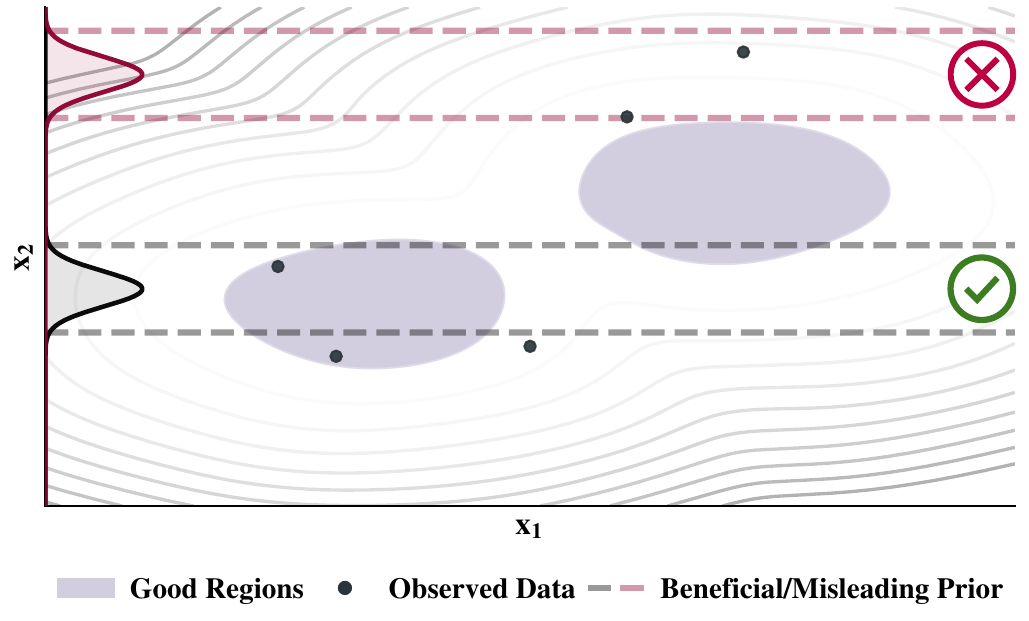}
        \caption{}
    \end{subfigure}
    \caption{\textbf{DisCo provides uncertainty estimates for PC-based SMBO and allows for accurate conformance assessment of external knowledge -- both with linear inference time.}  
    (a) Given a learned PC from observations, DisCo (\solidline{DisCoMBO}) provides uncertainty estimates with \textbf{linear inference cost} based on how well a given point $\mathbf{x}$ conforms with the distribution represented by the PC, resulting in uncertainty quantification akin to Gaussian Processes (GPs; \solidline{GPintro}).
    (b) DisCo can assess whether externally provided knowledge conforms to the PC's expectation for improving upon $y^*$, serving as a scale-independent basis for acceptance or rejection of given knowledge.}
    \label{fig:disco_concept}
\end{figure*}
Recently, different fields %disciplines 
called for enabling the integration of external knowledge, e.g., from humans or Large Language Models (LLMs). Examples of such paradigms include human-centered AutoML~\citep{Wang2019ATMSeer,Park2021HyperTendril,LindauerPosition24} and optimization in the physical sciences~\citep{kusne2020bomatsci} to increase sample efficiency, build user trust, and prevent optimizers from exploring implausible solutions. While traditional frameworks like $\pi$BO~\citep{hvarfner2022pibo} and DynaBO~\citep{fehring2026dynabo} integrate knowledge in the form of priors via acquisition re-weighting, this indirect mediation can result in "blurred`` guidance. \citet{seng2025ihpo} introduced Probabilistic Circuits (PCs) as a novel surrogate in the SMBO landscape, enabling highly targeted knowledge integration by conditioning candidate generation on externally provided priors, thereby circumventing opaque reweighting logic. However, two fundamental challenges hinder the broader adoption of PCs in (expert-guided) SMBO: (1) a lack of native epistemic uncertainty estimates, which forces a reliance on pure likelihood sampling to obtain new candidates and precludes formal uncertainty-guided exploration; and (2) the absence of a robust and accurate metric for assessing whether external guidance contradicts model evidence, acting as a safeguard against misleading priors. Existing methods for assessing conformance of priors w.r.t. model evidence either rely solely on heuristics, such as relative acquisition changes~\citep{fehring2026dynabo}, or introduce structural complexity via auxiliary preference models~\citep{xu2024principled}. To address the limitations of PC surrogates and existing safeguard mechanisms, we introduce the \textbf{Dis}tributional \textbf{Co}nformance Score (DisCo). While fundamentally a metric for evaluating the conformance of a sample against a learned circuit distribution, DisCo naturally serves as an uncertainty proxy under the premise that low conformance directly implies high epistemic uncertainty. We theoretically demonstrate that under specific PC parameterizations, this conformance-based proxy is rank-preserving w.r.t. kernel-like uncertainty estimates, maintaining the utility of traditional exploration-exploitation strategies while preserving the linear-time inference characteristic of PCs (see Fig. \ref{fig:disco_concept}).

Building on this foundation, we present DisCoMBO (\textbf{DisCo}-informed \textbf{M}odel-\textbf{B}ased \textbf{O}ptimization), a framework that operationalizes DisCo to unify the tractability and mixed-variable flexibility of PCs with rigorous decision-theoretic semantics. By leveraging this conformance score, DisCoMBO simultaneously establishes the robust exploration mechanism required for BBO and provides a scale-independent safeguard to detect external priors that contradict empirical model evidence. The resulting procedure bridges the gap between flexible generative modeling and traditional optimization rigor, all without sacrificing the targeted integration capabilities of the generative paradigm.

\textbf{Contributions:} \textbf{(1)} The Distributional Conformance (DisCo) score, a novel and mathematically grounded score providing an uncertainty proxy for PCs at linear inference cost. 
\textbf{(2)} The DisCoMBO framework, a novel SMBO algorithm that utilizes DisCo to enable principled exploration-exploitation trade-offs. The framework facilitates targeted knowledge integration via generative sampling and provides a robust signal for assessing the conformance of external guidance with empirical data. Theoretical analysis establishes DisCoMBO as a zero-regret algorithm.
\textbf{(3)} Extensive empirical evaluation across AutoML, material discovery, and wind park optimization benchmarks, demonstrating that DisCoMBO significantly outperforms fully generative PC-based SMBO and achieves performance competitive with state-of-the-art BO baselines in both unassisted and informed settings.

\section{DisCoMBO: DisCo-guided Model-Based Optimization}
To bridge the gap between purely generative Probabilistic Circuit (PC) driven SMBO and a rigorous uncertainty-guided exploration-exploitation trade-off, we first introduce the Distributional Conformance Score (DisCo) to measure model surprise in PCs. Building on this, we then present DisCoMBO, demonstrating how DisCo enables principled exploration-exploitation trade-offs and robust, bounded, and scale-invariant assessment of externally provided knowledge.

\begin{figure}
    \centering
    \includegraphics[width=1.0\linewidth]{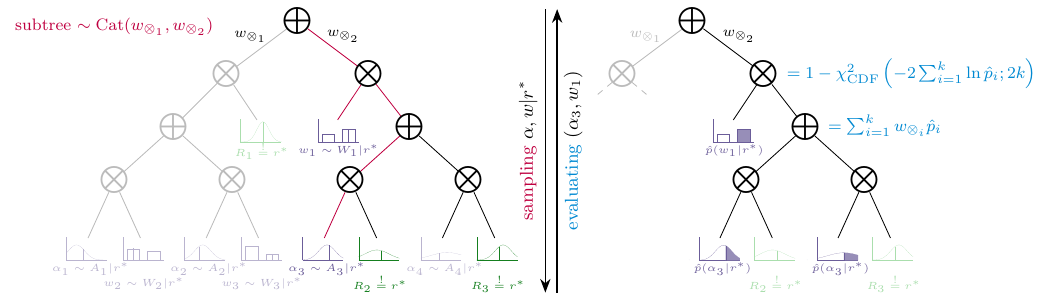}
    \caption{\textbf{DisCoMBO performs DisCo-guided candidate generation.} To generate new candidate configurations, DisCoMBO samples a large set of configurations from the conditional $p(\mathbf{x} | y^*)$ \textit{(Left)} along with random candidates from the search space. These configurations are considered potentially promising and are scored using DisCo in a single bottom-up pass \textit{(Right)}. Resulting DisCo scores are used to guide exploration and exploitation as well as to rate conformance of potentially given priors.}
    \label{fig:mchpo_sampling}
\end{figure}

\subsection{Probabilistic Circuits (PCs)}
Since DisCo is defined based on PCs, let us first %formally 
introduce PCs following~\citet{choi_2020}.
A PC is a computational graph encoding a distribution over a set of random variables $\bm{\mathcal{X}}$. It is defined as a tuple $(\graph, \phi)$ where $\graph = (V, E)$ is a rooted, directed acyclic graph and $\phi: V \rightarrow 2^{\bm{\mathcal{X}}}$ is the \textit{scope} function assigning a subset of random variables to each node in $\graph$. For each internal node $\Node$ of $\graph$, the scope is defined as the union of scopes of its children, i.e., $\phi(\Node) = \cup_{\Node' \in \ch{\Node}} \phi(\Node')$. Each leaf node $\Leaf$ computes a distribution/density over its scope $\phi(\Leaf)$. Internal nodes of $\graph$ are either a (weighted) sum node $\SumNode$ or a product node $\ProductNode$ where each sum node computes a convex combination of its children, i.e., $\SumNode = \sum_{\Node \in \ch{\SumNode}} w_{\SumNode, \Node}\Node$, and each product node computes a product of its children, i.e., $\ProductNode = \prod_{\Node \in \ch{\ProductNode}}\Node$~\citep{choi_2020}.
In this work, we will consider \textit{decomposable} and \textit{smooth} PCs. Decomposability means that for product node $\ProductNode$, the scopes of its children are disjoint: $\forall \Node_i, \Node_j \in \ch{\ProductNode}, i \neq j: \phi(\Node_i) \cap \phi(\Node_j) = \emptyset$ whereas smoothness means that for every sum node $\SumNode$, its children share the same scope: $\forall \Node \in \ch{\SumNode}: \phi(\Node) = \phi(\SumNode)$. Both are required to ensure tractable marginalization and conditioning in PCs~\citep{choi_2020}. See Fig. \ref{fig:mchpo_sampling} (left) for an illustration and App. \ref{app:PCs} for more details.
 
\subsection{Distributional Conformance Score (DisCo)}\label{sec:disco}
In their vanilla form, PCs lack native epistemic uncertainty estimates, which are essential for well-defined exploration-exploitation trade-offs. A common heuristic is to treat the likelihood $p(\mathbf{x})$ as a proxy for confidence; however, $p(\mathbf{x})$ is an unreliable measure of model "surprise." Because the magnitude of $p(\mathbf{x})$ is determined by leaf parameterization and the dimensionality of $\bm{\mathcal{X}}$, a low density does not necessarily imply high uncertainty, but may simply reflect a broad, well-understood distribution. To address this, we introduce the Distributional Conformance Score (DisCo). 

%Inspired by the logic of 
Akin to $p$-values, DisCo aggregates the cumulative probability of events that are "more extreme" than the current realization $\mathbf{x}$ under the model $p$, allowing to quantify the degree to which $\textbf{x}$ is within the modeled distribution. Formally, DisCo provides a normalized metric in $[0, 1]$ where values near $0$ indicate high surprise and values near $1$ indicate that the model is well-conformed to the observation.
A main advantage of DisCo is its computational efficiency; given a PC $p$ and a realization $\mathbf{x}$, the score is computed in a single bottom-up pass, guaranteeing inference time linear in the circuit size. We define the computation recursively across the graph structure:

\textbf{Leaf Nodes.} For a leaf node $\Leaf$ with scope $s = \phi(\Leaf)$, the localized conformance $\rho_{\Leaf}$ is defined by the two-tailed probability of the observation under the leaf's marginal distribution:\begin{equation}\rho_{\Leaf} = 2 \cdot \min\big(F_{\Leaf}(\mathbf{x}_s), 1 - F_{\Leaf}(\mathbf{x}_s)\big)
\end{equation}
where $F_{\Leaf}$ denotes the cumulative distribution function (CDF) of the leaf. Note that this ensures that $\rho_{\Leaf}$ bounded in $[0, 1]$.

\textbf{Sum Nodes.} Sum nodes maintain their latent variable interpretation as mixture models. To compute the mixture's conformance, we compute the convex combination of the scores of its children:\begin{equation}\rho_{\SumNode} = \sum\nolimits_{\Node \in \ch{\SumNode}} w_{\SumNode, \Node} \cdot \rho_{\Node}\end{equation}
Since $\sum_{\Node \in \ch{\SumNode}} = 1$ by the definition of PCs, the score is kept bounded in an interval $[0, 1]$.

\textbf{Product Nodes.} For a product node $\ProductNode$ with $k$ children, decomposability ensures that each child represents an independent random variable. To aggregate evidence across these independent dimensions, we leverage Fisher’s method~\citep{fisher1932statisticalmethods}, which satisfies three key desiderata for our setting: (1) values are bounded in $[0, 1]$, (2) surprise in a single dimension reliably propagates to the aggregated score, and (3) it scales linearly in time (see App.~\ref{app:fisher_discussion} for further discussion). Although we do not perform formal hypothesis testing or compute statistically valid $p$-values, we exploit the property that the sum of log-transformed independent tail probabilities follows a $\chi^2$ distribution with $2k$ degrees of freedom to construct our proxy score for "model surprise":
\begin{equation}
    \rho_{\ProductNode} = 1 - F_{\chi^2} \Big( - 2 \cdot \sum\nolimits_{i=1}^k \ln(\rho_i); 2k\Big).
\end{equation}
Here, $F_{\chi^2}(\cdot; 2k)$ is the $\chi^2$ CDF. Since any PC can be expanded into its induced tree representation~\citep{zhaoa16CollapsedVarInf}, the global DisCo score $\rho_p$ of a PC $p$ can be expressed in closed form:
\begin{equation}\label{eq:induced_tree}
\rho_p(\mathbf{x}) = \sum\nolimits_{t=1}^{\tau_s} \prod\nolimits_{(k, j) \in \mathcal{T}_{t}} w_{k j} \cdot \Big(1 - F_{\chi^2} \big( - 2 \cdot \sum\nolimits_{i=1}^n \ln(\rho_i); 2n\big)\Big).
\end{equation}
Here, $\tau_s$ is the number of induced trees, $\mathcal{T}_t$ is the $t$-th unique induced tree of $p$ and $\rho_i$ is the DisCo score of leaf $\Leaf_i$ in $\mathcal{T}_t$. Each tree $\mathcal{T}_t$ contains the root of the PC, for each sum node it contains exactly one child and its corresponding edge, and for each product $\mathcal{T}_t$ contains all children. This admits a mixture representation of PCs with depth one (see App. \ref{app:PCs} for more details). 

In the following proposition, we use this representation to show that DisCo can be mapped to a behavior similar to kernel-based uncertainty estimates under some circumstances.

\begin{restatable}[DisCo is order-isomorphic to symmetric kernels]{prop}{proporderiso}\label{prop:order-iso}
Let $p$ be a PC defined over real-valued random variables $X_1, \dots, X_n$, where each $X_i$ is modeled by $m$ leaves. Each leaf $\Leaf_{ij}$ is characterized by a symmetric density $f$ with a location parameter $\theta_{ij}$ and a monotone cumulative distribution function (CDF) $\Phi$. Let $d: \bm{\mathcal{X}} \times \bm{\mathcal{X}} \rightarrow \mathbb{R}_{\geq 0}$ be a $p$-norm and $k: \bm{\mathcal{X}} \times \bm{\mathcal{X}} \rightarrow \mathbb{R}$ be a symmetric, strictly increasing bounded kernel. Assume that both the leaf densities and $k$ are strictly monotonic functions of $d$ up to constant factors, define $\mathbf{\Theta} = \{\bm{\theta}_j\}_{j=1}^m$. Then, for any two points $\mathbf{x}_1, \mathbf{x}_2 \in \bm{\mathcal{X}}$ and $\bm{\theta} \in \bm{\Theta}$, $\rho_p$ is order-isomorphic to the kernel $k(\cdot, \cdot)$, such that $d(\mathbf{x}_1, \bm{\theta}) > d(\mathbf{x}_2, \bm{\theta})$ implies both $(1 - \rho_p(\mathbf{x}_1; \bm{\theta})) > (1 - \rho_p(\mathbf{x}_2; \bm{\theta}))$ and $k(\mathbf{x}_1; \bm{\theta}) > k(\mathbf{x}_2; \bm{\theta})$.
\end{restatable}
\begin{proof}[Proof (Sketch)]
    To show that $k$ and $1 - \rho_p(\cdot)$ are order-isomorphic, we show that due to the coupling of $d$ to $k$ and $\rho_p$, the values of $k$ increase when $\rho_p$ decreases and vice versa. Intuitively, we make use of applying the $\min$-operator to find out whether a datapoint lies in the left or right half of the symmetric leaf density of the PC. We then proceed by showing that applying \citet{fisher1932statisticalmethods} to obtain DisCo scores of product nodes exactly inverts the ranking, which is kept by the subsequent mixture aggregation of the root node in the induced tree representation. Since we use $(1 - \rho_p(\cdot))$ as an uncertainty proxy, the ranking is inverted once again, yielding the same ranking as $k$, which is assumed to only be a constant scale and shift of $d$. See App. \ref{app:iso-order-kernel} for the full proof.
\end{proof}
Thus, DisCo provides an uncertainty estimate functionally similar to a symmetric kernel commonly used in BO algorithms when the circuit structure scales with the number of observations. Crucially, by exploiting the decomposability and smoothness of PCs, DisCoMBO enables exact inference in $O(n)$ relative to the circuit size, bypassing the $O(n^3)$ bottleneck of GPs while maintaining comparable decision-theoretic benefits.

%\SetCustomAlgoRuledWidth{6.6cm}
%\newpage
% \clearpage

\subsection{DisCo-Guided SMBO and Knowledge Integration}\label{sec:discombo}

%\begin{wrapfigure}[24]{R}{0.53\textwidth} % "R" for right-aligned, adjust width as needed
\begin{wrapfigure}[25]{R}{0.53\textwidth} % "R" for right-
\vspace{-0.25cm}
\begin{minipage}{0.53\textwidth} % Adjust the minipage width to fit within wrapfigure
\begin{algorithm}[H]
\caption{\textbf{DisCoMBO}}
\label{algo:DisCoMBO}
\begin{algorithmic}[1]
\STATE \textbf{Input:} Search space $\bm{\mathcal{X}}$, objective $f: \bm{\mathcal{X}} \rightarrow \mathbb{R}$, user prior $q(\hat{\bm{\mathcal{X}}} \subset \bm{\mathcal{X}})$ (can be provided at any time), decay $\gamma$, number of bins $K$, binning schedule $s$
\STATE Sample $J$ configurations $\{\mathbf{x}\}_{i=1}^J$ randomly
\STATE $\bm{\mathcal{D}} \gets \{(\mathbf{x}_i, f(\mathbf{x}_i) \}$ for $i \in \{1, ..., J\}$
\WHILE{not converged}
    \STATE Fit PC $p$ on $\bm{\mathcal{D}}$
    \STATE Set $y^* \gets \max Y \in \bm{\mathcal{D}}$
    \IF{$d \sim \text{Bernoulli}(\gamma^t p_{\text{use\_prior}}) = 1$}
        \STATE Sample $\hat{\mathbf{x}} \sim q(\hat{\bm{\mathcal{X}}})$ if $q$, else None
        \IF{$\rho_p(\hat{\mathbf{x})} < \tau$}
            $\hat{\mathbf{x}} \gets$ None
        \ENDIF
    \ENDIF
    \STATE Sample $\mathcal{C} = \{\mathbf{x} \sim p(\bm{\mathcal{X}} | \hat{\mathbf{x}}, y^*)\}_{i=1}^N$
    \IF{use random}
        \STATE $\mathcal{C} = \mathcal{C} \cup \{\mathbf{x} \sim \mathcal{U(\bm{\mathcal{X}})}\}_{i=1}^M$
    \ENDIF
    \STATE Compute DisCo scores $\mathbf{r} = [\rho_p(\mathbf{x}) \forall \mathbf{x} \in \mathcal{C}]$
    \STATE Compute annealing $\beta_t = s(t)$
    \STATE Compute bins $\mathbf{b} = [(\frac{b}{K})^{\beta_t}]_{b=0}^{K-1}$
    \STATE Assign DisCo scores to corresponding bins
    \STATE Sample $\rho$ from randomly chosen bin
    \STATE Define $\mathbf{X} = \{\mathbf{x}: \mathbf{x} \in \mathcal{C} | \rho_p(\mathbf{x}) = \rho \}$
    \STATE $\bm{\mathcal{D}} \gets \bm{\mathcal{D}} \cup \{(\mathbf{x}', f(\mathbf{x}')) \}$ for random $\mathbf{x}' \sim \mathbf{X}$
\ENDWHILE
\end{algorithmic}
\end{algorithm}
\end{minipage}
%\vspace{-0.3cm}
\end{wrapfigure}
With DisCo at hand, we are ready to propose DisCoMBO, the first PC-based SMBO framework guided by rigorous uncertainty quantification. DisCoMBO synthesizes the natural knowledge integration of generative PC-based SMBO with the exploration-exploitation semantics of traditional SMBO.
This unification enables a formal detection and feedback mechanism for external expertise 
that contradicts empirical evidence. The resulting procedure (Algorithm \ref{algo:DisCoMBO}) ensures robust optimization by treating expert-in-the-loop priors as informative yet fallible inputs, preserving convergence guarantees even under misleading priors.

Specifically, DisCoMBO operates in two phases: (1) \textit{knowledge-informed candidate generation}, utilizing the tractability of PCs and DisCo-based prior quality assessment, and (2) the \textit{final candidate selection}, leveraging DisCo to manage the exploration-exploitation trade-off.

\textbf{(1) Knowledge-informed Candidate Generation:} 
At each iteration $t$, we use LearnSPN~\citep{gens2013learnSPN} to fit a PC $p(\bm{\mathcal{X}}, Y)$ to the observed data $\bm{\mathcal{D}}$ \textbf{(Line 5)}. To construct the candidate set $\mathcal{C}$, we leverage the tractability of PCs to perform \textit{prior-conditioned sampling}. Given external knowledge as a prior $q(\hat{\bm{\mathcal{X}}})$ over a (possibly empty) subspace $\hat{\bm{\mathcal{X}}} \subset \bm{\mathcal{X}}$, we sample $\hat{\mathbf{x}} \sim q$ and condition the PC on both this prior and the best-observed performance $y^* = \max Y \in \bm{\mathcal{D}}$:
\begin{equation} \label{eq:conditional}p(\bm{\mathcal{X}}' | \hat{\mathbf{x}}, y^*) = \frac{p(\bm{\mathcal{X}}', \hat{\bm{\mathcal{X}}} = \hat{\mathbf{x}}, Y = y^*)}{ \int_{\bm{\mathcal{X}}'} p(\bm{\mathcal{X}}', \hat{\bm{\mathcal{X}}} = \hat{\mathbf{x}}, Y = y^*)}
\end{equation}
where $\bm{\mathcal{X}}' = \bm{\mathcal{X}} \setminus \hat{\bm{\mathcal{X}}}$. We then sample $N$ configurations $\mathbf{x} = (\hat{\mathbf{x}}, \mathbf{x}')$ where $\mathbf{x}' \sim p(\bm{\mathcal{X}}' | \hat{\mathbf{x}}, y^*)$ (\textbf{Line 6-9}). This strategy introduces a proximity bias towards high-performing regions weighted by the expert prior, while $\mathcal{C}$ is optionally augmented with $M$ random samples from $\bm{\mathcal{X}}$ to maintain global coverage (\textbf{Line 10-11}). To safeguard against misleading priors, we employ a two-fold recovery mechanism. The first layer consists of a semantic prior filter that utilizes the DisCo score to quantify how well a sample $\hat{\mathbf{x}} \sim q$ conforms to the learned surrogate. By marginalizing the PC over the remaining search space $\bm{\mathcal{X}}'$ and conditioning the PC on $Y = y^*$, we obtain $p'(\hat{\bm{\mathcal{X}}} | y^*)$ and compute the conformance $\rho_p(\hat{\mathbf{x}})$ as described in Sec. \ref{sec:disco}. If $\rho_p(\hat{\mathbf{x}}) < \tau$ for a threshold $\tau \in [0, 1]$, the sample is rejected for failing to align with empirical evidence (\textbf{Line 9}). This is complemented by a fallback mechanism following \citet{seng2025ihpo}, which decays the probability of utilizing the external prior $q$ by a factor $\gamma$ at each iteration, ensuring the optimizer asymptotically reverts to the default SMBO regime without external priors (\textbf{Line 7-8}).

\textbf{(2) Final Candidate Selection:} 
From the generated set $\mathcal{C}$, we select the most promising configuration for evaluation by computing the DisCo score $\rho_p(\mathbf{x})$ for each $\mathbf{x} \in \mathcal{C}$ (\textbf{Line 13}). To calculate $\rho_p(\mathbf{x})$ in the absence of an observed $Y$, we apply a proxy strategy: setting $Y = y^*$ yields scores reflecting conformance with high-performing regions (useful to bias towards proximity to $\mathbf{x}^*$ during exploitation), while marginalizing $Y$ allows DisCo to quantify general "novelty" relative to all prior observations (useful to foster more aggressive exploration). Which strategy is applied can either be fixed (e.g., on use the $Y=y^*$ strategy) or can depend on the schedule $s(t)$ (e.g., alternating both strategies). Selection is governed by an annealing parameter $\beta_t$ derived from a schedule $s(t)$, which modulates the exploration-exploitation trade-off at iteration $t$ (\textbf{Line 14}). We partition candidates into bins within the unit interval $[0, 1]$ based on their $\rho_p(\mathbf{x})$ values. Under this scheme, $\beta_t \gg 1$ causes the selection probability to concentrate around bins near $0$ (high uncertainty/exploration), while $\beta_t \rightarrow 0$ shifts focus toward bins near $1$ (low uncertainty/exploitation) (\textbf{Line 15}). The final candidate $\mathbf{x}'$ is then sampled from a randomly selected bin according to this distribution (\textbf{Line 16-19}). This work employs a sinusoidal schedule $s(t)$ ensuring smooth transitions between exploration ($\beta_t \gg 1$) and exploitation ($\beta_t \rightarrow 0$). However, the framework admits any functional form, including dynamic schedules based on surrogate "energy" or uncertainty decay. We leave the exploration of different scheduling alternatives for future work.

\paragraph{DisCoMBO is Zero Regret and Efficient} 
After presenting DisCoMBO, we analyze its convergence behavior. We find that DisCoMBO is a zero regret algorithm as stated in the following proposition.
\begin{restatable}[DisCoMBO is a zero-regret algorithm]{prop}{propzeroregret}\label{prop:zero-regret}
    Assume some $f \in \mathcal{H}_k$ with $\mathcal{H}_k$ being a Reproducing Kernel Hilbert Space (RKHS) with squared exponential kernel $k = \phi \circ \rho$ such that $f: \bm{\mathcal{X}} \rightarrow \mathbb{R}$ and $\mathbf{x}^*$ being the optimum of $f$. Further, assume a periodic schedule $s(t) \in \mathbb{R}_{\geq 0}$ that determines whether focus is set on exploitation or exploration, s.t. exploitation periods get longer over time, i.e. $\pi_t = c \cdot \pi_{t-1}$ for some $c > 1$ and period length $\pi_t$. Let the number of bins $K \rightarrow \infty$ and $p(y | \mathbf{x})$ be unbiased where $p$ is a PC learned on observations $\mathcal{D} = \{(\mathbf{x}_1, y_1), \dots, (\mathbf{x}_n, y_n) \}$. Then, DisCoMBO has sub-linear cumulative regret $R_T$, i.e. $\frac{R_T}{T} \rightarrow 0$ with $T \rightarrow \infty$.
\end{restatable}

\begin{proof}[Proof (Sketch)]
    To show that DisCoMBO is a zero regret algorithm, we exploit the fact that DisCo is order-isomorphic to symmetric kernels (see Prop. \ref{prop:order-iso}). Given that we assume a periodic exploration-exploitation schedule $s(t)$ which controls the exploration and exploitation, we show that DisCoMBO's cumulative regret grows sub-linear in exploration and exploitation. With $T \rightarrow \infty$, we switch infinitely often between exploration and exploitation, thus converging to zero regret eventually. See App. \ref{app:regret_analysis} for the full proof.
\end{proof}

Furthermore, since DisCoMBO leverages PCs as surrogates, all inference and sampling operations (including computation of DisCo) are linear in inference time w.r.t. the circuit size. Therefore, DisCoMBO also scales linearly with increasing sample and search space sizes.
%The proof of Prop. \ref{prop:conditioning_transform} is given in App. \ref{app:proofs}.
%We now proceed and evaluate IBO-HPC empirically.

\section{Experimental Evaluation} \label{sec:Experiments}
We provide an extensive empirical evaluation of DisCoMBO to investigate the following research questions: \textbf{(Q1) Competitive Performance:} Can DisCoMBO bridge the performance gap between purely generative PC-based methods and state-of-the-art Bayesian or ensembling surrogates in settings without external knowledge guidance? \textbf{(Q2) Effectiveness of Informed Optimization:} Does DisCoMBO effectively translate external domain expertise into increased sample efficiency, and does it maintain the effectiveness of generative knowledge integration? \textbf{(Q3) Robustness and Recovery:} Does the DisCo score provide a reliable, accurate signal for detecting misleading expertise, and does it enable principled recovery in the presence of misleading priors?

\paragraph{Experimental Setup}To evaluate DisCoMBO as a general-purpose optimizer (\textbf{Q1}), we compared it against established SMBO frameworks with various surrogates, including BO with Random Forests (RF)~\citep{skopt}, BO with Tree-Parzen Estimators (TPE)~\citep{akiba2019optuna}, SMAC~\citep{hutter20111SMAC, lindauer2022SMAC}, HEBO (using a GP surrogate)~\citep{cowenrivers2022hebo}, and CMA-ES~\citep{hansen2023cmaes}. To assess its performance as a generative surrogate, we compared it against IBO-HPC~\citep{seng2025ihpo}. For knowledge integration tasks (\textbf{Q2}, \textbf{Q3}), we evaluated against IBO-HPC, $\pi$BO~\citep{hvarfner2022pibo}, and DynaBO~\citep{fehring2026dynabo}.
Our evaluation spans 9 tasks and 6 search spaces—including continuous, discrete, graph, and mixed domains—across multiple fields. For AutoML, we utilize NAS-Bench-201~\citep{nasbench201}, JAHS~\citep{jahsbench201}, and HPO-B~\citep{arango2021hpoblargescalereproduciblebenchmark} to cover NAS and HPO. In the Physical Sciences, we employ the PV crossed-barrel benchmark~\citep{gongora2020crossedbarrel} for material discovery and the Floris Wake Modeling suite for windpark optimization~\citep{Fleming2020FLORIS}. We repeat each experiment with 50 different seeds to ensure statistical significance of our results. We report normalized regret against wall-clock time whenever wall-clock time measurements are provided by benchmarks (iterations if not). See App. \ref{app:benchmarks} for a detailed overview.

\begin{figure}[t]
    \centering
    \includegraphics[scale=0.35]{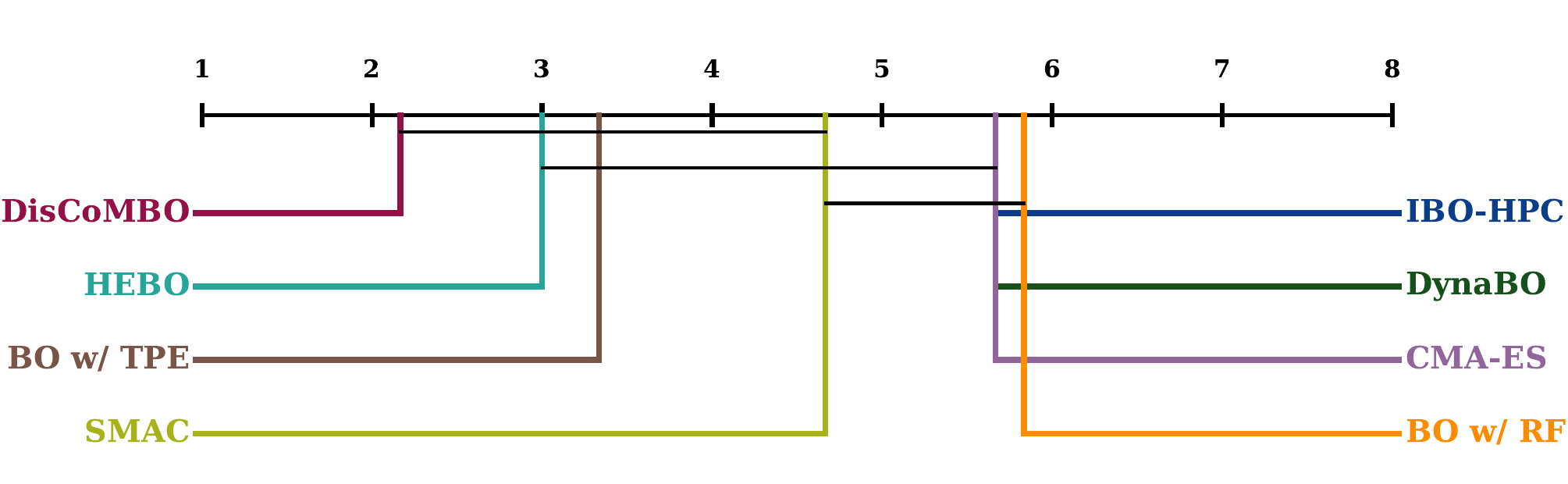}
    \caption{\textbf{DisCoMBO is competitive to top-tier SMBO methods.} DisCoMBO ranks \textbf{first} across 9 tasks when \textbf{no prior} is given, being statistically indistinguishable to strong baselines like HEBO. DisCoMBO significantly outperforms IBO-HPC, its PC-based counterpart, lacking a well-defined exploitation-exploration trade-off. Assessment is based on Wilcoxon signed-rank test ($p=0.05$).}
    \label{fig:cd_standard}
\end{figure}

\begin{figure}[t]
    %\centering
    %% First row of subfigures
    %\begin{subfigure}[c]{0.24\textwidth}
    %    \centering
    %    \includegraphics[width=\textwidth]{images/discombo/standard/jahs_co.pdf}
    %    \caption{JAHS (C. Histology)}
    %    \label{fig:jahs_co}
    %\end{subfigure}
    %\hfill
    %\begin{subfigure}[c]{0.24\textwidth}
    %    \centering
    %    \includegraphics[width=\textwidth]{images/discombo/standard/hpob_6767_9914.pdf}
    %    \caption{HPO-B (6767:9914)}
    %    \label{fig:jahs_co}
    %\end{subfigure}
    %\hfill
    %\begin{subfigure}[c]{0.24\textwidth}
    %    \centering
    %    \includegraphics[width=\textwidth]{images/discombo/standard/nas201.pdf}
    %    \caption{NAS-Bench-201}
    %    \label{fig:jahs_fm}
    %\end{subfigure}
    %\hfill
    %\begin{subfigure}[c]{0.24\textwidth}
    %    \centering
    %    \includegraphics[width=\textwidth]{images/discombo/standard/pv_crossed_barrel.pdf}
    %    \caption{PV (Crossed Barrel)}
    %    \label{fig:nas101}
    %\end{subfigure}
    \centering
    \includegraphics[width=\linewidth]{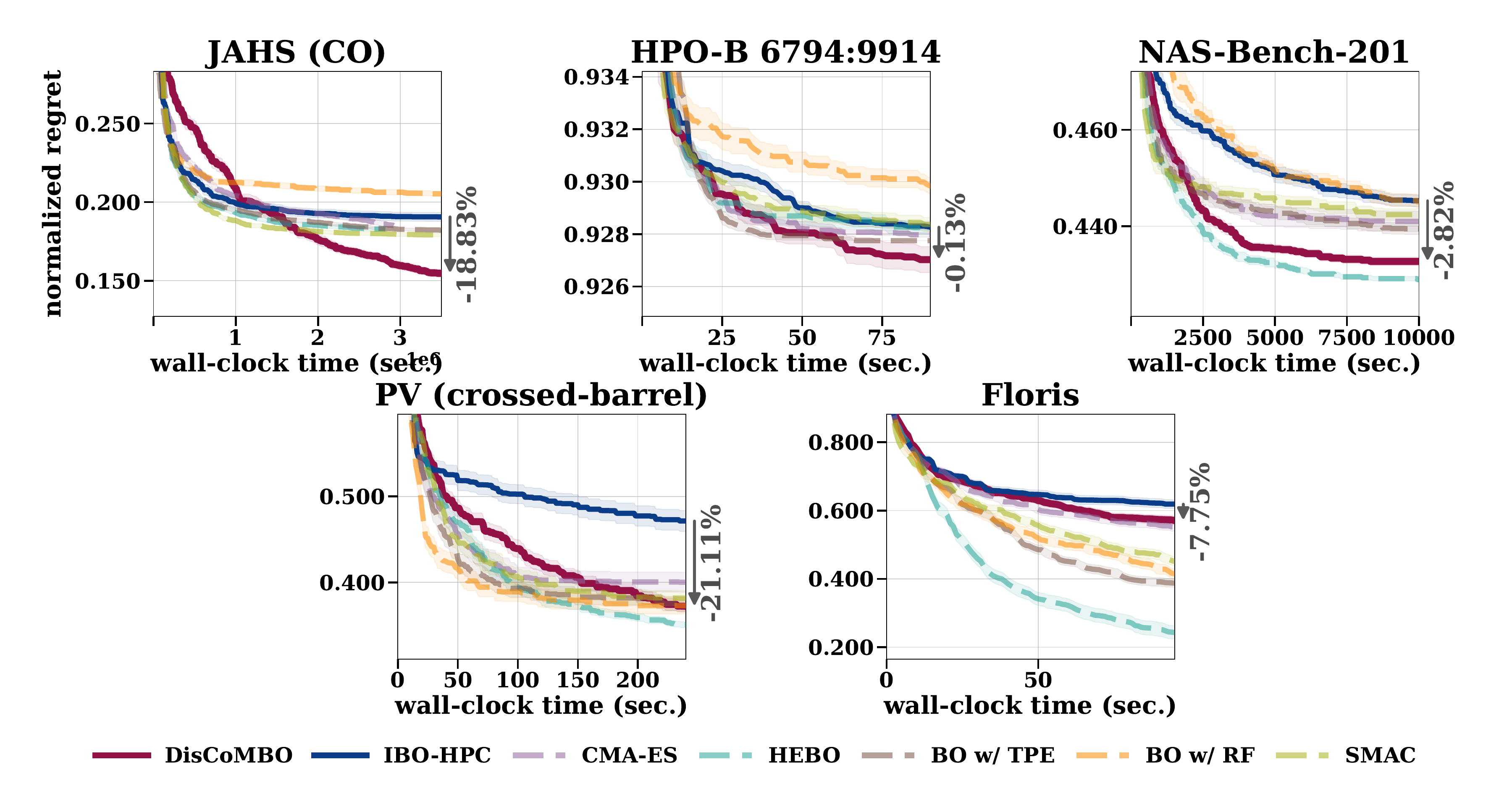}
    \caption{\textbf{DisCoMBO is effective on diverse tasks.} DisCoMBO (\solidline{DisCoMBO}) effectively optimizes black-box functions across tasks with diverse search spaces and optimization landscapes. In \textbf{8/9} tasks, DisCoMBO is competitive to top-tier SMBO algorithms (see App.~\ref{app:further_results} for remaining tasks). We achieve significant improvement upon IBO-HPC (\solidline{IBO}) across \textbf{9/9} tasks.}
    \label{fig:smbo_default}
    \vspace{-0.5cm}
\end{figure}

\paragraph{Knowledge Integration Protocol}To evaluate the efficacy of knowledge integration, we followed the protocol in~\citet{seng2025ihpo} to define \textit{beneficial} and \textit{misleading} priors over search space subspaces. External knowledge is provided as priors over a subset of dimensions. Beneficial knowledge favors values near the global optima $\mathbf{x}^+$ for a small subset of dimensions. Conversely, misleading knowledge favors regions near poor-performing configurations $\mathbf{x}^-$ across a larger subset of dimensions to rigorously test the robustness of the DisCo rejection mechanism. This setup allows us to quantify the surrogate's ability to benefit from helpful hints while maintaining safety when faced with misleading domain expertise (see App.~\ref{app:interaction_definition} for exact definitions).
\begin{figure*}[t]
     \centering
     \includegraphics[width=\linewidth]{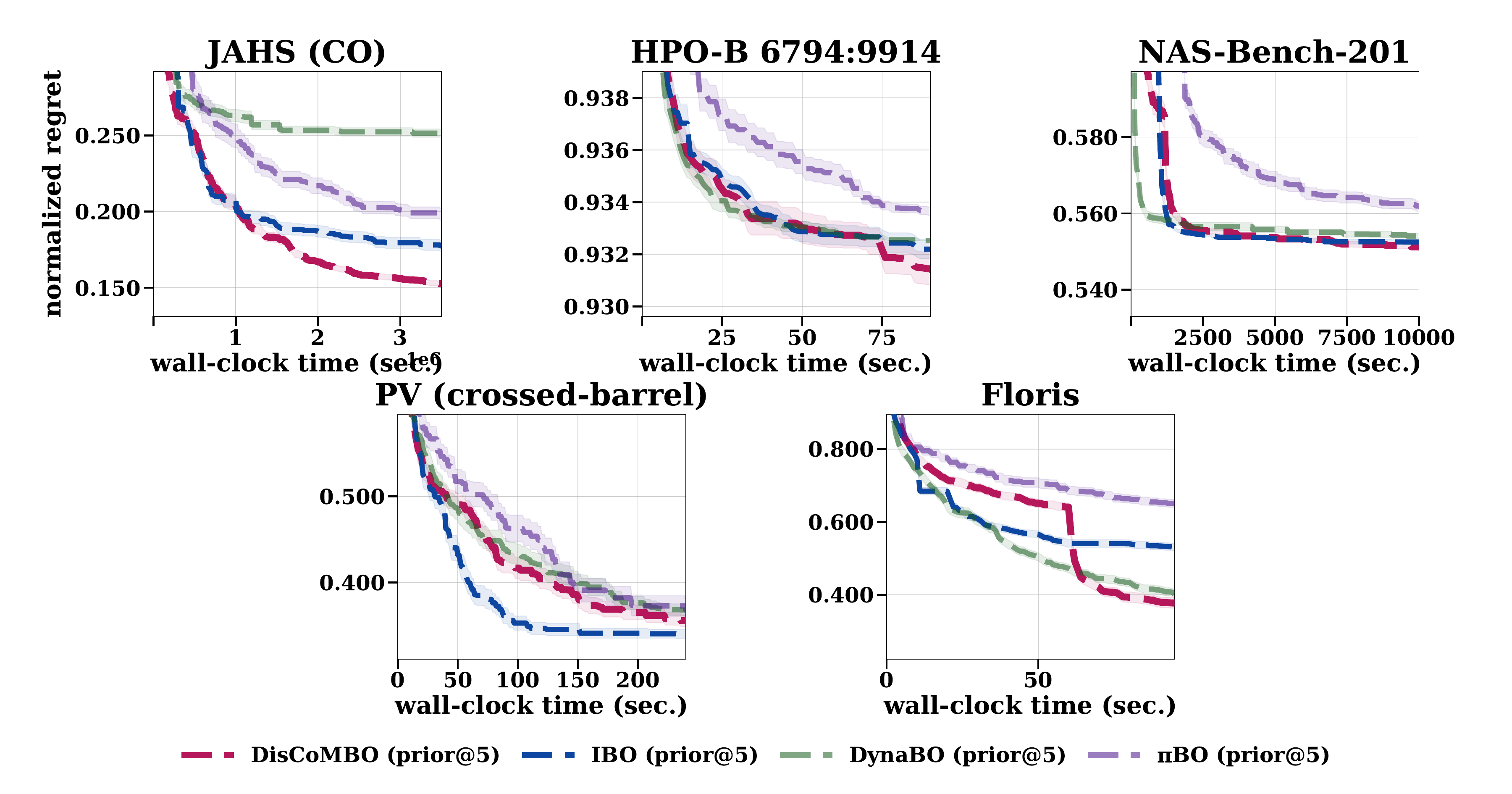}
     %\vspace{-0.3cm}
    \caption{\textbf{DisCoMBO leverages beneficial priors effectively.} We find that DisCoMBO (\dashedline{DisCoMBOearly}) retains IBO-HPC's (\dashedline{IBOearly}) efficacy in incorporating priors. We outperform baselines or perform competitively on \textbf{7/9} tasks. See App.~\ref{app:further_results} for additional results and discussion.}
    \label{fig:smbo_int_bene}
    \vspace{-0.5cm}
\end{figure*}
\subsection{\textbf{(Q1)} DisCoMBO is Competitive to Traditional SMBO}\label{subsec:Q1}
To evaluate whether DisCoMBO successfully bridges the gap between generative modeling and decision-theoretic rigor, we compared it against six established baselines—including state-of-the-art SMBO methods and the generative IBO-HPC—in the unassisted setting.
As shown in the critical difference diagram (Fig.~\ref{fig:cd_standard}), DisCoMBO significantly outperforms the generative baseline (IBO-HPC) and ranks first across the aggregated benchmark tasks. Statistically, DisCoMBO is indistinguishable from top-tier SMBO baselines, whereas IBO-HPC ranks significantly lower. These results, supported by the regret analysis in Fig.~\ref{fig:smbo_default} and Fig.~\ref{fig:full_regret_standard}, confirm that DisCoMBO (\solidline{DisCoMBO}) achieves competitive performance to traditional SMBO while maintaining the unique flexibility of the PC architecture.
DisCoMBO demonstrates multi-domain robustness, excelling across highly distinct optimization landscapes. It achieves top-tier performance on the materials discovery benchmark (PV Crossed-Barrel) and on structurally diverse AutoML tasks, natively navigating classical hyperparameter spaces (HPO-B), topological graphs for neural architectures (NAS-Bench-201), and complex joint spaces combining both (JAHS). This consistent performance highlights DisCoMBO's ability to capture diverse search space semantics without degradation.
The only exception to this dominant profile occurs on Floris, a purely continuous landscape where DisCoMBO substantially outperforms IBO-HPC but falls slightly short of traditional SMBO. We attribute this to the topological properties of the Floris objective, which naturally favors discriminative surrogates explicitly optimized for continuous regression (see App.~\ref{app:further_results} for an extended discussion). Nonetheless, DisCoMBO consistently outperforms prior PC-based optimization on \textbf{9/9} tasks and remains highly competitive with the state of the art on \textbf{8/9} tasks. We thus answer \textbf{(Q1)} affirmatively, establishing DisCoMBO as a robust, principled successor to purely generative PC-based methods.

%\subsection{\textbf{(Q2, Q3)} Interactive and Resilient HPO \& NAS with IBO-HPC}
\subsection{\textbf{(Q2, Q3)} DisCoMBO Integrates Priors Effectively and is Resilient}
\label{sec:Exp_InteractiveHPO}
We demonstrate that DisCoMBO retains and refines the capabilities for targeted and robust integration of external priors. We evaluate DisCoMBO under beneficial and adversarial expert guidance; detailed specifications for these prior distributions are provided in App.~\ref{app:interaction_definition}. Lastly, we show that DisCo accurately discriminates between beneficial and misleading priors.

\paragraph{Beneficial Priors}
As illustrated in Fig.~\ref{fig:smbo_int_bene} and Fig.~\ref{fig:full_regret_early}, DisCoMBO consistently retains or enhances the targeted knowledge integration characteristic of generative PC-based approaches when guided by beneficial priors. Across our evaluation, DisCoMBO matches or outperforms the purely generative IBO-HPC baseline. Furthermore, although optimizing assisted convergence was not the primary objective of this framework, DisCoMBO outperforms or matches baselines on \textbf{7/9} tasks. This robust behavior is replicated across the remaining JAHS, HPO-B, and Floris benchmarks (see App.~\ref{app:further_results}), confirming that our novel uncertainty semantics does not degrade the efficacy of generative prior guidance. Given these results, we answer \textbf{(Q2)} affirmatively.

\paragraph{Recovery and Misleading Prior Detection}
\begin{figure}[t]
    \vspace{0.3cm}
    \includegraphics[width=1.\linewidth, height=3.5cm]{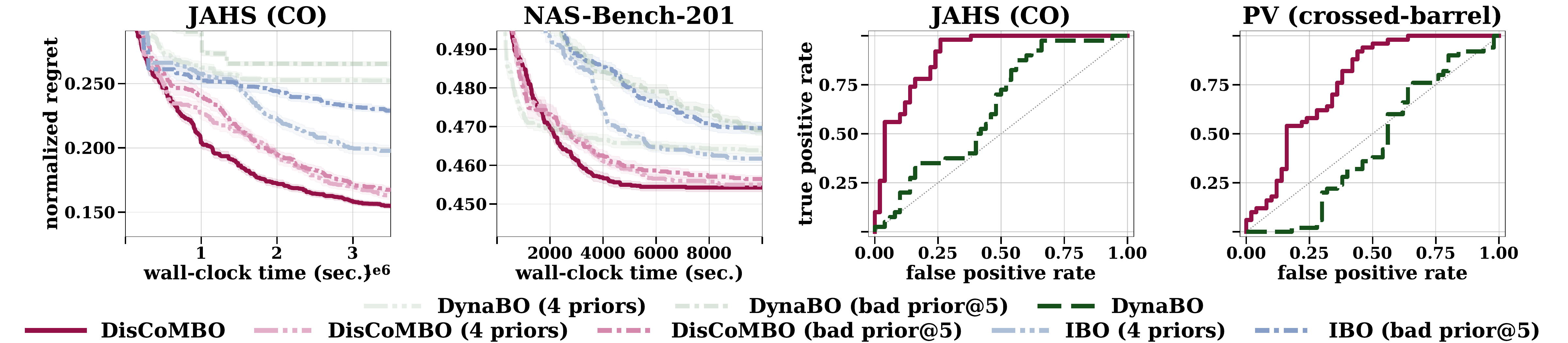}
    \caption{\textbf{DisCoMBO successfully recovers from and detects misleading priors.} (\textit{Left}) DisCoMBO effectively recovers (\dashdottedline{DisCoMBOrecover}) from misleading priors, even in worst-case scenarios in which misleading priors are always accepted. Also, it effectively handles sequences of priors (\dashdottedline{DisCoMBOmany}) with contradicting information. (\textit{Right}) We find that our DisCo-based detection mechanism's accuracy (\solidline{DisCoMBO}) is significantly higher compared to DynaBO's heuristic-based detection mechanism (\dashedline{dynabo}), indicating DisCoMBO is able to detect misleading priors.}
    \label{fig:roc_misleading}
\end{figure}
We evaluate DisCoMBO's resilience against misleading priors across three operational dimensions: the prior detection mechanism, the fallback recovery layer, and the capacity to handle sequences of contradicting priors. As shown in Fig.~\ref{fig:roc_misleading} (Left) and Fig.~\ref{fig:full_regret_recover}, when a severely misleading prior is initially accepted, DisCoMBO successfully recovers over time (\dashdottedline{DisCoMBOrecover}), demonstrating robust fallback behavior under worst-case assumptions. Furthermore, the framework effectively handles sequences of alternating beneficial and highly misleading priors introduced at different optimization phases (\dashdottedline{DisCoMBOmany}), maintaining stability under these adversarial conditions (see App.~\ref{app:interaction_definition} for exact prior definitions). To isolate the prior detection mechanism from downstream optimization dynamics, we evaluate its accuracy against the heuristic-driven safeguard of DynaBO—the only baseline equipped with a comparable detection layer. Specifically, we measure detection accuracy after 10\%, 20\%, and 30\% of the allowed budget. Fig.~\ref{fig:roc_misleading} (Right) and Fig.~\ref{fig:mis_prior_detection} illustrate the ROC curve aggregated across iterations on JAHS (CO) and PV to ensure task diversity. DisCo-based detection yields a consistently higher AUC compared to DynaBO's heuristic approach, demonstrating that DisCo enables robust and accurate detection of misleading priors. We thus answer \textbf{(Q3)} affirmatively.

\section{Related Work}
In this section, we contextualize DisCoMBO within the broader landscape of sequential optimization, examining the evolution of external knowledge integration in SMBO and the emerging role of Probabilistic Circuits as a flexible alternative to traditional surrogates.

\paragraph{Sequential Model Based Optimization (SMBO) with External Priors} SMBO seeks to optimize expensive black-box functions $f: \bm{\mathcal{X}} \rightarrow \mathbb{R}$ by alternating between a probabilistic surrogate and a selection policy determining the next query point $\mathbf{x}'$~\citep{ShahriariSWAF16}. Although effective, SMBO often significantly benefits from additional information available to simplify the problem at hand or provide additional guidance through the search process. To improve sample efficiency, especially in complex spaces, various strategies for incorporating domain expertise have emerged. Existing approaches either impose "hard" constraints to prune the search space \textit{a priori}~\citep{perrone2019learning, DBLP:journals/kbs/RamachandranGRL20}, which risks excluding global optima, or leverage "warm start" transfer learning from structured historical logs~\citep{Salinas2020AQA, bai_transfer_2023}. More flexible "soft" integration methods, such as those proposed by ~\citet{souza2021bopro}, ~\citet{hvarfner2022pibo}, and ~\citet{fehring2026dynabo}, re-weight the acquisition function according to user-specified priors. However, this re-weighting mechanism can result in an opaque selection process where the prior's influence is mediated indirectly through the acquisition landscape. To tackle this, ~\citet{seng2025ihpo} exploits the conditional sampling of Probabilistic Circuits (PCs) to allow for more targeted knowledge integration. Yet, existing generative routines lack the formal exploration--exploitation semantics intrinsic to traditional SMBO and provide no normalized, scale-invariant metric to evaluate the quality of the external guidance. Our work, DisCoMBO, bridges these gaps by introducing a formal conformance score that establishes a rigorous exploration-exploitation trade-off for PC-based surrogates while enabling the robust rejection of misleading external guidance.

\paragraph{Probabilistic Circuits (PCs)} 
%To navigate the trade-offs between the flexibility of tree-based models and the formal rigor of Bayesian models like Gaussian Processes while retaining targeted external knowledge integration, we leverage 
Probabilistic Circuits (PCs) are computational graphs that represent multivariate distributions while guaranteeing tractable exact inference for a broad range of probabilistic queries~\citep{choi_2020}. By imposing architectural constraints such as \textit{decomposability} and \textit{smoothness}, PCs reduce computationally expensive marginalization tasks to linear-time operations over tractable leaf distributions~\citep{molina_2018, perharz2020einsum}. Beyond their efficiency, PCs are uniquely suited for integrating external knowledge; unlike traditional regressive surrogates, their generative nature allows for the direct injection of domain expertise through conditional sampling, providing a more targeted alternative to heuristic re-weighting of acquisition functions or model predictions~\citep{seng2025ihpo}. While PC architectures range from GPU-optimized fixed structures~\citep{perharz2020einsum} to those learned directly from data~\citep{gens2013learnSPN}, the latter are particularly effective for the small-to-medium datasets typical of SMBO. These learned structures are derived by recursively partitioning data via clustering and independence tests, naturally capturing the structural dependencies of the search space. Given the hybrid nature of real-world optimization domains, DisCoMBO specifically utilizes Mixed Sum-Product Networks (MSPNs)~\citep{molina_2018}, which extend the PC framework to mixed continuous and discrete variables using piecewise polynomial leaves. For aggregation of evidence, DisCo uses Fisher's method at product nodes as introduced in Membership Circuits~\citep{wittelsbach2026membership}. For formal details, see App.~\ref{app:PCs}.

\section{Conclusion}
\label{sec:conclusion}
We introduced DisCo, a novel conformance score that quantifies model surprise for Probabilistic Circuits (PCs), and developed DisCoMBO, a principled PC-based SMBO framework. By leveraging DisCo as an uncertainty proxy, DisCoMBO simultaneously equips PC-based optimization with formal exploration-exploitation semantics and provides an accurate, scale-independent detection mechanism for misleading priors. Empirically, DisCoMBO successfully addresses the performance limitations of purely generative PC-based methods, operating on par with top-tier SMBO baselines across a diverse suite of real-world benchmarks. Crucially, it achieves this decision-theoretic rigor while fully retaining—and expanding upon—the targeted, robust external prior integration capabilities inherent to generative PC architectures.

\paragraph{Limitations \& Future Work}
Despite these advancements, several promising avenues for future research remain. First, standard PC architectures are structurally constrained by a discrete latent variable formulation, which can introduce a discretization bias on smooth, fully continuous manifolds and limit regression accuracy. Advancing underlying PC parameterizations to natively support continuous topologies would directly enhance DisCoMBO's optimization efficiency on such landscapes. Furthermore, a deeper theoretical analysis of DisCo's kernel-like properties offers an exciting frontier. Refining this connection may ultimately allow PC-based optimization to bridge the gap between linear-time probabilistic inference and the elegant, mathematically rigorous properties characteristic of kernel methods like Gaussian Processes.

%\paragraph{Impact Statement} After careful reflection, the authors have determined that this work presents no notable negative impacts to society or the environment.

\section*{Acknowledgements}
This work was supported by the National High-Performance Computing Project for Computational Engineering Sciences (NHR4CES) and by the German Research Foundation (DFG) under Germany's Excellence Strategy (EXC-3057/1 \enquote{Reasonable Artificial Intelligence}).

\bibliography{references}
\bibliographystyle{plainnat}

\clearpage

\appendix

\section{Probabilistic Circuits}\label{app:PCs}
Since probabilistic circuits (PCs) are a key component of our method, we provide more details on these models in the following. Let us first start with a rigorous definition of PCs.

\begin{defin}
    A probabilistic circuit (PC) is a computational graph encoding a distribution over a set of random variables $\mathbf{X}$. It is defined as a tuple $(\graph, \phi)$ where $\graph = (V, E)$ is a rooted, directed acyclic graph and $\phi: V \rightarrow 2^{\mathbf{X}}$ is the \textit{scope} function assigning a subset of random variables to each node in $\graph$.
    For each internal node $\Node$ of $\graph$, the scope is defined as the union of scopes of its children, i.e. $\phi(\Node) = \cup_{\Node' \in \ch{\Node}} \phi(\Node')$. Each leaf node $\Leaf$ computes a distribution/density over its scope $\phi(\Leaf)$. All internal nodes of $\graph$ are either a sum node $\SumNode$ or a product node $\ProductNode$ where each sum node computes a convex combination of its children, i.e.,
    $\SumNode = \sum_{\Node \in \ch{\SumNode}} w_{\SumNode, \Node}\Node$, and each product computes a product of its
    children, i.e., $\ProductNode = \prod_{\Node \in \ch{\ProductNode}}\Node$. 
\end{defin}

With this definition at hand, we describe the tractable key operations of PCs relevant to our method in more detail.

\paragraph{Inference}
Inference in PCs is a bottom-up procedure.
To compute the probability of given evidence $\mathbf{X} = \mathbf{x}$, the densities of the leaf nodes are evaluated first. This yields a density value for each leaf. The leaf densities are then propagated bottom-up by computing all product/sum nodes.
Eventually, the root node holds the probability/density of $\mathbf{x}$. 
Note that typically, multiple leaf nodes correspond to the same random variable. Thus, if the children of a sum node have the same scope, we can interpret sum nodes as mixture models. Conversely, if the children of a product node have \textit{non}-overlapping scopes, a product node can be interpreted as a product distribution of two (independent) random variables. We call these two properties smoothness and decomposability. 
More formally, \textit{smoothness} means that for each sum node $\SumNode \in V$ it holds that $\phi(\Node) = \phi(\Node')$ for $\Node, \Node' \in \ch{\SumNode}$. \textit{Decomposability} means that for each product node $\ProductNode \in V$ it holds that $\phi(\Node) \cap \phi(\Node') = \emptyset$ for $\Node, \Node' \in \ch{\ProductNode}$, $\Node \neq \Node'$.
Hence, PCs can be interpreted as hierarchical mixture models.

\paragraph{Marginalization} Decomposability implies that marginalization is tractable in PCs and can be done in linear time of the circuit size. This is because integrals that can be rewritten by nesting single-dimensional integrals can be computed only in terms of leaf integrals, which are assumed to be tractable as they follow certain distributions (e.g., Gaussian). Computing such nested integrals only in terms of leaf integrals is possible because single-dimensional integrals commute with the sum operation and affect only a single child of product nodes. For more details on the computational implications of decomposability, refer to \citep{PeharzTPD15}.

Practically, there are two ways to marginalize certain variables from the scope of a PC. One approach is structure-preserving, and marginalization is achieved by setting all leaves corresponding to the set of random variables that are supposed to be marginalized to $1$. The second approach constructs a new PC representing the marginal distribution, i.e. the structure of the PC is changed. The second approach is beneficial if samples should be drawn from the marginalized PC because the sampling procedure remains the same, i.e. the PC is adapted to obtain the marginal distribution, not vice versa.

\paragraph{Conditioning}
Computing a conditional distribution $p(\mathbf{X}_1 | \mathbf{X}_2) = \frac{p(\mathbf{X})}{\int_{\mathbf{X}_1} p(\mathbf{X})}$ where $\mathbf{X}_1 \cup \mathbf{X}_2 = \mathbf{X}$ and $\mathbf{X}_1 \cap \mathbf{X}_2 = \emptyset$ is achieved by combining marginalization (denominator) and inference (numerator). Since inference is tractable for PCs in general and marginalization is tractable for decomposable PCs, conditioning is also tractable. 
%Conditioning essentially yields another PC over $\mathbf{X}_1$ for each evidence $\mathbf{X}_2 = \mathbf{x}_2$. 

\paragraph{Sampling}
Sampling in PCs is a top-down procedure and recursively samples a sub-tree, starting at the root. Each sum node $\SumNode$ holds a parameter vector $\mathbf{w}$ s.t. $\sum_{i=0}^{|\ch{\SumNode}|} \mathbf{w}_i = 1$. Based on the distribution induced by $\mathbf{w}$, one of the children of $\SumNode$ is sampled as a sub-tree. By decomposability, the scope of the children of a product node is non-overlapping; thus, sampling from a product node corresponds to sampling from all its child nodes. If a leaf node is reached, a sample is obtained from the distribution at that leaf.

\paragraph{Learning}
Learning PCs consists of two steps: Identifying the structure of the PC and learning the parameters of the PC. A common approach to learning both the structure and parameters is LearnSPN~\citep{gens2013learnSPN}. We employ LearnSPN to learn the PC after obtaining new data. The basic idea of LearnSPN is to split the data by alternating clustering (i.e., split the data along the sample dimension) and independence tests (i.e., split the data along the features dimension). In other words, the data matrix is split by rows (samples) and columns (features). Usually, rows are clustered when the independence test fails in splitting the features. Clusters correspond to sum nodes in the learned PC, while product nodes correspond to successfully passed independence tests (assessing that two subsets of features are statistically independent). The parameters (i.e., weights of sum nodes) are set proportional to the cluster sizes of clusters represented by the child nodes of a sum node. Leaf parameters are commonly defined via maximum likelihood estimation.

LearnSPN is efficient on small-to-medium datasets (e.g., $50$ variables and $1000$ samples), as it is based on K-Means and RDC for independence testing, resulting in a quadratic runtime complexity. Since we are interested in high sample efficiency in SMBO, we operate in the regime of small to medium-sized datasets.

\paragraph{Induced Tree Representation} Here, we briefly describe the induced tree representation from \citet{zhaoa16CollapsedVarInf}. We use this representation in our proofs.
\begin{defin}
    \label{def:induced_pc}
        Induced Trees~\citep{zhaoa16CollapsedVarInf}. Given a complete and decomposable PC $s$ over $\bm{\mathcal{H}} = \{H_1, \dots, H_n\}$, $\mathcal{T} = (\mathcal{T}_V, \mathcal{T}_E)$ is called an induced tree PC from $s$ if
    \begin{enumerate}
        \item $\Node \in \mathcal{T}_V$ where $\Node$ is the root of $s$.
        \item for all sum nodes $\SumNode \in \mathcal{T}_V$, exactly one child of $\SumNode$ in $s$ is in $\mathcal{T}_V$, and the corresponding edge is in $\mathcal{T}_E$.
        \item for all product node $\ProductNode \in \mathcal{T}_V$, all children of $\ProductNode$ in $s$ are in $\mathcal{T}_V$, and the corresponding edges in $\mathcal{T}_E$.
    \end{enumerate}
\end{defin}

We can use Def. \ref{def:induced_pc} to represent decomposable and complete PCs as mixtures~\citep{zhaoa16CollapsedVarInf}.

\begin{prop}[Induced Tree Representation]
    \label{prop:induced_tree}
        Let $\tau_s$ be the total number of induced trees in $s$. Then the output at the root of $s$ can be written as $\sum_{t=1}^{\tau_s} \prod_{(k, j) \in \mathcal{T}_{t E}} w_{k j} \prod_{i=1}^n p_t(H_i = \bm{\theta}_i)$, where $\mathcal{T}_t$ is the $t$-th unique induced tree of $s$ and $p_t(H_i)$ is a univariate distribution over $H_i$ in $\mathcal{T}_t$ as a leaf node.
\end{prop}

\section{Proofs}\label{app:proofs}
This section provides proofs for all propositions in the main paper.

\subsection{Order-Isomorphism between Kernels and DisCo}
\label{app:iso-order-kernel}

Here, we provide the proof that DisCo is order-isomorphic to symmetric kernels under certain parameterizations of PCs.

\proporderiso*

% \begin{prop} \label{proof:order-iso}
% Let $p$ be a PC defined over real-valued random variables $X_1, \dots, X_n$, where each $X_i$ is modeled by $m$ leaves. Each leaf $\Leaf_{ij}$ is characterized by a symmetric density $f$ with a location parameter $\theta_{ij}$ and a monotone cumulative distribution function (CDF) $\Phi$. Let $d: \bm{\mathcal{X}} \times \bm{\mathcal{X}} \rightarrow \mathbb{R}_{\geq 0}$ be a $p$-norm and $k: \bm{\mathcal{X}} \times \bm{\mathcal{X}} \rightarrow \mathbb{R}$ be a symmetric, strictly increasing bounded kernel. Assume that both the leaf densities and $k$ are strictly monotonic functions of $d$ up to constant factors, define $\mathbf{\Theta} = \{\bm{\theta}_j\}_{j=1}^m$. Then, for any two points $\mathbf{x}_1, \mathbf{x}_2 \in \bm{\mathcal{X}}$ and $\bm{\theta} \in \bm{\Theta}$, $\rho_p$ is order-isomorphic to the kernel $k(\cdot, \cdot)$, such that $d(\mathbf{x}_1, \bm{\theta}) > d(\mathbf{x}_2, \bm{\theta})$ implies both $(1 - \rho_p(\mathbf{x}_1; \bm{\theta})) > (1 - \rho_p(\mathbf{x}_2; \bm{\theta}))$ and $k(\mathbf{x}_1; \bm{\theta}) > k(\mathbf{x}_2; \bm{\theta})$.
% \end{prop}

\begin{proof}
    To show order-isomorphism between $k(\cdot)$ and $1 - \rho_p(\cdot)$, we exploit the induced tree representation from Eq. \ref{eq:induced_tree}. Further, we use that - by definition - $k$ is an affine transformation of $d$, thus $k$ preserves the order induced by $d$.

    We now have to show three things:
    \begin{enumerate}
        \item DisCo preserves the order of $d$ in all leaf nodes, i.e., $\rho_{\Leaf}$ has the same order as $d$.
        \item Since product nodes combine DisCo scores of children using Fisher's method~\citep{fisher1932statisticalmethods}, we have to show that Fisher is order-preserving w.r.t $d$, i.e., all $\rho_{\ProductNode}$ inherit the order induced by $d$.
        \item The final weighted sum in the root node of the induced tree representation preserves the order of incoming child DisCo scores, i.e., $\rho_p$ inherits the order of $d$.
    \end{enumerate}

    \textbf{$\rho_{\Leaf}$ is order-preserving.} We start by showing that $\rho_{\Leaf}$ preserves the order of $d$. Remember that each leaf $\Leaf_{i j}$ has a symmetric density $f_{i j}$ that is an affine transformation of $d$, i.e. $f_{i j}(\cdot) = a_{i j} \cdot d(\cdot)$ with $a_{i j} > 0$ and s.t. $\int f_{i j}(x; \theta_{i j})dx = 1$ where $\theta_{i j}$ parameterizes the density. Remember
    \begin{equation*}
        \rho_{\Leaf_{i j}}(x) = 2 \cdot \min\Big(\int_{-\infty}^x f_{i j}(x; \theta_{i j}), \int_{x}^{\infty} f_{i j}(x; \theta_{i j})\Big).
    \end{equation*}
    Since $f_{i j}$ is symmetric, where the symmetry point is given by its parametrization $\theta_{i j}$, $\rho_{\Leaf_{i j}}$ is also symmetric at $\theta_{i j}$ because the probability mass is exactly split in two halves at the symmetry point and $\min$ only selects which side we consider.
    Thus, for $x_1, x_2 \in \text{supp}(f_{i j})$ we have to consider two cases:
    \begin{enumerate}
        \item $x_1 < \theta_{i j}$ and $x_2 < \theta_{i j}$
        \item $x_1 < \theta_{i j}$ and $x_2 > \theta_{i j}$
    \end{enumerate}
    Starting with 1., since $d(x_1, \theta_{i j}) < d(x_2, \theta_{i j})$ is assumed, we find that
    \begin{equation*}
        \int_{-\infty}^{x_1} f_{i j}(x; \theta_{i j})dx > \int_{-\infty}^{x_2} f_{i j}(x; \theta_{i j})dx
    \end{equation*}
    and thus $\rho_{\Leaf_{i j}}(x_1) >\rho_{\Leaf_{i j}}(x_2)$, therefore $1 - \rho_{\Leaf_{i j}}(x_1) < 1 - \rho_{\Leaf_{i j}}(x_2)$.

    Proceeding with 2., consider
    \begin{align*}
        \rho_{\Leaf_{i j}}(x_1) = \int_{-\infty}^{x_1} f_{i j}(x; \theta_{i j})dx \\
        \rho_{\Leaf_{i j}}(x_2) = \int_{x_2}^{\infty} f_{i j}(x; \theta_{i j})dx.
    \end{align*}
    This is because the $\min$ operator always selects the minimum area under the density curve which spans from $-\infty$ to $x_1$ for $x_1$ and spans from $x_2$ to $\infty$ for $x_2$.
    Since $f_{i j}$ is symmetric and $d(x_1, \theta_{i j}) < d(x_2, \theta_{i j})$ is assumed, $\rho_{\Leaf_{i j}}(x_1)$ "collects`` more mass than $\rho_{\Leaf{i j}}(x_2)$, therefore $\rho_{\Leaf_{i j}}(x_1) >\rho_{\Leaf_{i j}}(x_2)$, and thus $1 - \rho_{\Leaf_{i j}}(x_1) < 1 - \rho_{\Leaf_{i j}}(x_2)$.

    \textbf{Fisher preserves order.} Next, we show that applying Fisher to incoming DisCo values at product nodes preserves the order between points $\mathbf{x}_1, \mathbf{x}_2 \in \text{supp}(\ProductNode)$ induced by $d$. Note that we now consider vectors instead of scalar values since product nodes model product distributions over multiple independent random variables. Remember that for a product $\ProductNode$, DisCo is computed as
    \begin{equation*}
        \rho_{\ProductNode}(\mathbf{x}) = \Big(1 - F_{\chi^2}\big(-2 \cdot \sum_{i=1}^j \text{ln}(\rho_i(\mathbf{x}_i)); 2j\big)\Big)
    \end{equation*}
    where $j$ is the number of children of $\ProductNode$. Again, we assume $d(\mathbf{x}_1, \bm{\theta}_{\ProductNode}) < d(\mathbf{x}_2, \bm{\theta}_{\ProductNode})$ where $\bm{\theta}_{\ProductNode}$ is the parameter vector of all $j$ leaf nodes connected to $\ProductNode$. Note that we can safely assume that $\ProductNode$ is only connected to exactly these $j$ children due to the induced tree representation. 
    Since $d$ is a $p$-norm, we have
    \begin{equation*}
        \Big(\sum_{i=1}^j | \mathbf{x}_{1, i} - \bm{\theta}_{\ProductNode_i} |^p\Big)^{1/p} < \Big(\sum_{i=1}^j | \mathbf{x}_{2, i} - \bm{\theta}_{\ProductNode_i} |^p\Big)^{1/p}
    \end{equation*}
    which implies 
    \begin{equation*}
        - 2 \cdot \sum_{i = 1}^j \text{ln}(\rho_i(\mathbf{x}_{1, i})) < - 2 \cdot \sum_{i = 1}^j \text{ln}(\rho_i(\mathbf{x}_{2, i}))
    \end{equation*}
    by 1. and the fact that $\text{ln}$ is a monotonic function. Since $F_{\chi^2}$ is a monotonically increasing function as well and bound in $[0 ,1]$, applying $1 - F_{\chi^2}(\cdot)$ yields $\rho_{\ProductNode}(\mathbf{x}_1) > \rho_{\ProductNode}(\mathbf{x}_2)$ and thus $1 - \rho_{\ProductNode}(\mathbf{x}_1) < 1 - \rho_{\ProductNode}(\mathbf{x}_2)$.

    \textbf{Root node preserves order.} Finally, let us proceed with showing that under our assumptions, the weighted sum of the root node preserves the order induced by its children (and thus $d$). Since we assume that $\mathbf{x}_1$ and $\mathbf{x}_2$ share the same infimum $\bm{\theta} \in \mathbf{\Theta}$, that is the same neighbor, the sum reduces to an affine transformation
    \begin{equation*}
        \prod_{(j, k) \in \mathcal{T}_t} w_{j, k} \cdot \rho_{\ProductNode_{j, k}}(\mathbf{x}).
    \end{equation*}
    Therefore, the order is preserved and $1 - \rho_{p}(\mathbf{x}_1) < 1 - \rho_{p}(\mathbf{x}_2)$ holds.
\end{proof}

\paragraph{Discussion of Assumptions}
While we acknowledge that some of our assumptions might be violated in practice, especially in mixed domains, this result shows an interesting and promising bridge between uncertainty estimation in PCs and kernel-based methods.

From a modeling perspective, pairing symmetric leaf distributions in the PC with a symmetric kernel is a natural and well-justified choice. In classical Bayesian optimization, symmetric kernels (such as Gaussian or Matérn) encode the fundamental stationary prior that information propagates uniformly in all directions from an observed point. By utilizing symmetric leaf distributions (e.g., isotropic Gaussians or symmetric categorical priors), the PC surrogate mirrors this isotropic inductive bias locally within its tractable density estimation architecture. This alignment ensures that "model surprise" (DisCo) scales monotonically with the distance from existing data clusters, providing a coherent uncertainty proxy.
While out of the scope of this paper, different PC parameterizations might allow for relaxed assumptions regarding the form of the kernel.

Beyond our theoretical analysis, we provide empirical evidence demonstrating that DisCo remains an effective uncertainty proxy even when the assumptions of Prop.~\ref{prop:order-iso} are relaxed—specifically within mixed or discrete search spaces. First, the regret analysis in App.~\ref{app:further_results} shows that DisCo-based guidance consistently identifies high-performing regions across diverse landscapes. Furthermore, Fig.~\ref{fig:mis_prior_detection} demonstrates that DisCo reliably discriminates between beneficial and misleading priors across continuous, mixed, and discrete domains. This indicates that the score provides a robust signal for distinguishing conformant from non-conformant samples, regardless of the underlying domain structure.

\subsection{Regret Analysis}
\label{app:regret_analysis}

In this section, we provide the regret analysis of DisCoMBO.

\begin{lem}[Coverage]\label{proof:coverage}
    Let $f$ be a function from a Reproducing Kernel Hilbert Space (RKHS) $\mathcal{H}_k$ such that kernel $k = \phi \circ \rho$ where $\phi$ is a mapping such that $k$ is a valid symmetric kernel. Assume a PC $p$ such that $p(y | \mathbf{x})$ is unbiased, $f(\mathbf{x}) - \mathbb{E}[p(y | \mathbf{x})]$ is sub-Gaussian and $y = f(\mathbf{x}) + \epsilon$ with $\epsilon \sim \text{SubGauss}(\sigma)$. Then, $\mathbb{E}[p(y | \mathbf{x})] + \beta_t \cdot \phi(\rho(\mathbf{x}))$ entails $f$ with probability $1 - \delta_t$, for a given $\beta_t \geq \sqrt{2\text{ln}(\frac{2}{\delta_t})}$.
\end{lem}

\begin{proof}
    For an RKHS we have that $|f(\mathbf{x})| = \langle f, k(\mathbf{x}, \cdot)\rangle \leq ||f||_k \cdot \sqrt{k(\mathbf{x}, \mathbf{x})}$.
    Given the existence of $\phi$, we have
    \begin{equation*}
        |f(\mathbf{x}) - \mathbb{E}[p(y | \mathbf{x})]| \leq ||f||_k \cdot \phi(\rho(\mathbf{x}))
    \end{equation*}
    Now, for $\beta_t$, we want that
    \begin{equation*}
        P(|f(\mathbf{x}) - \mathbb{E}[p(y | \mathbf{x})]| \leq \beta_t \cdot \phi(\rho(\mathbf{x}))) \leq \delta_t
    \end{equation*}
    holds. Since the error is assumed to be sub-Gaussian, we have
    \begin{equation*}
        P(|f(\mathbf{x}) - \mathbb{E}[p(y | \mathbf{x})]| \leq \beta_t \cdot \phi(\rho(\mathbf{x}))) \leq 2 \cdot \text{exp}\Big(- \frac{\beta_t \cdot \phi(\rho(\mathbf{x}))^2}{2 \phi(\rho(\mathbf{x}))^2}\Big).
    \end{equation*}
    We solve for $\beta_t$:
    \begin{align*}
        \delta_t = 2 \text{exp}(- \frac{\beta_t^2}{2}) \\
        \frac{\delta_t}{2} = \text{exp}(-\frac{\beta_t^2}{2}) \\
        \text{ln}(\frac{\delta_t}{2}) = - \frac{\beta_t^2}{2} \\
        \beta_t = \sqrt{2 \text{ln}(\frac{2}{\delta_t})}
    \end{align*}
\end{proof}

\propzeroregret*

% \begin{prop}[DisCoMBO is a Zero Regret Algorithm]
%     Assume some $f \in \mathcal{H}_k$ with $\mathcal{H}_k$ being a Reproducing Kernel Hilbert Space (RKHS) with squared exponential kernel $k = \phi \circ \rho$ s.t. $f: \bm{\mathcal{X}} \rightarrow \mathbb{R}$ and $\mathbf{x}^*$ being the optimum of $f$. Further assume a periodic schedule $s(t) \in \mathbb{R}_{\geq 0}$ that determines whether focus is set on exploitation or exploration, s.t. exploitation periods get longer over time, i.e. $\pi_t = c \cdot \pi_{t-1}$ for some $c > 1$ and period length $\pi_t$. Let the number of bins $K \rightarrow \infty$ and $p(y | \mathbf{x})$ be unbiased where $p$ is a PC learned on observations $\mathcal{D} = \{(\mathbf{x}_1, y_1), \dots, (\mathbf{x}_n, y_n) \}$. Then, DisCoMBO has sub-linear cumulative regret $R_T$, i.e. $\frac{R_T}{T} \rightarrow 0$ with $T \rightarrow \infty$.
% \end{prop}

\begin{proof}
    In order to prove zero-regret, we have to prove that DisCoMBO  (1) has sub-linear regret during exploration, (2) has sub-linear regret during exploitation, and (3) spends almost all iterations in exploitation with $T \rightarrow \infty$. We assume that the candidate pool is constructed by sampling from $p(\mathbf{x} | y^*)$ at each iteration.

    \textbf{DisCoMBO has sub-linear regret during exploration.} We start by showing that $R_T$ is sub-linear during exploration. With $K \rightarrow \infty$ and when $s(t) \rightarrow 0$, our bucketing mechanism defines a density over DisCo scores that follows a Beta-distribution $\text{Beta}(s(t), 1)$. Since $s(t) \rightarrow 0$, we concentrate all probability mass at $0$, i.e., minimum conformance scores (corresponding to maximizing uncertainty). By Prop. \ref{prop:order-iso}, $\rho$ is order-isomorphic to any symmetric kernel, allowing us to assume that there exists $\phi$ such that $k = \phi \circ \rho$ is a squared exponential kernel. We know that for all $\mathbf{x} \in \bm{\mathcal{X}}$ that
    \begin{equation*}
        f(\mathbf{x}) \leq \mathbb{E}[p(y | \mathbf{x})] + \beta_t \cdot \phi(\rho(\mathbf{x}))
    \end{equation*}
    holds for all $\mathbf{x}$ with probability $1 - \delta_t$ (see Lemma \ref{proof:coverage}) as well as
    \begin{equation*}
        f(\mathbf{x}^*) \leq \mathbb{E}[p(y | \mathbf{x}^*)] + \beta_t \cdot \phi(\rho(\mathbf{x}^*))
    \end{equation*}
    holds. Now, let $\bm{\mathcal{X}}' = \{\mathbf{x} \in \bm{\mathcal{X}} | \int_{\mathbf{x}} p(\mathbf{x} | y^*) > \epsilon_t \}$ the set of candidates with probability of at least $\epsilon_t$ of being sampled into the candidate pool. Intuitively, this defines an $\epsilon_t$-ball around the best observations. We know that $f$ is $L$-Lipschitz smooth (implied by RKHS assumption). Let $\mathbf{x}_t = \arg \max_{\mathbf{x}' \in \bm{\mathcal{X}}'} \phi(\rho(\mathbf{x}'))$ be the candidate maximizing DisCo. Then, 
    \begin{align*}
        & \mathbb{E}[p(y | \mathbf{x}_t)] + \beta_t \cdot \phi(\rho(\mathbf{x}_t)) \geq \mathbb{E}[p(y | \mathbf{x}^*)] + \beta_t \cdot \phi(\rho(\mathbf{x}^*)) \\
        \equiv & \mathbb{E}[p(y | \mathbf{x}^*_t) - L \epsilon_t] + \beta_t \cdot \phi(\rho(\mathbf{x}_t)) \geq \mathbb{E}[p(y | \mathbf{x}^*)] + \beta_t \cdot \phi(\rho(\mathbf{x}^*))
    \end{align*}
    holds with probability $1 - \delta_t$ (Lemma \ref{proof:coverage}). Following Lemma 5.2 from ~\citet{Srinivas2012UCB}, we can conclude that the regret $r_t$ at iteration $t$ is bounded by $L \epsilon_t + 2\beta_t\phi(\rho(\mathbf{x}))$. To obtain the cumulative regret bound, we note that $k$ can be assumed to be a squared exponential kernel; we know that the information gain $\gamma_T$ grows at a rate of $O(\text{log}^{d+1}(T))$ with $d$ being the search space dimension. Also, with growing coverage, we can shrink $\epsilon_t \rightarrow 0$ at the same rate. Therefore, we can again follow ~\citet{Srinivas2012UCB} and conclude that for all $T \geq 1$ we have
    \begin{equation*}
        \sum_{t=1}^T r_t^2 = (f(\mathbf{x}^*) - f(\mathbf{x}^*_t) + 2L \epsilon_t) \leq C\beta_T \gamma_T
    \end{equation*}
    by the Cauchy-Schwarz inequality in the exploration phase.

    \textbf{DisCoMBO has sub-linear regret during exploitation.} Next, we need to show that the sub-linear regret property also holds for the exploitation phase, i.e., when $s(t) \rightarrow \text{max}(s(t))$. Assuming we can choose a sufficiently high amplitude for $s(t)$, our Beta-distribution $\text{Beta}(s(t), 1)$ concentrates all of the probability mass at $1$, i.e., maximum conformance scores (corresponding to minimum uncertainty). Since we then only sample $\mathbf{x} \in \bm{\mathcal{X}}$ with $\rho(\mathbf{x}) \rightarrow 1$ from $p(\mathbf{x} | y^*)$, this amounts to minimizing uncertainty near the current incumbent. Since $f$ is L-Lipschitz smooth (a direct consequence from the RKHS), with long enough exploration we cover enough space s.t.
    \begin{equation*}
        f(\mathbf{x}^*) \leq f(\mathbf{x}_t^*) + L\epsilon
    \end{equation*}
    holds for some incumbent $\mathbf{x}_t^*$ at $t$ and some $\epsilon > 0$. Note that we can assume that we explore "enough`` for this condition to hold since we switch infinitely often between exploration and exploitation. Since we know that
    \begin{equation*}
        \mathbb{E}[p_t(y | \mathbf{x}_t^*)] \geq \mathbb{E}[p_t(y | \mathbf{x}^*)]
    \end{equation*}
    holds at $t$ (i.e., the model belief is that $\mathbf{x}_t^*$ is the optimum) and since $\beta_t \cdot \phi(\rho(\mathbf{x})) \rightarrow 0$ for all $\mathbf{x}$ as $t \rightarrow \infty$, we can conclude that
    \begin{equation*}
        \mathbb{E}[p_t(y | \mathbf{x}_t^*)] + \beta_t \cdot \phi(\rho(\mathbf{x}_t^*)) \geq \mathbb{E}[p_t(y | \mathbf{x}^*)] \beta_t \cdot \phi(\rho(\mathbf{x}^*)).
    \end{equation*}
    As before, following ~\citet{Srinivas2012UCB} implies that
    \begin{equation*}
        \sum_{t=1}^T r_t^2 = (f(\mathbf{x}^*) - f(\mathbf{x}^*_t) + 2L \epsilon_t) \leq C\beta_T \gamma_T
    \end{equation*}
    as $T \rightarrow \infty$ by Cauchy-Schwarz inequality.

    \textbf{DisCoMBO spends most time exploiting.} Finally, we show that DisCoMBO spends most of its time in exploitation. Since the period length is increased every iteration, $\pi_t \rightarrow \infty$ as $T \rightarrow \infty$. This shows that the exploitation period gets infinitely long.

    Combining these three results, we see that 
    \begin{equation*}
        \frac{\sqrt{\frac{1}{2} \cdot 2 \cdot T C \beta_T \gamma_T}}{T} \rightarrow 0 \quad \text{with} \quad  T \rightarrow \infty.
    \end{equation*}
\end{proof}

\textbf{Interpretation of DisCoMBO's Zero-Regret Property.} The zero-regret property of DisCoMBO shows that our algorithm ultimately finds the best solution in the limit of infinite compute budget and does not waste exploration budget in unpromising regions. However, note that it does not provide any finite-time convergence properties. Also, note that the exploration-exploitation trade-off in our current implementation is schedule-driven. While our proof considers periodic schedules, other ways to implement the exploration-exploitation trade-off are of course possible. For example, DisCoMBO can easily be adapted to fit into the classical acquisition function-based BO loop by dropping the schedule and using the DisCo score as an uncertainty signal, guiding the acquisition function's selection. 
However, our schedule-driven trade-off works similarly to a stochastic UCB version: The fact that we construct the candidate pool based on $p(\mathbf{x} | y^*)$ acts similarly to the mean-term in UCB, while maximizing/minimizing DisCo acts similarly to the exploration-exploitation trade-off in UCB.
Empirically, we show in \ref{sec:Experiments} that this strategy achieves performances on par with state-of-the-art BO algorithms, demonstrating that DisCo provides a useful uncertainty estimate for BO tasks.

\section{Experimental Details}\label{app:experiments}
Here, we present additional details of our empirical evaluation. The raw logs from our experiments are available at \url{https://doi.org/10.6084/m9.figshare.32183937}, and our code is available at \url{https://github.com/bewit/DisCoMBO}.

\subsection{Benchmarks}
\label{app:benchmarks}
Benchmarks serve as a critical component for the systematic evaluation of SMBO algorithms, providing a reproducible and computationally efficient framework for assessing performance in high-stakes or resource-intensive domains. A task within a benchmark is defined by a search space $\bm{\mathcal{X}}$ and an objective function $f: \bm{\mathcal{X}} \rightarrow \mathbb{R}$, such as maximizing annual energy production in wind farm layout design or minimizing validation error in hyperparameter optimization. To bypass the prohibitive costs and latency of real-world evaluations, benchmarks typically utilize three primary representations: tabular benchmarks, where a discrete subset of the search space is pre-evaluated and stored; surrogate benchmarks, which employ regression models to interpolate between observed points and enable optimization over continuous parameters; and simulation benchmarks, which leverage low-fidelity numerical proxies to capture complex system dynamics. By abstracting away the execution costs of physical experiments or large-scale model training, these benchmarks facilitate standardized, large-scale comparisons across diverse SMBO strategies under consistent budget and noise constraints.

In our experimental evaluation, we used the following benchmarks: JAHS~\citep{jahsbench201}, NAS-Bench-201~\citep{nasbench201}, HPO-B~\citep{arango2021hpoblargescalereproduciblebenchmark}, crossed barrel~\citep{gongora2020crossedbarrel}, and Floris~\citep{Fleming2020FLORIS}. We now briefly describe the characteristics of these benchmarks.

\textbf{JAHS.} JAHS aims to provide a reproducible benchmark for joint optimization of neural architectures and other hyperparameters on real-world tasks. It consists of a 14-dimensional search space (of which some dimensions, like the number of epochs, can also be treated as fidelities). The search space contains both discrete and continuous hyperparameters. Regarding tasks, JAHS provides training and evaluation statistics of models trained on CIFAR-10, Fashion-MNIST and Colorectal Histology. It further provides runtime statistics for each trained configuration (i.e., architecture and hyperparameter settings) in terms of the wall clock time. Since all configurations were trained on the same hardware, wall clock time serves as a proxy for how expensive a given configuration is. For a fair evaluation, we follow standard practice from the AutoML community and report regrets against wall clock time for each algorithm. All tasks are image classification tasks. See~\citep{jahsbench201} for more information.

\textbf{NAS-Bench-201.} NAS-Bench-201 provides a reproducible benchmark for neural architecture search. It comes with a dense encoding of architectures, leading to a 6-dimensional search space. NAS-Bench-201 provides validation accuracy, test accuracy and training time statistics for each architecture on the CIFAR-10, CIFAR-100, and Imagenet datasets. Like for JAHS, we follow standard practices from the AutoML community and report regret against wall clock time for a fair comparison of methods. In our experiments, we only used CIFAR-10. See~\citep{nasbench201} for more information.

\textbf{HPO-B.} HPO-B is a large-scale HPO benchmark based on a diverse set of OpenML tasks. It comes with 176 search spaces and 196 datasets. Since some search spaces have been evaluated on multiple datasets, many tasks are available. Most of the tasks are classification and regression tasks on different data modalities, such as tabular data or image data. In our evaluation, we used the credit-g and vehicle datasets (dataset IDs 31 and 9914), paired with search spaces over hyperparameters of a random search model (search space ID 6794) and a gradient boosting model (search space ID 6767). Both datasets are tabular datasets. We chose these since both provide many evaluations reported at OpenML (506k for credit-g and 31k for vehicle), allowing rigorous comparability. The search spaces we used contained discrete and continuous hyperparameters. The gradient boosting search space is defined over 17 hyperparameters, while the random forest search space is defined over 9 hyperparameters. For more information, refer to~\citep{arango2021hpoblargescalereproduciblebenchmark}.

\textbf{PV.} The crossed barrel benchmark is a mechanical design task where the objective is to maximize the toughness (energy absorption per unit volume) of a 3D-printed architected material. The structure consists of two parallel platforms connected by $n$ hollow, thin-walled columns. The geometry of these columns is defined by their outer radius ($r$), thickness ($t$), and a twist angle ($\theta$) between the platforms, which creates a non-linear, reentrant geometry. Performance is evaluated through physical mechanical testing, where the structure is compressed to measure its force-displacement response. The task is characterized by a complex, non-convex objective landscape where the highest toughness results from a balance between high strength and high ductility. See ~\citep{gongora2020crossedbarrel} for more details.

\textbf{Floris.} The Floris benchmark evaluates optimization performance in the design and control of wind power plants. The primary objective is to maximize the Annual Energy Production (AEP) or net revenue of a wind farm by mitigating wake-induced power losses. The design space includes the spatial coordinates of $N$ turbines (layout optimization) and their operational setpoints, such as yaw misalignment angles and induction factors (wake steering and control). The benchmark employs steady-state engineering wake models to simulate the complex, non-linear aerodynamic interactions where upstream turbines reduce the wind speed and increase turbulence for downstream units. This task is characterized by high-dimensional search spaces, significant multimodality due to the nonlinear nature of wake deficits, and the requirement to integrate over diverse wind roses (stochastic wind speeds and directions). See ~\citep{Fleming2020FLORIS} for more information.

\subsection{Search Space Extension of JAHS}
\label{app:jahs_extension}
To make the HPO problem on JAHS more challenging, we decided to extend the search space slightly as JAHS -- as a surrogate benchmark -- allows us to query hyperparameter values which were not tested explicitly in the benchmark. We defined three search spaces for JAHS which are presented in the following table.

\begin{table}[h]
\caption{\textbf{JAHS Search Space.} We define three versions of the JAHS search space, ranging from simpler to harder spaces.}
\resizebox{\textwidth}{!}{
\begin{tabular}{l|lll}
                 & S1                          & S2                          & S3                          \\ \hline
Activation       & {[}Mish, ReLU, Hardswish{]} & {[}Mish, ReLU, Hardswish{]} & {[}Mish, ReLU, Hardswish{]} \\
Learning Rate    & {[}1e-3, 1e0{]}             & {[}1e-3, 1e0{]}             & {[}1e-3, 1e0{]}             \\
Weight Decay     & {[}1e-5, 1e-2{]}            & {[}1e-5, 1e-2{]}            & {[}1e-5, 1e-2{]}            \\
Trivial Augment & {[}True, False{]}           & {[}True, False{]}           & {[}True, False{]}            \\
Op1              & 0-6                         & 0-6                         & 0-6                         \\
Op2              & 0-6                         & 0-6                         & 0-6                         \\
Op3              & 0-6                         & 0-6                         & 0-6                         \\
Op4              & 0-6                         & 0-6                         & 0-6                         \\
Op5              & 0-6                         & 0-6                         & 0-6                         \\
Op6              & 0-6                         & 0-6                         & 0-6                         \\
N                & 1-15                        & 1-11                        & 1-5                         \\
W                & 1-31                        & 1-23                        & 1-16                        \\
Epoch            & 1-200                       & 1-200                       & 1-200                       \\
Resolution       & 0-1                         & 0-1                         & 0-1                        
\end{tabular}
}
\end{table}

\newpage
\subsection{External Knowledge}
\label{app:interaction_definition}
To evaluate the robustness and adaptability of DisCoMBO, we define a structured suite of external knowledge profiles designed to simulate various expert-AI interaction scenarios. These profiles consist of both beneficial and misleading priors, characterized as probability distributions over the search space that favor specific configurations. A beneficial prior is constructed by identifying the best-performing configuration $\mathbf{x}^+$ from a random sample of $10^4$ evaluations and assigning it a sampling probability up to $1000$ times higher than the baseline for a subset of dimensions. Conversely, a misleading prior is centered on the worst-performing configuration $\mathbf{x}^-$. We emphasize strong priors to rigorously test the recovery mechanisms of DisCoMBO against baseline methods like $\pi$BO and DynaBO, which inherently favor high-strength priors in their selection policies. Our experimental design systematically explores four distinct interaction regimes: (1) sparse beneficial knowledge provided in the initial optimization stages to assess early-stage acceleration; (2) sparse beneficial knowledge introduced in later iterations to evaluate late-stage integration; (3) dense misleading knowledge covering a large majority of dimensions to challenge the algorithm’s ability to recover from deceptive guidance; and (4) a temporal sequence of contradictory priors that alternates between beneficial and misleading information. In the sparse scenarios, the prior influences fewer than 25\% of the total parameters, testing whether the model can effectively leverage localized information. In the misleading and sequential scenarios, we increase the dimensionality of the prior to demonstrate that DisCoMBO can reliably filter high-strength, low-accuracy information even when it dominates the input signal, thereby approximating the noisy and potentially inconsistent feedback loops found in real-world expert-in-the-loop systems.

\textbf{JAHS.} The following JSON code shows the priors given in our JAHS experiments.
\begin{verbatim}
[
    {
        "type": "bad",
        "intervention": {"Activation": 1, "LearningRate": 0.8201676371308472, "N": 15,
        "Op1": 3, "Op2": 4, "Op3": 1, "Op4": 2, "Resolution": 0.5096959403985494,
        "TrivialAugment": 0, "W": 14,
         "WeightDecay": 0.002697686639935806, "epoch": 10},
        "iteration": 5
    },
    {
        "type": "good",
        "intervention": {"N": 3, "W": 16, "Resolution": 1},
        "iteration": 15
    },
    {
        "type": "good",
        "intervention": null,
        "iteration": 20
    },
    {
        "type": "good",
        "kind": "dist",
        "intervention": {"N": {"dist": "cat", "parameters": 
        [1, 1, 1, 1e4, 1, 1, 1, 1, 1, 1, 1, 1, 1, 1, 1, 1]},
        "W": {"dist": "cat", "parameters": 
        [1, 1, 1, 1, 1, 1, 1, 1, 1, 1, 1, 1, 1, 1, 1, 1, 1e4]},
        "Resolution": {"dist": "uniform", "parameters": [0.98, 1.02]}},
        "iteration": 5
    }
]
\end{verbatim}

\textbf{NAS-Bench-201.} The following JSON code shows the priors provided in our experiments on NAS-Bench-201.

\begin{verbatim}
[
    {
        "type": "good",
        "kind": "point",
        "intervention": {"Op_0": 2, "Op_1": 2, "Op_2": 0},
        "iteration": 5
    },
    {
        "type": "bad",
        "kind": "point",
        "intervention": {"Op_0": 1, "Op_1": 2, "Op_2": 1},
        "iteration": 5
    },
    {
    
        "type": "good",
        "kind": "point",
        "intervention": null,
        "iteration": 20
    },
    {
        "type": "good",
        "kind": "dist",
        "intervention": {"Op_0": {"dist": "cat", "parameters": [1, 1, 1e4, 1, 1]},
                         "Op_1": {"dist": "cat", "parameters": [1, 1, 1e4, 1, 1]},
                         "Op_2": {"dist": "cat", "parameters": [1e4, 1, 1, 1, 1]}},
        "iteration": 5
    }
]
\end{verbatim}

\textbf{HPO-B.} The following shows the priors defined for our HPO-B experiments. Each of the three priors corresponds to one of the three tasks we tested from HPO-B. The order is 6767:31, 6794:31, 6794:9914. Note that we only applied beneficial priors in the case of HPO-B.

\begin{verbatim}
    {
        "type": "good",
        "kind": "dist",
        "intervention": {
            "eta": {"dist": "uniform", "parameters": [0.45, 0.55]}, 
            "subsample": {"dist": "uniform", "parameters": [0.55, 0.65]}, 
            "lambda": {"dist": "uniform", "parameters": [-150, 0]}, 
            "min_child_weight": {"dist": "uniform", "parameters": [-15, -10]}},
        "iteration": 1
    },
    {
        "type": "good",
        "kind": "dist",
        "intervention": {
            "num_trees": {"dist": "int_uniform", "parameters": [1690, 1720]}, 
            "mtry": {"dist": "int_uniform", "parameters": [30, 34]},
            "min_node_size": {"dist": "int_uniform", "parameters": [780, 800]}
        },
        "iteration": 1
    },
    {
        "type": "good",
        "kind": "dist",
        "intervention": {
            "num_trees": {"dist": "int_uniform", "parameters": [1690, 1720]}, 
            "mtry": {"dist": "int_uniform", "parameters": [280, 320]},
            "sample_fraction": {"dist": "uniform", "parameters": [0.5, 0.53]}
        },
        "iteration": 1
    }
\end{verbatim}

\textbf{PV.} The following shows the interactions defined for the task considered on crossed barrel.

\begin{verbatim}
        {
        "type": "good",
        "kind": "point",
        "intervention": {"theta": 142.46},
        "iteration": 5,
        "task": "crossed_barrel"
    },
    {
        "type": "bad",
        "kind": "point",
        "intervention": {"theta": 32.11, "n": 11.91, "r": 2.48},
        "iteration": 5,
        "task": "crossed_barrel"
    },
    {
        "type": "good",
        "kind": "dist",
        "intervention": {"theta": {"dist": "uniform", "parameters": [141, 144]}},
        "iteration": 5,
        "task": "crossed_barrel"
    },


    {
        "type": "good",
        "kind": "point",
        "intervention": {"theta": 142.46},
        "iteration": 15,
        "task": "crossed_barrel"
    },
    {
        "type": "bad",
        "kind": "point",
        "intervention": {"theta": 33.11, "n": 12.01, "r": 2.58},
        "iteration": 25,
        "task": "crossed_barrel"
    },
    {
        "type": "good",
        "kind": "point",
        "intervention": {"theta": 141.96},
        "iteration": 35,
        "task": "crossed_barrel"
    }
\end{verbatim}

\textbf{Floris.} The following shows the interactions defined for the task considered on Floris.

\begin{verbatim}
    {
        "type": "good",
        "kind": "point",
        "intervention": {
            "yaw_t000_deg": -10.83748690531058, 
            "yaw_t001_deg": 11.472652376984005, 
            "yaw_t002_deg": 13.89467307805834
        },
        "iteration": 5
    },
    {
        "type": "bad",
        "kind": "point",
        "intervention": {
            "yaw_t000_deg": -27.275959662325022, 
            "yaw_t001_deg": 1.4292151641649902, 
            "yaw_t002_deg": 26.63570488559511, 
            "yaw_t003_deg": -25.844124879037324, 
            "yaw_t004_deg": -18.74817789106919, 
            "yaw_t005_deg": -29.885699970885156, 
            "yaw_t006_deg": 9.70821606188883, 
            "yaw_t007_deg": -27.2264816726402, 
            "yaw_t008_deg": -21.44977989395875, 
            "yaw_t009_deg": 23.53069799422112, 
            "yaw_t010_deg": -8.328123875812214, 
            "yaw_t011_deg": -29.022544007776677, 
            "yaw_t012_deg": -15.702255207610008, 
            "yaw_t013_deg": -19.815213890373023, 
            "yaw_t014_deg": 23.418886185782853, 
            "yaw_t015_deg": -1.0822319543582601, 
            "yaw_t016_deg": 3.496271047466287
        },
        "iteration": 5    
    },
    {
        "type": "good",
        "kind": "dist",
        "intervention": {
                         "yaw_t000_deg": 
                         {"dist": "uniform", "parameters": [-11, -10]},
                         "yaw_t001_deg": 
                         {"dist": "uniform", "parameters": [10, 12]},
                         "yaw_t002_deg": 
                         {"dist": "uniform", "parameters": [12, 14]}
                        },
        "iteration": 5
    },

    {
        "type": "good",
        "kind": "point",
        "intervention": {
            "yaw_t000_deg": -10.83748690531058, 
            "yaw_t001_deg": 11.472652376984005, 
            "yaw_t002_deg": 13.89467307805834
        },
        "iteration": 15
    },
    {
        "type": "bad",
        "kind": "point",
        "intervention": {
            "yaw_t000_deg": -28.275959662325022, 
            "yaw_t001_deg": 1.6292151641649902, 
            "yaw_t002_deg": 27.53570488559511, 
            "yaw_t003_deg": -25.844124879037324, 
            "yaw_t004_deg": -18.74817789106919, 
            "yaw_t005_deg": -29.885699970885156, 
            "yaw_t006_deg": 9.70821606188883, 
            "yaw_t007_deg": -27.2264816726402, 
            "yaw_t008_deg": -21.44977989395875, 
            "yaw_t009_deg": 23.53069799422112, 
            "yaw_t010_deg": -8.328123875812214, 
            "yaw_t011_deg": -29.022544007776677, 
            "yaw_t012_deg": -15.702255207610008, 
            "yaw_t013_deg": -19.815213890373023, 
            "yaw_t014_deg": 23.418886185782853, 
            "yaw_t015_deg": -1.0822319543582601, 
            "yaw_t016_deg": 3.496271047466287
        },
        "iteration": 25    
    },
    {
        "type": "good",
        "kind": "point",
        "intervention": {
            "yaw_t000_deg": -11.23748690531058, 
            "yaw_t001_deg": 12.072652376984005, 
            "yaw_t002_deg": 14.09467307805834
        },
        "iteration": 35
    }
\end{verbatim}

\subsection{Further Results \& Ablations}
\label{app:further_results}

In this section, we describe further results we obtained during our experimental evaluation. 

\paragraph{Regret Analysis}
To examine the performance of DisCoMBO, we provide plots showing the normalized regret over time (either cumulative wall-clock time if provided by benchmarks or iterations). 

Fig.~\ref{fig:full_regret_standard} compares DisCoMBO (\solidline{DisCoMBO}) with all our baselines in the default SMBO setting, that is, without any external priors to guide optimization. We find that DisCoMBO ranks first on \textbf{5/9} tasks, is competitive to top-tier baselines on \textbf{3/9} tasks, and only falls slightly short on one task (Floris). Overall, we find that DisCoMBO effectively identifies high-performing solutions. Notably, we achieve our goal of improving upon purely generative PC-based SMBO (IBO-HPC; \solidline{IBO}) on \textbf{9/9} tasks.

Fig.~\ref{fig:full_regret_early} compares the regret over time between DisCoMBO (\dashdottedline{DisCoMBO}) and our baselines $\pi$Bo, IBO-HPC, and DynaBO when a beneficial prior is provided after 5 iterations. We find that DisCoMBO outperforms or matches all baselines on \textbf{7/9} tasks, demonstrating that it successfully retains or even improves upon the efficacy of incorporating priors of our baselines, especially IBO-HPC.

Fig.~\ref{fig:full_regret_recover} demonstrates that DisCoMBO (\dashedline{DisCoMBO}) successfully recovers from misleading priors on all tasks. In \textbf{4/6} tasks, we recover faster than our baselines, demonstrating robustness. We evaluate recovery on JAHS, NAS-Bench-201, Floris, and PV as a representative subset of tasks.

Fig.~\ref{fig:full_regret_many} demonstrates that DisCoMBO (\dashedline{DisCoMBO}) effectively handles sequences of priors, even if they contain contradicting information. We first provide a misleading prior, followed by a beneficial prior, followed by a misleading prior, and finally a beneficial prior again. It can be observed that DisCoMBO leverages beneficial priors effectively while not suffering from misleading ones. Priors are provided at iterations 5, 15, 25, and 35. Note that $\pi$BO is not considered here as it assumes that priors are available before optimization.

\paragraph{Ranking Analysis}
Another key metric when analyzing the performance of an SMBO algorithm is to consider the average rank over time. This shows at which phase algorithms excel and in which phases they flatten out.

Fig.~\ref{fig:full_ranking_standard} shows the average rank of each algorithm over time across 50 seeds when no prior is provided. We find that in terms of final rank, DisCoMBO (\solidline{DisCoMBO}) wins on \textbf{6/9} tasks while being competitive to top-tier baselines on the remaining tasks except Floris (see App.~\ref{app:floris_discussion} for more information). On most tasks, DisCoMBO seems to find good solutions quickly as it ranks high early on all AutoML tasks. For PV (crossed-barrel), DisCoMBO seems to benefit from exploration, indicating that our exploration samples diverse candidates, helping DisCoMBO to exploit this knowledge effectively in later iterations. This confirms the effectiveness of DisCoMBO across our diverse task suite.

Fig.~\ref{fig:full_ranking_early} shows the average rank of each algorithm over time across 50 seeds when a beneficial prior is provided after 5 iterations. We find that in terms of final rank, DisCoMBO (\solidline{DisCoMBO}) wins on \textbf{7/9} tasks. On all tasks except Floris, DisCoMBO seems to effectively leverage provided priors. This confirms that DisCoMBO retains, or even improves, the efficacy of IBO-HPC in incorporating provided priors across our diverse task suite.

Fig.~\ref{fig:full_ranking_recover} shows that DisCoMBO (\solidline{DisCoMBO}) achieves the top final rank on \textbf{4/6} tasks. Across all benchmarks except Floris and PV (crossed-barrel), the framework successfully recovers from a misleading prior given after 5 iterations, demonstrating that DisCoMBO retains or enhances the targeted knowledge integration capabilities of purely generative approaches like IBO-HPC. Similar results can be obtained in Fig.~\ref{fig:full_ranking_many} when a sequence of priors is provided.

\paragraph{Relative Improvement}
Since one of the key goals of DisCoMBO is to improve upon existing PC-based SMBO to close the gap between PC-based and traditional SMBO, we show the relative improvement of DisCoMBO and our baselines against IBO-HPC. Fig.~\ref{fig:full_diff2pc_standard}, \ref{fig:full_diff2pc_early}, \ref{fig:full_diff2pc_recover}, and \ref{fig:full_diff2pc_many} show that DisCoMBO achieves significant improvement across all tasks when no prior is provided. If priors are provided, regardless of whether they are beneficial, misleading, or both, DisCoMBO improves upon IBO-HPC and most baselines in the majority of tasks. This demonstrates that DisCo effectively serves as an uncertainty measure, effectively guiding our exploration and exploitation, and it shows that DisCo handles priors well in various settings.

\paragraph{Critical Difference Analysis}
Our critical difference diagrams (see Fig.~\ref{fig:cd_standard_app}, \ref{fig:cd_early_app}, \ref{fig:cd_recover_app}, and \ref{fig:cd_many_app} reveal that DisCoMBO ranks first throughout all settings, namely: Standard SMBO (without priors), SMBO with beneficial priors, SMBO with misleading priors, and SMBO with multiple, contradicting priors. Especially noticeable is that DisCoMBO is statistically indistinguishable from established state-of-the-art BO baselines like HEBO. This demonstrates the efficacy of our uncertainty-guided exploration-exploitation strategy as well as the robustness and flexibility of DisCoMBO in diverse settings. We used the Wilcoxon signed-rank test at $p=0.05$ with a multiple-testing correction.

\paragraph{Ablation and Sensitivity Analysis of Hyperparameters}
Since DisCoMBO has a few hyperparameters as well, we provide an ablation on the three most critical ones ($\beta$, $n$, and $c$). $\beta$ is a tuple and controls the level of exploration and exploitation. High $\beta$ values mean high exploration, while $\beta \rightarrow 0$ means full exploitation. $c$ controls the cycle length of each exploration and exploitation phase. The parameter $n$ controls the number of candidate samples generated, which are then rated by DisCo and from which one is selected for evaluation.

Fig.~\ref{fig:ablation_jahs} shows 9 different settings of these parameters. It can be seen that DisCoMBO is robust to hyperparameter choices, demonstrating that it achieves high performance even with non-optimal hyperparameter settings.

\paragraph{Misleading Prior Detection}
As illustrated by the ROC AUC curves in Fig.~\ref{fig:mis_prior_detection}, we evaluate the isolating performance of DisCoMBO’s (\solidline{DisCoMBO}) prior detection layer against DynaBO's (\dashedline{dynabo}) heuristic safeguard after 5, 10, and 15 optimization iterations. Empirically, DisCoMBO consistently outperforms DynaBO across all \textbf{5/5} benchmark tasks, demonstrating a robust capability to reliably discriminate between beneficial and misleading external guidance in \textbf{4/5} cases. The singular exception to this strong discriminative performance occurs on the Floris task. Because the underlying PC density estimator occasionally struggles to model the unique and highly specialized topological properties of the Floris objective landscape, the resulting conformance scores become less distinctive, consequently degrading detection accuracy in this specific domain (see App.~\ref{app:floris_discussion} for an extended discussion).

\paragraph{DisCoMBO on JAHS}
DisCoMBO exhibits exceptionally strong performance on the JAHS benchmark (see Fig.~\ref{fig:full_regret_standard}). We attribute this efficacy to the inductive biases of Mixed Sum-Product Networks (MSPNs), the specific Probabilistic Circuit parameterization utilized in our evaluation. MSPNs are structurally engineered to handle heterogeneous domains by natively integrating continuous and discrete random variables without requiring artificial encodings or continuous relaxations.

This architectural alignment allows DisCoMBO to accurately capture the complex joint distribution and search space semantics inherent to mixed-variable domains. Consequently, the learned surrogate maintains high fidelity across the structural and numerical boundaries of the JAHS search space, facilitating a highly efficient traversal of the optimization landscape.

\paragraph{DisCoMBO on Floris}
\label{app:floris_discussion}
On the Floris benchmark, DisCoMBO yields significant performance gains over the purely generative IBO-HPC baseline but does not fully close the gap to state-of-the-art discriminative baselines (see Fig.~\ref{fig:full_regret_standard}). This behavior is fundamentally rooted in the architectural contrast between joint density estimators and specialized regression models, as well as in the properties of the search space. Floris' search space is a fully continuous search space governed by smooth target topologies. While models explicitly designed for regression naturally exploit surface continuity, Probabilistic Circuits natively function as discrete latent variable models. Consequently, standard PC parameterizations tend to partition continuous manifolds into coarse, localized densities, introducing a discretization bias that risks omitting high-performing regions during generative candidate sampling.

Despite this architectural constraint, DisCoMBO consistently outpaces IBO-HPC. We argue that this resilience stems directly from the formal uncertainty semantics introduced by the DisCo score. Because DisCo behaves analogously to symmetric kernels on continuous domains, it effectively injects a smoothness bias into the framework's uncertainty estimates. Even when the underlying PC density representation lacks continuous regression accuracy, the kernel-like properties of DisCo provide a well-calibrated exploration-exploitation signal. This robust mechanism mitigates the coarse discretization of the surrogate, allowing DisCoMBO to navigate smooth continuous topologies far more effectively than heuristics driven by pure likelihood sampling.

\paragraph{DynaBO on JAHS}
During our experiments, we observed that DynaBO~\citep{fehring2026dynabo} occasionally struggled to converge on specific landscapes, most notably on JAHS (CIFAR-10) (see Fig.~\ref{fig:full_regret_early}, \ref{fig:full_regret_recover}, and \ref{fig:full_regret_many}). Although we expected DynaBO's performance to closely parallel that of $\pi$BO, this discrepancy prompted us to consult directly with the authors of the original framework. Despite a collaborative review of the implementation, the precise root cause underlying this localized behavior could not be definitively isolated or resolved. Given that DynaBO performs as expected across the other evaluated benchmarks and given that we used the implementation of \citet{fehring2026dynabo}, we suspect this anomaly stems either from an undiscovered misconfiguration unique to the JAHS environment or from a subtle interplay between DynaBO’s safeguard mechanisms and the specific topology of the JAHS search space.

\paragraph{DisCoMBO with broader priors} To provide a fair and comparable analysis, the prior definitions in our main results reused the prior definitions of \citet{seng2025ihpo}. Since these priors are quite narrow, we also provide a speed-up analysis for broader priors. It can be seen that also under broader priors, DisCoMBO clearly benefits from the provided priors (see Tab. \ref{tab:speedup_results}). The broader priors were defined by doubling the variance of the prior for continuous variables compared to the priors in App. \ref{app:interaction_definition}. For discrete parameters, the configurations were upweighted by a factor of $3$ instead of $1e^4$.

\paragraph{DisCoMBO vs. GP+EI} To strengthen our results further, we add the final regret of GP+EI and compare it against DisCoMBO. Tab.~\ref{tab:gp_ei_discombo} shows that DisCoMBO outperforms GP+EI in \textbf{6/8} cases, demonstrating its effectiveness on various tasks. NASBench201 was excluded due to time constraints (GP+EI runs did not finish within a reasonable time).

\paragraph{DisCo in UCB acquisition function} We can use DisCo in a UCB acquisition function by setting $\mu(x)$ as the standardized expected conditional performance $\text{std}(\mathbb{E}[y|x])$ for a configuration $x$ and $\sigma(x)$ as the $\lambda$-scaled inverse DisCo score, i.e. $\mathrm{UCB}_{\mathrm{DisCo}}(x) = \mu(x) + \lambda (1 - \rho(x))$. We used UCB to score $N$ samples from $p(x | y^*)$ and $M$ samples randomly drawn from the search space, and selected the $K$ highest values (with $K$ = $L$ = 5 as reported in App.~\ref{app:hp_ibo}). Tab.~\ref{tab:ucb_vs_scheduling} shows that the UCB-version is on-par with our originally reported results, which is expected given that the acquisition function is directly derived from our underlying model and the DisCo score. We suspect that this perspective may ease the integration with existing BO frameworks. However, please note that careful tuning of $\lambda$ and extensive analysis of the acquisition function's properties is left for future work.

\begin{figure}[t]
    \centering
    \includegraphics[width=1.0\linewidth]{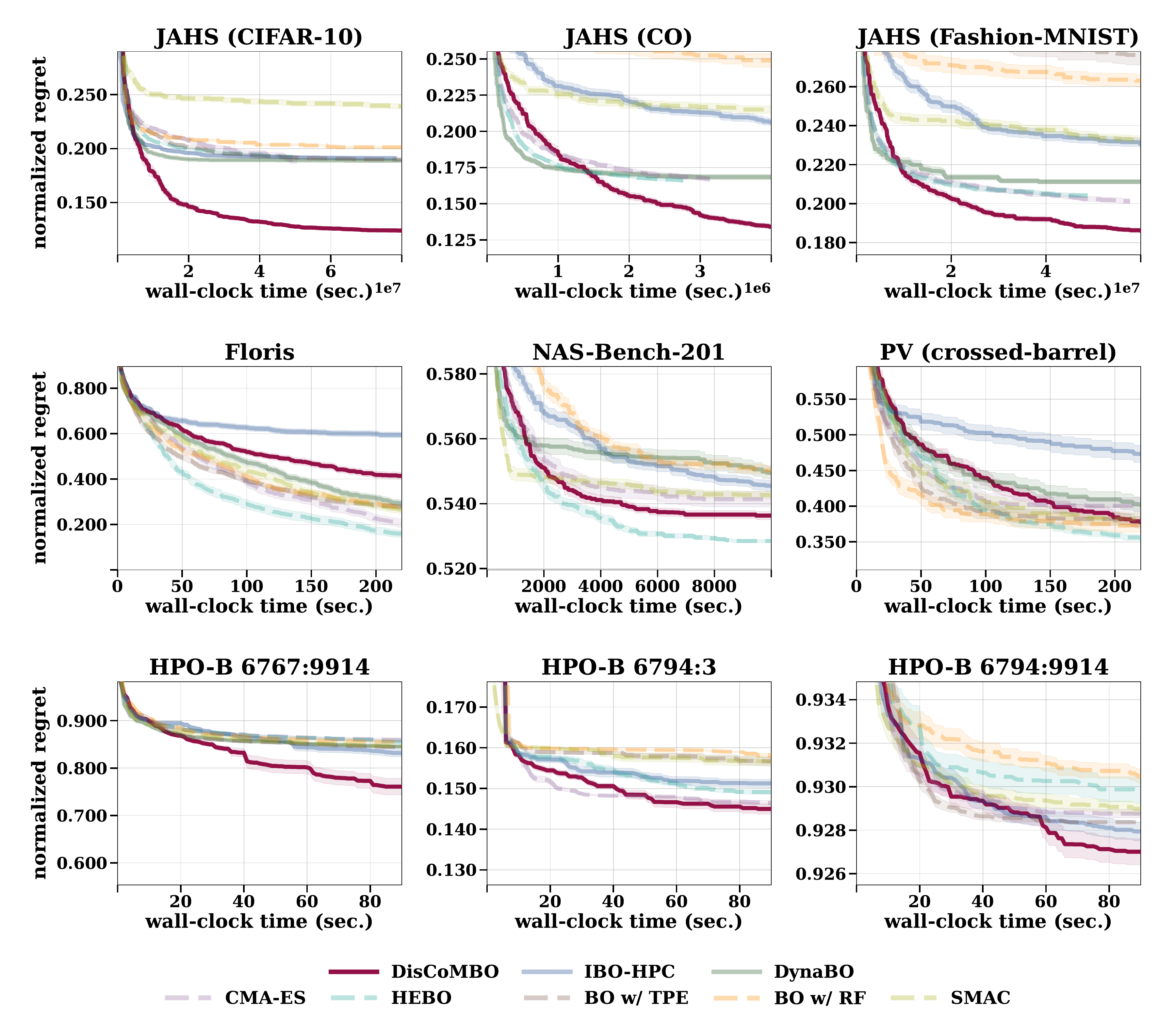}
    \caption{\textbf{DisCoMBO achieves top-tier performance without external priors.} Without guidance with external priors, DisCoMBO (\solidline{DisCoMBO}) achieves top-tier performance on \textbf{8/9} tasks. On Floris, we still significantly improve upon IBO-HPC (\solidline{IBO}), however, DisCoMBO cannot fully catch up with top-tier baselines (see App.~\ref{app:further_results}) for further discussion.}
    \label{fig:full_regret_standard}
\end{figure}
\begin{figure}
    \centering
    \includegraphics[width=1.0\linewidth]{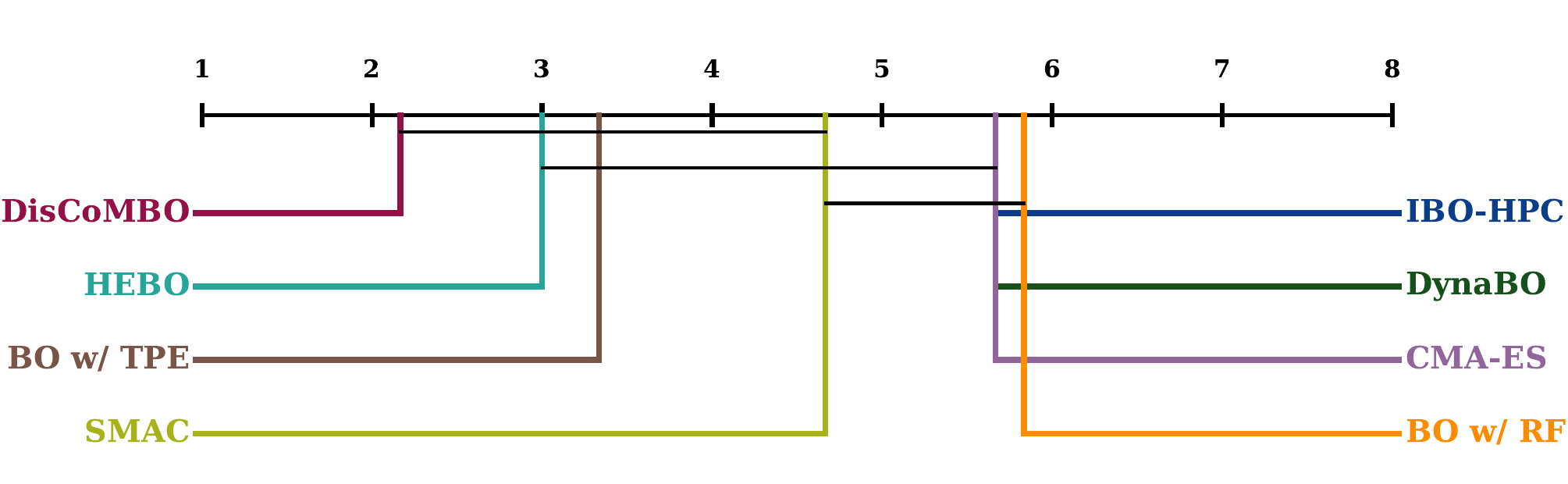}
    \caption{\textbf{DisCoMBO is competitive to top-tier SMBO methods without priors.} DisCoMBO ranks \textbf{first} across 9 tasks when \textbf{no prior} is given, being statistically indistinguishable from strong baselines like HEBO. DisCoMBO significantly outperforms IBO-HPC, its PC-based counterpart, lacking a principled exploitation-exploration trade-off. Assessment is based on Wilcoxon signed-rank test ($p=0.05$), corrected for multiple tests.}
    \label{fig:cd_standard_app}
\end{figure}
\begin{figure}[t]
    \centering
    \includegraphics[width=1.0\linewidth]{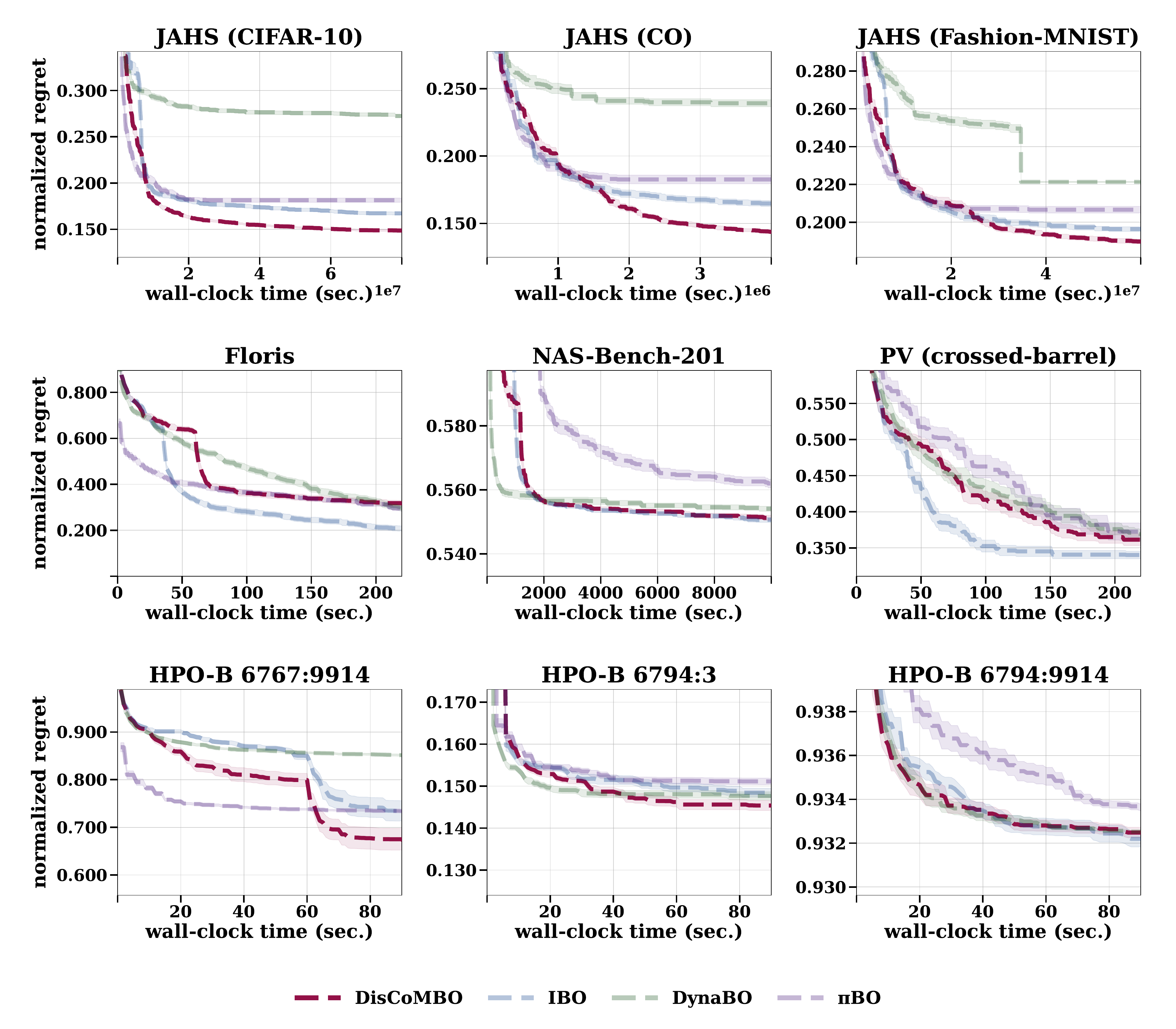}
    \caption{\textbf{DisCoMBO retains knowledge integration efficacy of generative PC-based SMBO.} When provided with external priors early in optimization, DisCoMBO (\dashedline{DisCoMBO}) retains the efficacy of IBO-HPC (\dashdottedline{IBO}) in leveraging that provided knowledge to converge faster and to better solutions. We outperform baselines on \textbf{7/9} tasks.}
    \label{fig:full_regret_early}
\end{figure}
\begin{figure}
    \centering
    \includegraphics[width=1.0\linewidth]{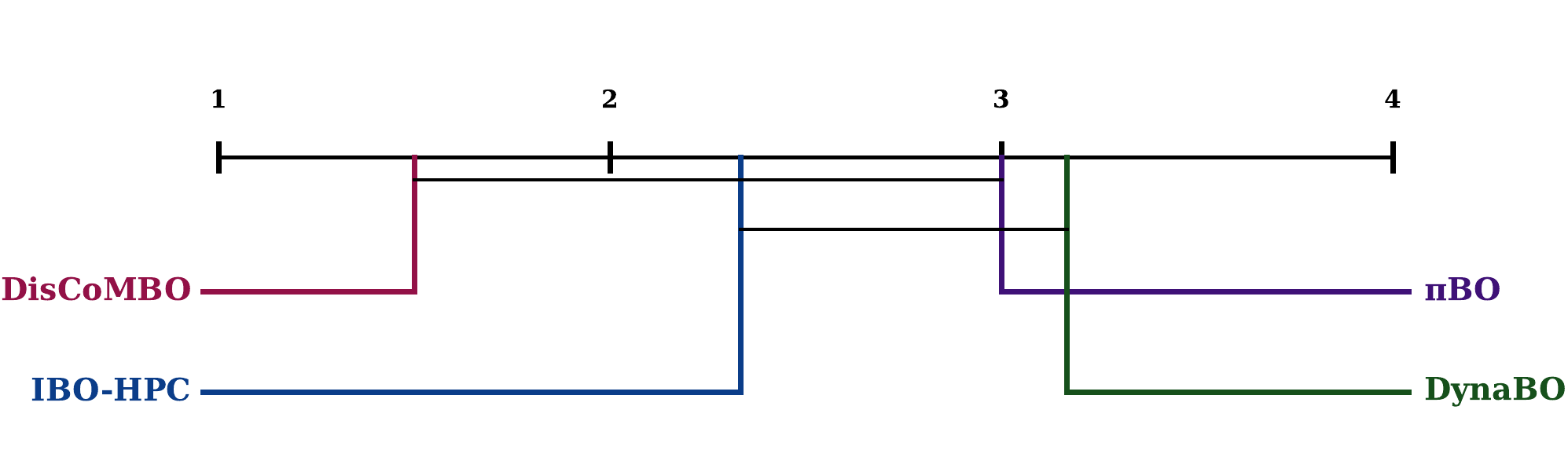}
    \caption{\textbf{DisCoMBO effectively incorporates priors into optimization.} DisCoMBO ranks \textbf{first} across 9 tasks when a prior is given, being statistically indistinguishable to IBO-HPC and $\pi$BO. This also demonstrates that DisCoMBO retains the efficacy of IBO-HPC in integrating external priors into the optimization process. Assessment is based on Wilcoxon signed-rank test ($p=0.05$) and corrected for multiple tests.}
    \label{fig:cd_early_app}
\end{figure}
\begin{figure}
    \centering
    \includegraphics[width=1.0\linewidth]{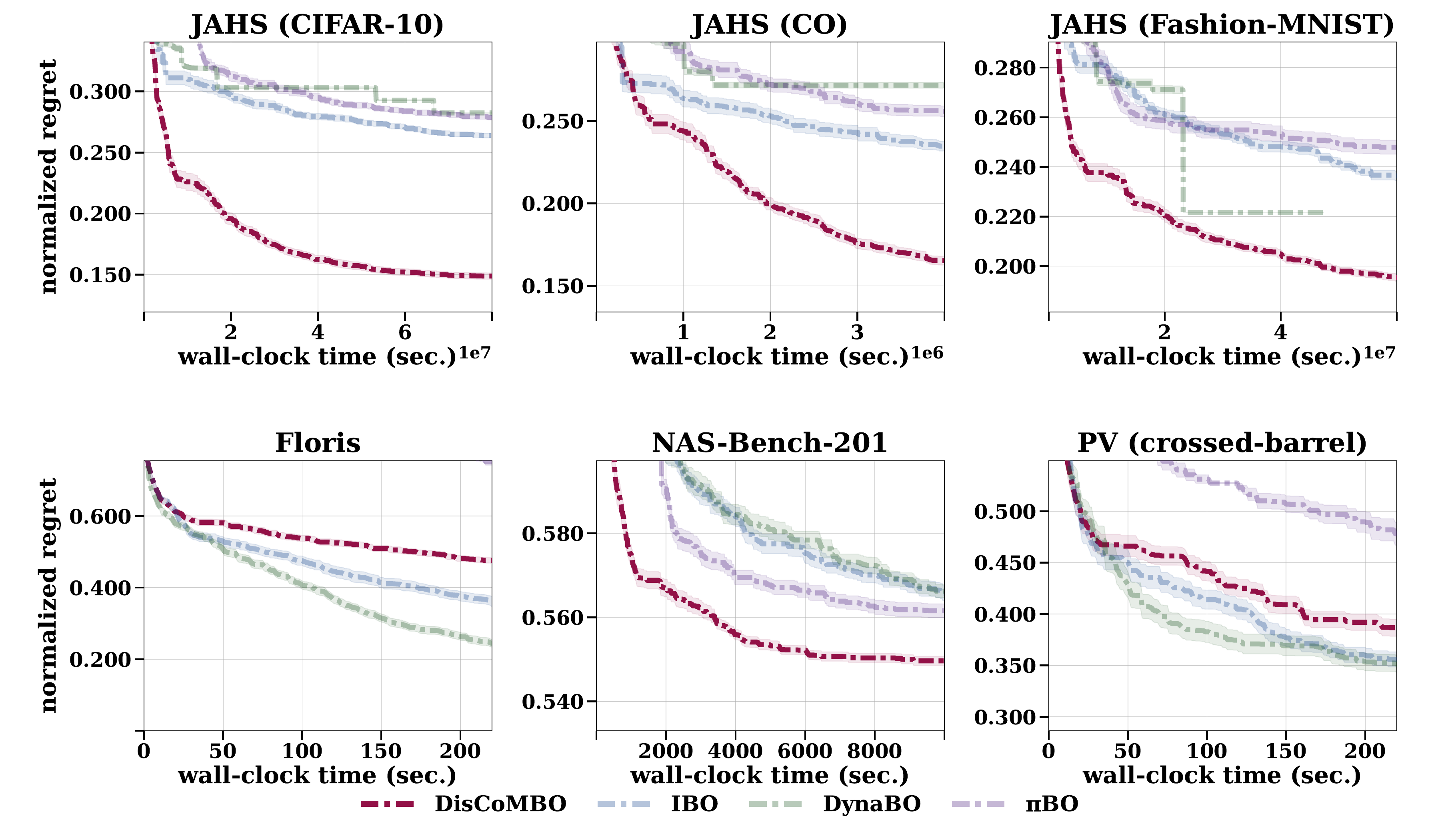}
    \caption{\textbf{DisCoMBO effectively recovers from misleading priors.} DisCoMBO (\dashedline{DisCoMBO}) successfully recovers from misleading priors on all tasks. In \textbf{4/6} tasks, we recover faster than our baselines, demonstrating robustness. We evaluate recovery on JAHS, NAS-Bench-201, Floris, and PV as a representative subset of tasks. On Floris, $\pi$BO was not able to recover from the misleading prior.}
    \label{fig:full_regret_recover}
\end{figure}
\begin{figure}
    \centering
    \includegraphics[width=1.0\linewidth]{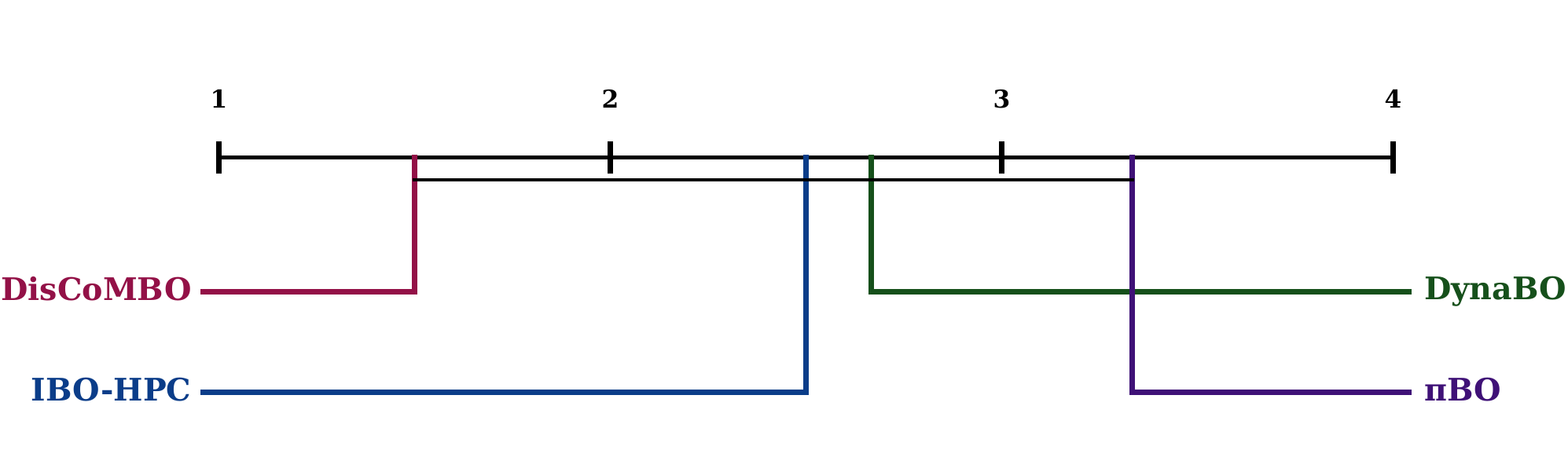}
    \caption{\textbf{DisCoMBO effectively recovers from misleading priors.} DisCoMBO ranks \textbf{first} across 9 tasks when \textbf{a misleading} prior is given. While DisCoMBO is statistically indistinguishable to IBO-HPC and DynaBO, the results show that DisCoMBO retains or even improves upon the robustness of all competing methods. Assessment is based on Wilcoxon signed-rank test ($p=0.05$) and corrected for multiple tests.}
    \label{fig:cd_recover_app}
\end{figure}
\begin{figure}
    \centering
    \includegraphics[width=1.0\linewidth]{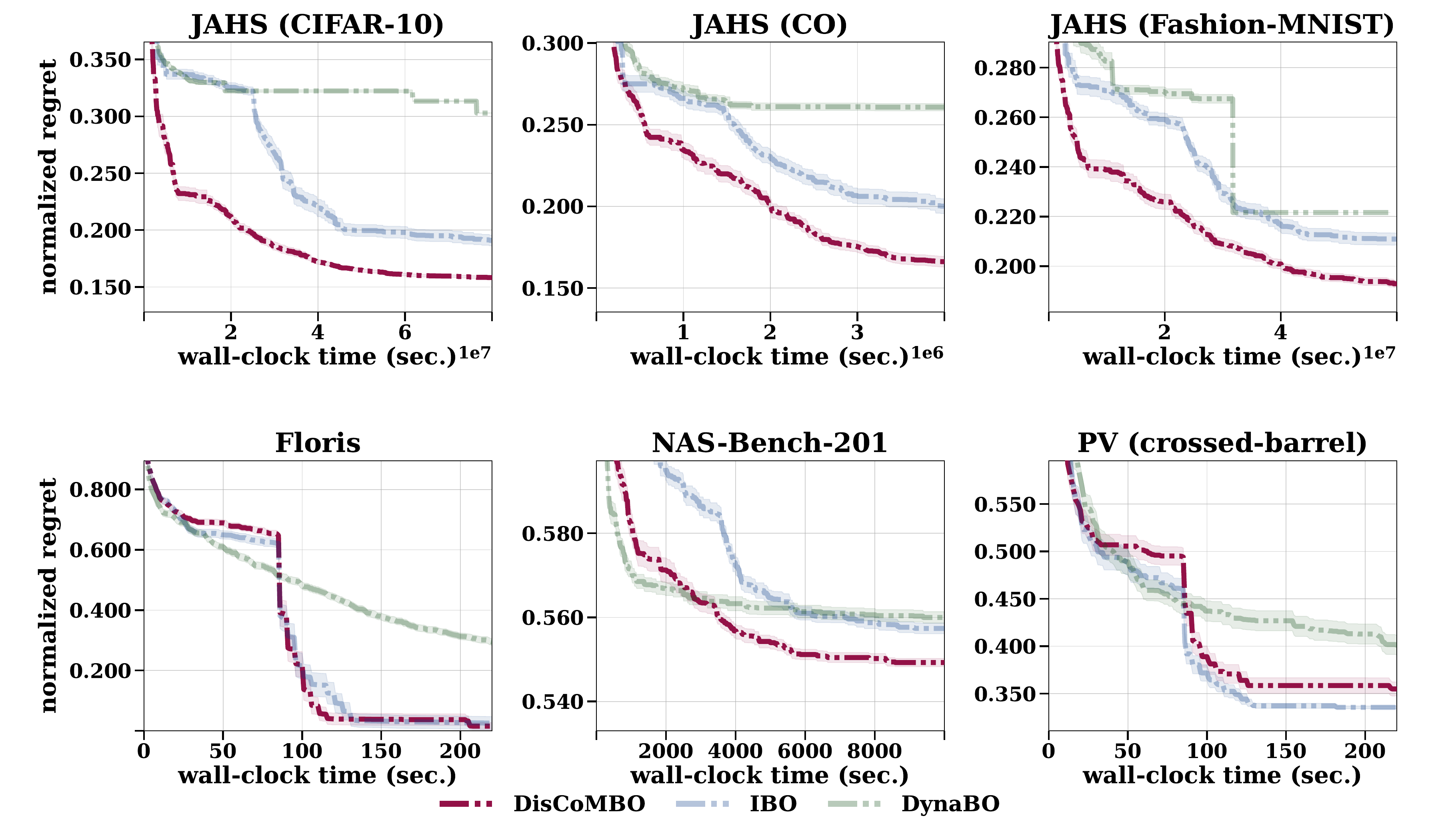}
    \caption{\textbf{DisCoMBO effectively handles sequences of priors.} We demonstrate that DisCoMBO (\dashedline{DisCoMBO}) effectively handles sequences of priors, even if they contain contradicting information. We first provide a misleading prior, followed by a beneficial prior, followed by a misleading prior, and finally a beneficial prior again. It can be observed that DisCoMBO leverages beneficial priors effectively while not suffering from misleading ones. $\pi$BO does not allow for multiple priors.}
    \label{fig:full_regret_many}
\end{figure}
\begin{figure}
    \centering
    \includegraphics[width=1.0\linewidth]{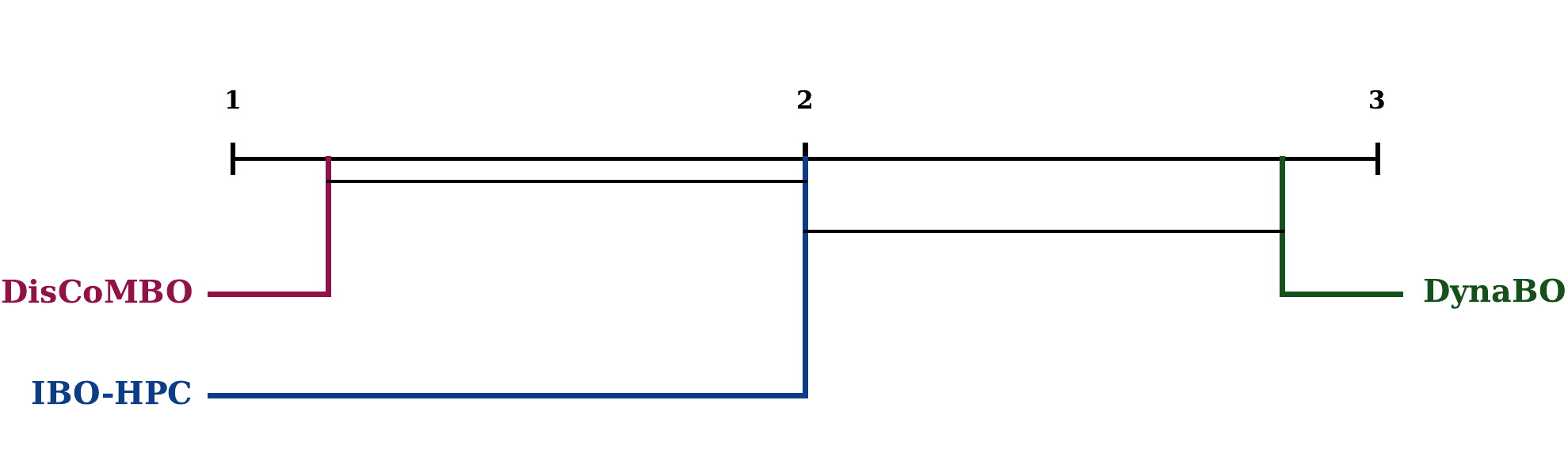}
    \caption{\textbf{DisCoMBO effectively handles sequences of priors.} DisCoMBO ranks \textbf{first} across 9 tasks when \textbf{a sequence} of priors is given. This demonstrates that DisCoMBO successfully handles multiple, even contradicting priors. Assessment is based on Wilcoxon signed-rank tests ($p=0.05$) and corrected for multiple tests.}
    \label{fig:cd_many_app}
\end{figure}

\begin{figure}
    \centering
    \includegraphics[width=1.0\linewidth]{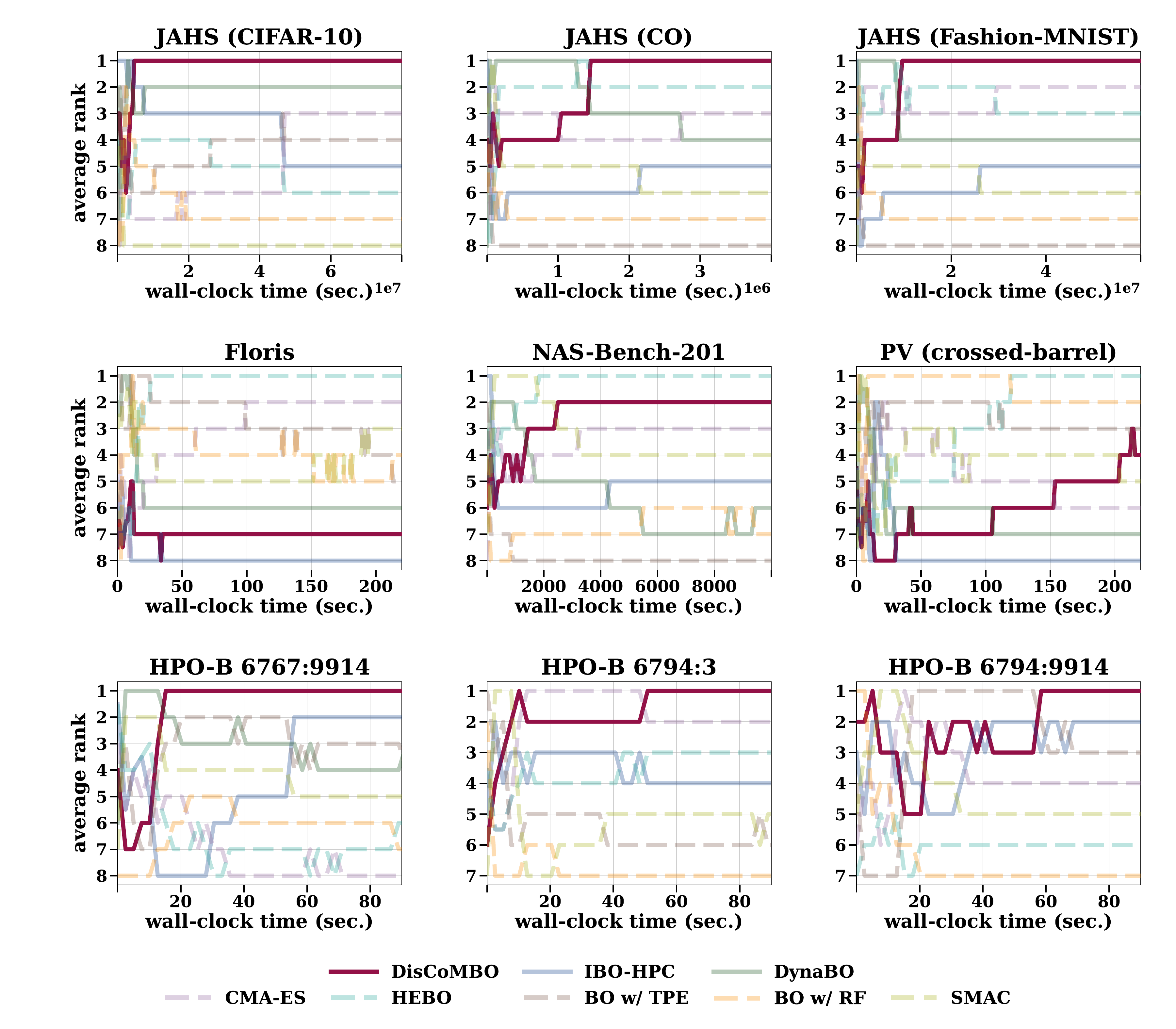}
    \caption{\textbf{Average algorithmic rank over time across 50 independent seeds.} DisCoMBO (\solidline{DisCoMBO}) achieves the top final rank on 6/9 tasks while remaining highly competitive on the remaining landscapes, with the exception of Floris (see App.~\ref{app:floris_discussion}). On the AutoML suite, DisCoMBO quickly identifies high-quality solutions early in the optimization trajectory. On PV (crossed-barrel), it benefits from structured exploration to effectively exploit the search space in later iterations.}
    \label{fig:full_ranking_standard}
\end{figure}
\begin{figure}
    \centering
    \includegraphics[width=1.0\linewidth]{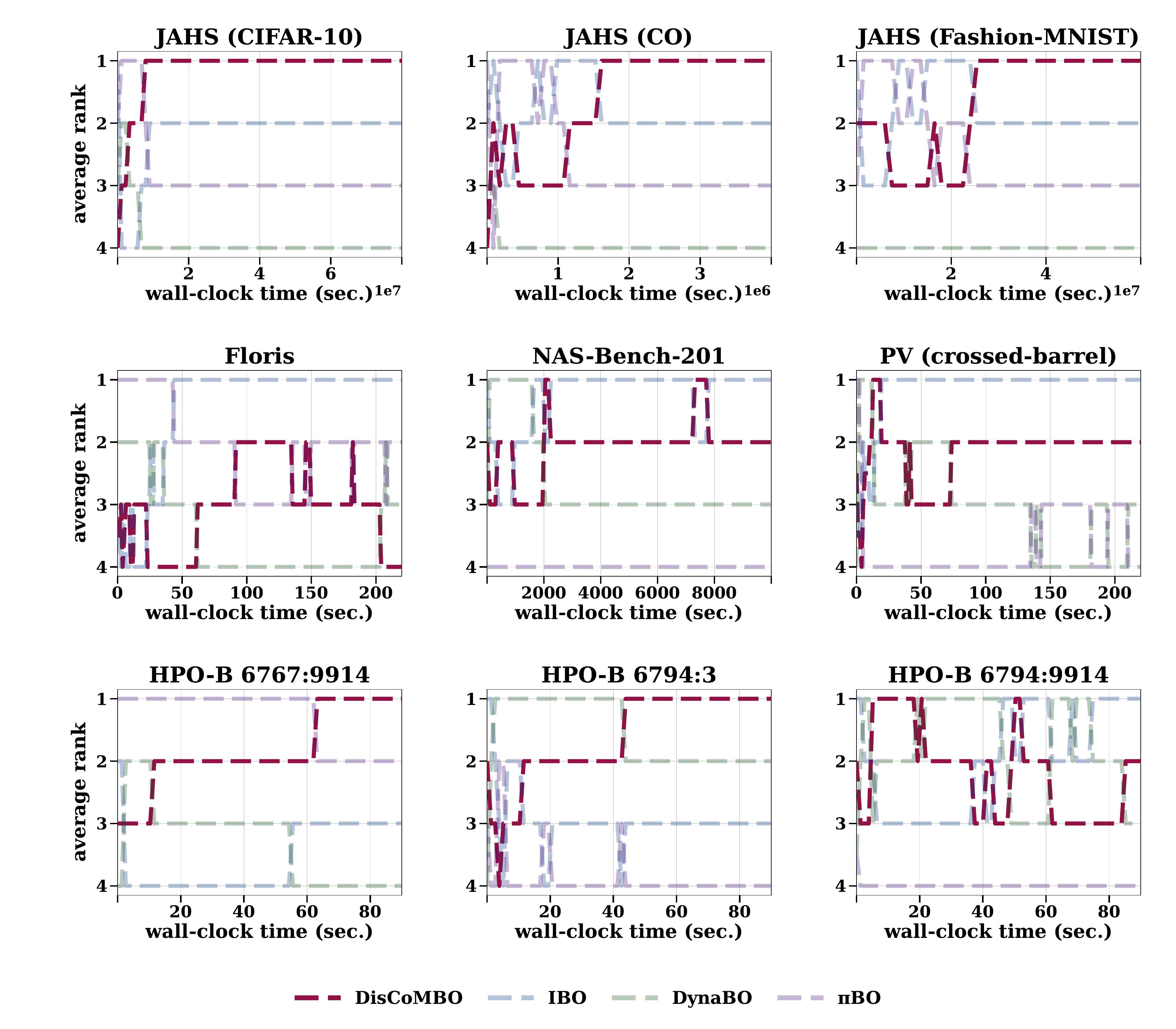}
    \caption{\textbf{Average algorithmic rank over time given a beneficial prior at iteration 5 (averaged across 50 seeds).} DisCoMBO (\dashedline{DisCoMBO}) achieves the top final rank on \textbf{7/9} tasks. Across all benchmarks except Floris, the framework successfully capitalizes on the injected prior, demonstrating that DisCoMBO retains or enhances the targeted knowledge integration capabilities of purely generative approaches like IBO-HPC.}
    \label{fig:full_ranking_early}
\end{figure}
\begin{figure}
    \centering
    \includegraphics[width=1.0\linewidth]{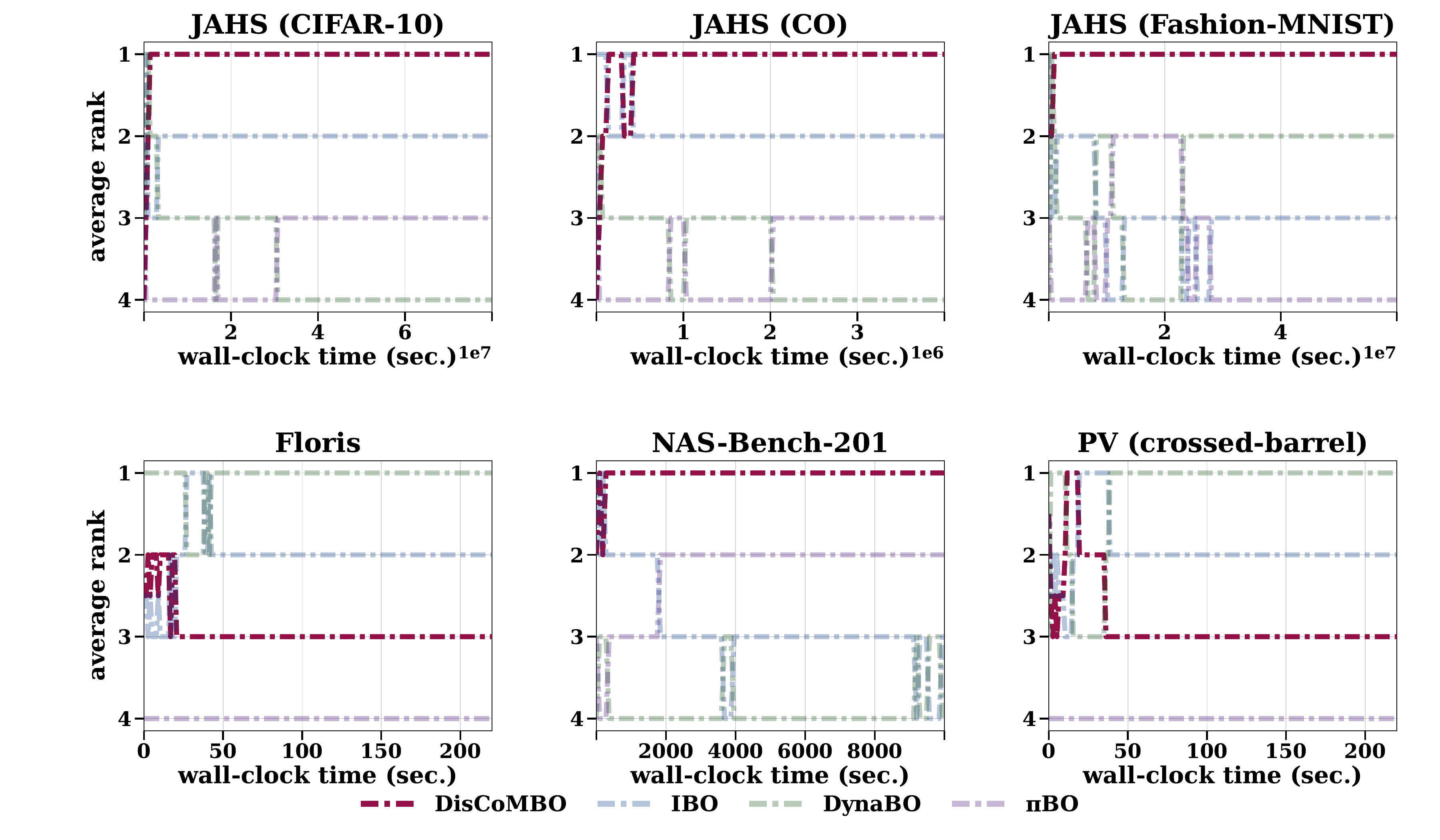}
    \caption{\textbf{Average algorithmic rank over time given a misleading prior at iteration 5 (averaged across 50 seeds).} DisCoMBO (\dashdottedline{DisCoMBO}) achieves the top final rank on \textbf{4/6} tasks. Across all benchmarks except Floris and PV (crossed-barrel), the framework successfully recovers from a misleading prior given after 5 iterations, demonstrating that DisCoMBO retains or enhances the targeted knowledge integration capabilities of purely generative approaches like IBO-HPC.}
    \label{fig:full_ranking_recover}
\end{figure}
\begin{figure}
    \centering
    \includegraphics[width=1.0\linewidth]{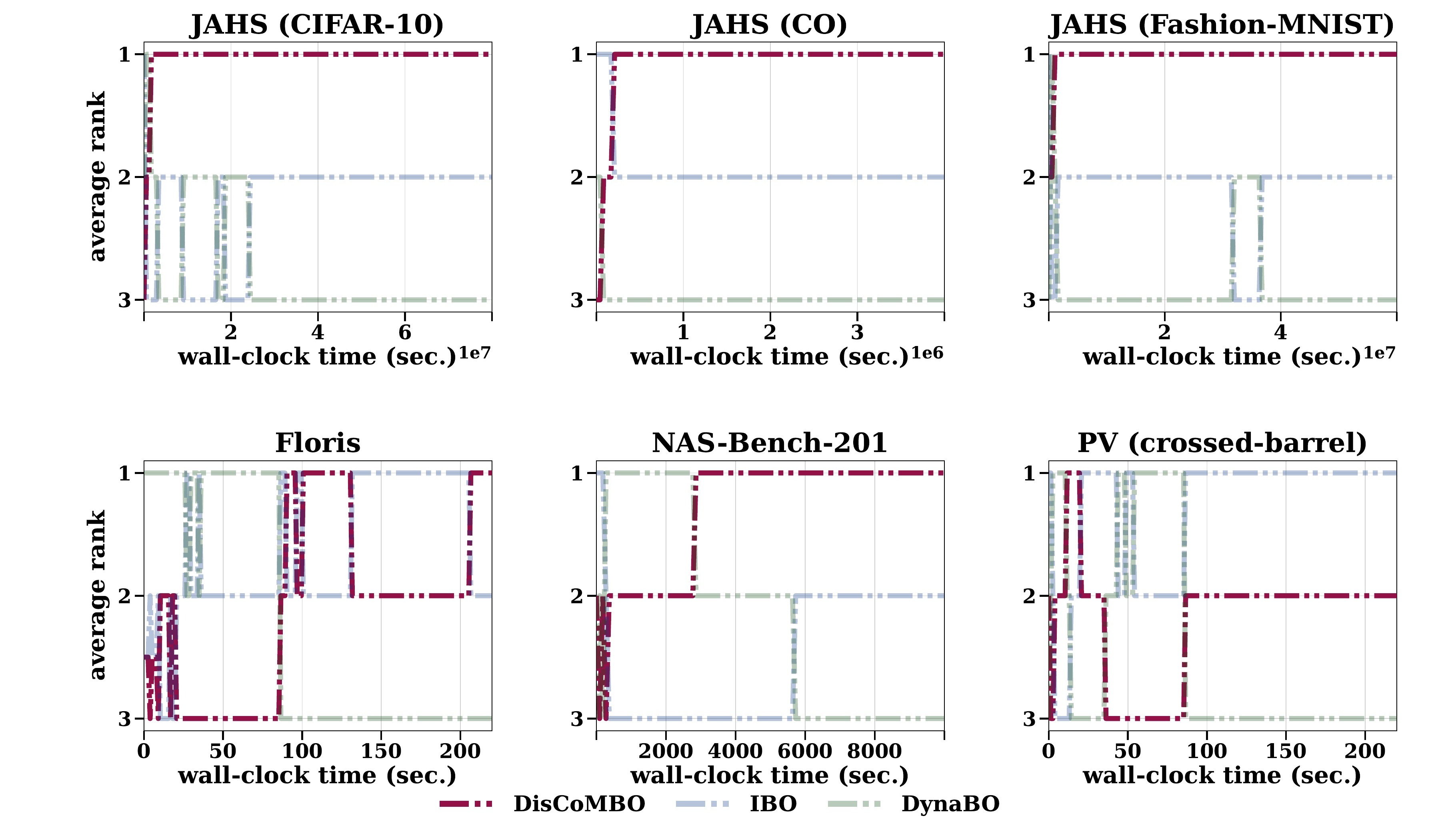}
    \caption{\textbf{Average algorithmic rank over time given a sequence of priors (averaged across 50 seeds).} DisCoMBO (\dashdottedline{DisCoMBO}) achieves the top final rank on \textbf{5/6} tasks. Across all benchmarks except PV (crossed-barrel), the framework successfully recovers from a misleading prior given after 5 iterations, demonstrating that DisCoMBO retains or enhances the targeted knowledge integration capabilities of purely generative approaches like IBO-HPC.}
    \label{fig:full_ranking_many}
\end{figure}

\begin{figure}
    \centering
    \includegraphics[width=1.0\linewidth]{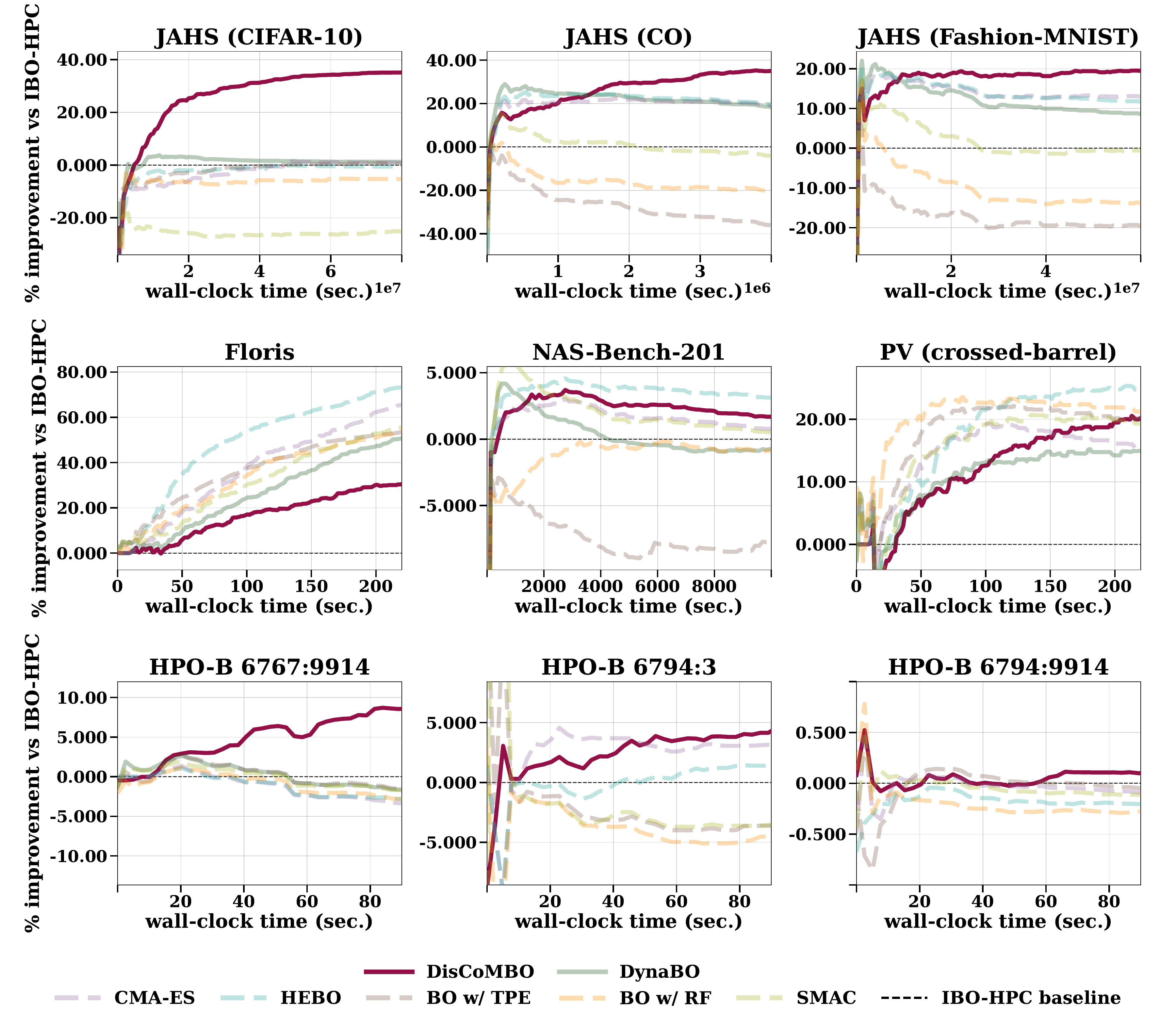}
    \caption{\textbf{Relative improvement of DisCoMBO against IBO-HPC over time without priors.} We observe that DisCoMBO (\solidline{DisCoMBO}) consistently achieves a significant improvement over its fully generative counterpart IBO-HPC on \textbf{8/9} tasks, demonstrating the effectiveness of our exploration-exploitation strategy driven by DisCo.}
    \label{fig:full_diff2pc_standard}
\end{figure}
\begin{figure}
    \centering
    \includegraphics[width=1.0\linewidth]{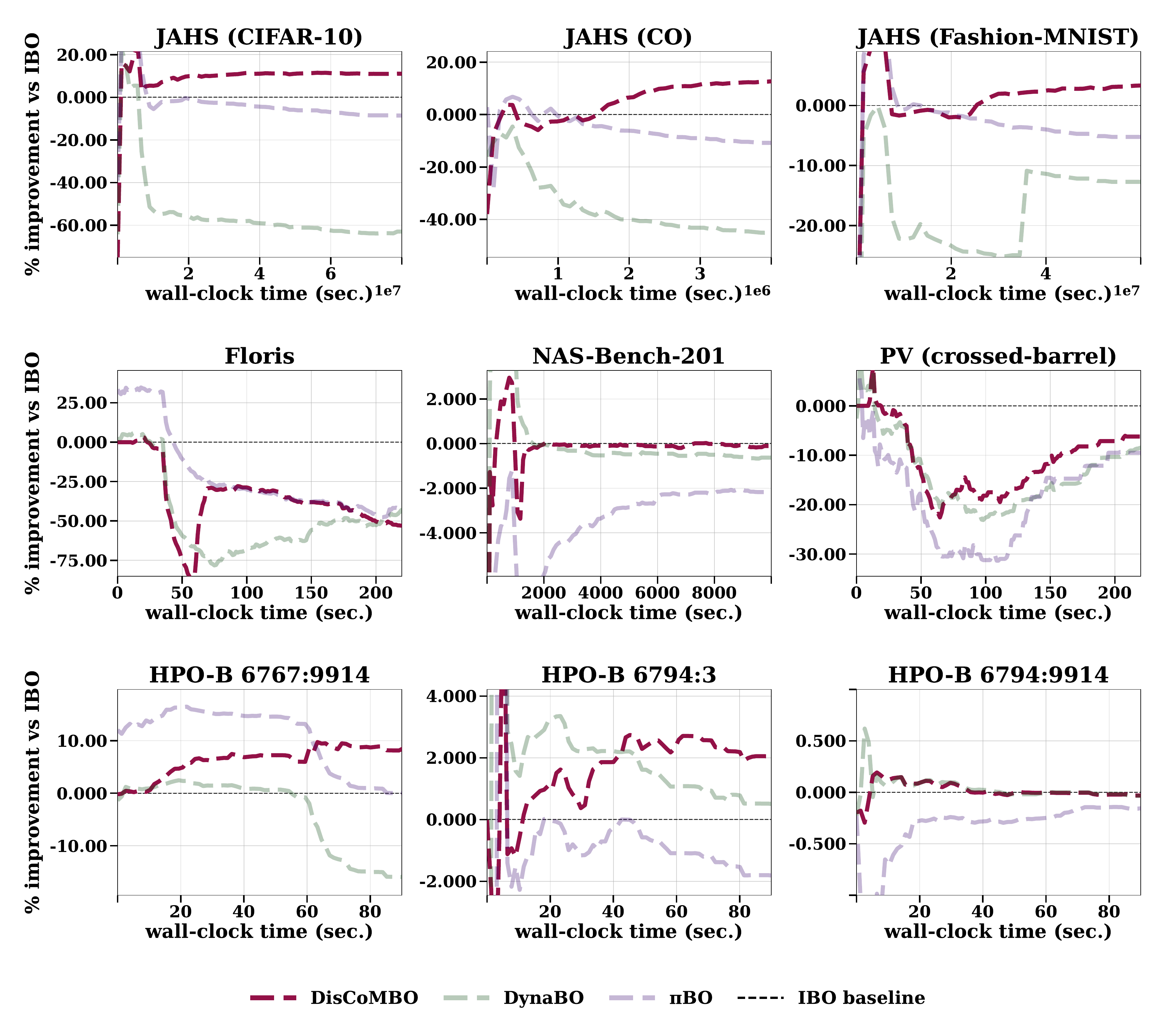}
    \caption{\textbf{Relative improvement of DisCoMBO against IBO-HPC over time with beneficial prior at 5 iterations.} We observe that DisCoMBO (\dashedline{DisCoMBO}) consistently achieves a significant improvement over its fully generative counterpart IBO-HPC or matches its performance on \textbf{7/9} tasks when beneficial priors are provided early during optimization, demonstrating that DisCoMBO effectively incorporates priors.}
    \label{fig:full_diff2pc_early}
\end{figure}
\begin{figure}
    \centering
    \includegraphics[width=1.0\linewidth]{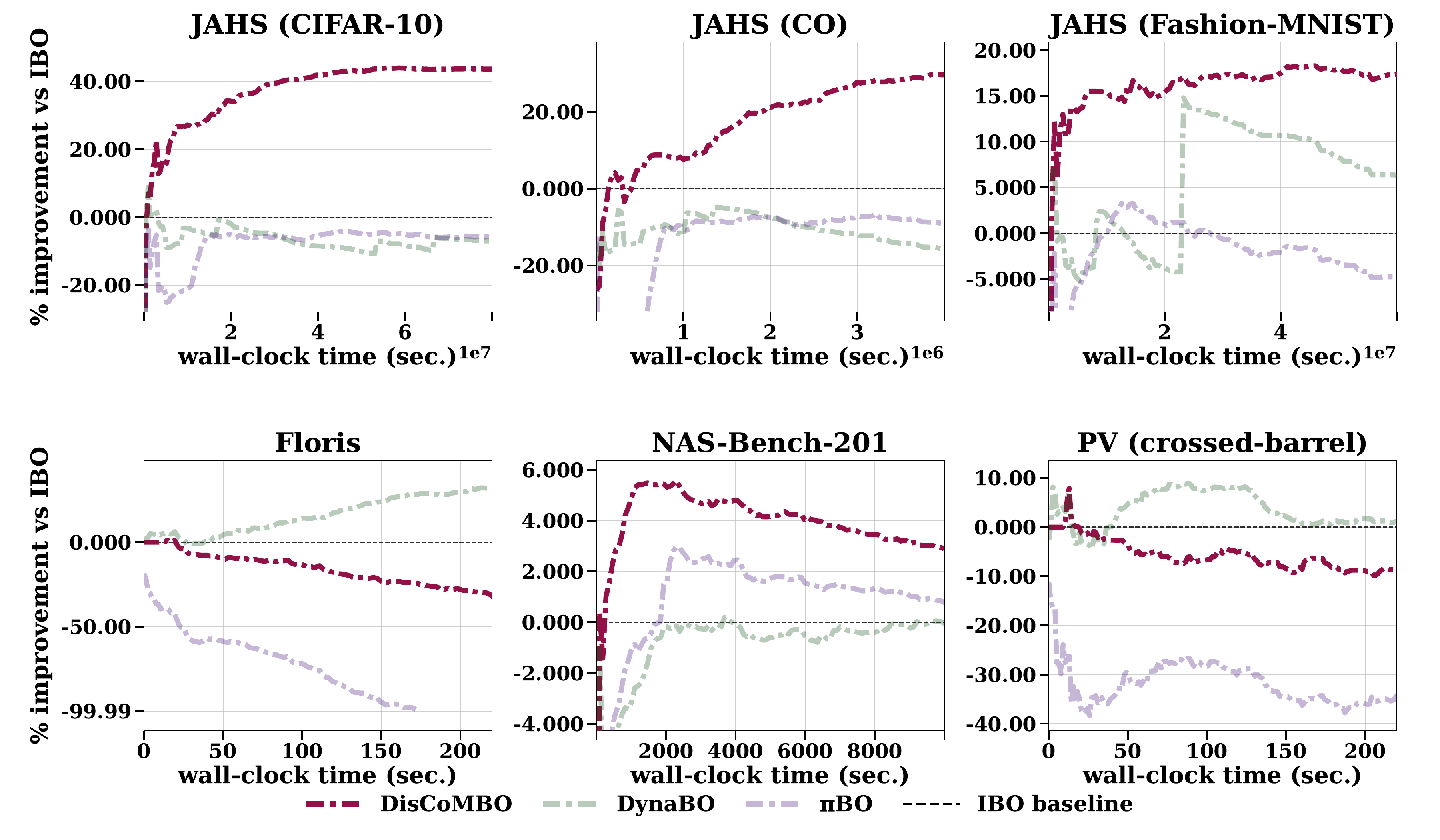}
    \caption{\textbf{Relative improvement of DisCoMBO against IBO-HPC over time with misleading prior at 5 iterations.} We observe that DisCoMBO (\dashdottedline{DisCoMBO}) consistently achieves a significant improvement over its fully generative counterpart IBO-HPC or matches its performance on \textbf{4/6} tasks when misleading priors are provided early during optimization, demonstrating that DisCoMBO effectively recovers from misleading priors.}
    \label{fig:full_diff2pc_recover}
\end{figure}
\begin{figure}
    \centering
    \includegraphics[width=1.0\linewidth]{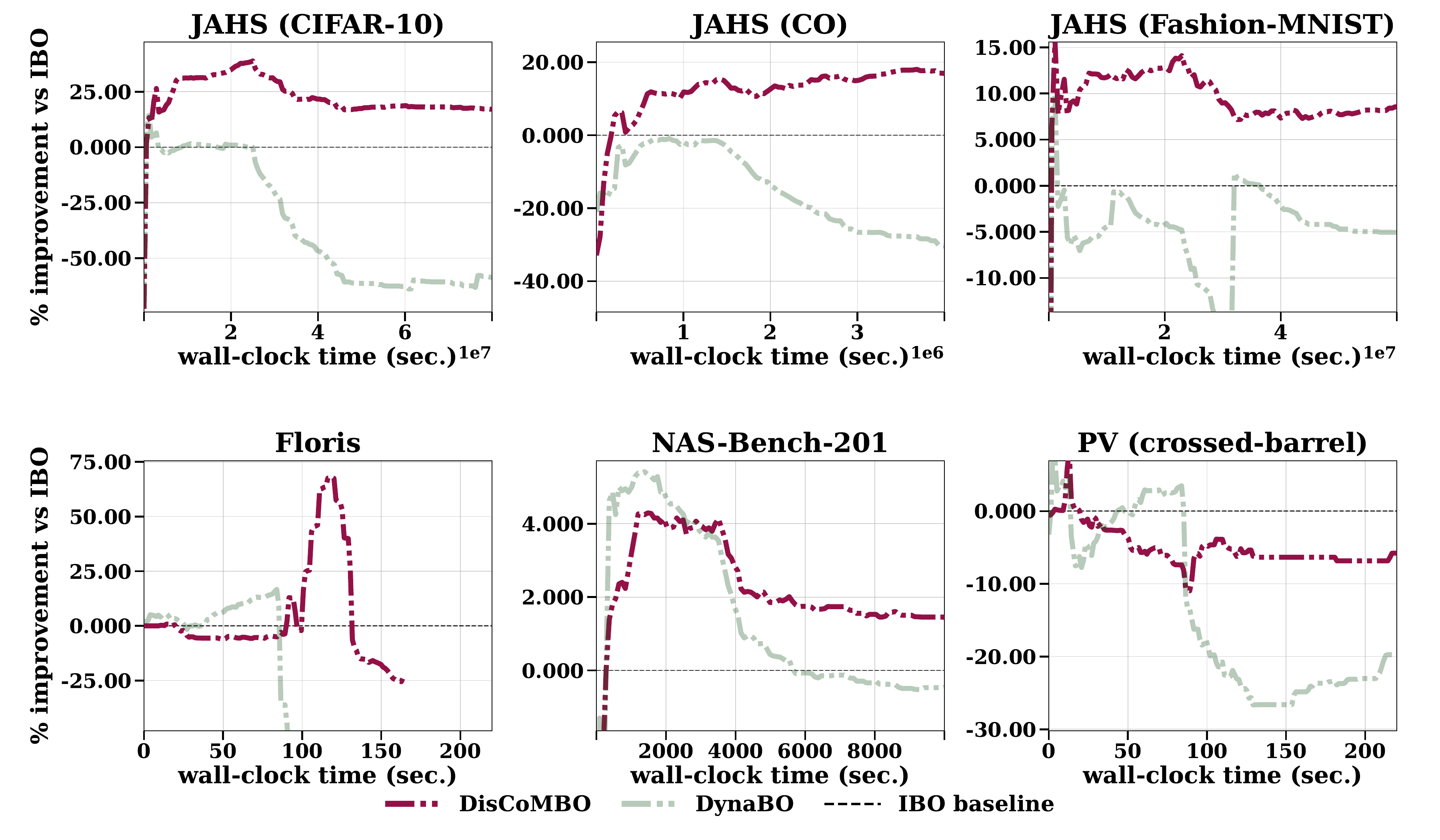}
    \caption{\textbf{Relative improvement of DisCoMBO against IBO-HPC over time with a sequence of priors given.} We observe that DisCoMBO (\dashdottedline{DisCoMBO}) consistently achieves a significant improvement over its fully generative counterpart IBO-HPC or matches its performance on \textbf{4/6} tasks when a sequence of priors with contradicting information is provided during optimization, demonstrating that DisCoMBO effectively handles multiple priors.}
    \label{fig:full_diff2pc_many}
\end{figure}

\begin{figure}
    \centering
    \includegraphics[width=1.0\linewidth]{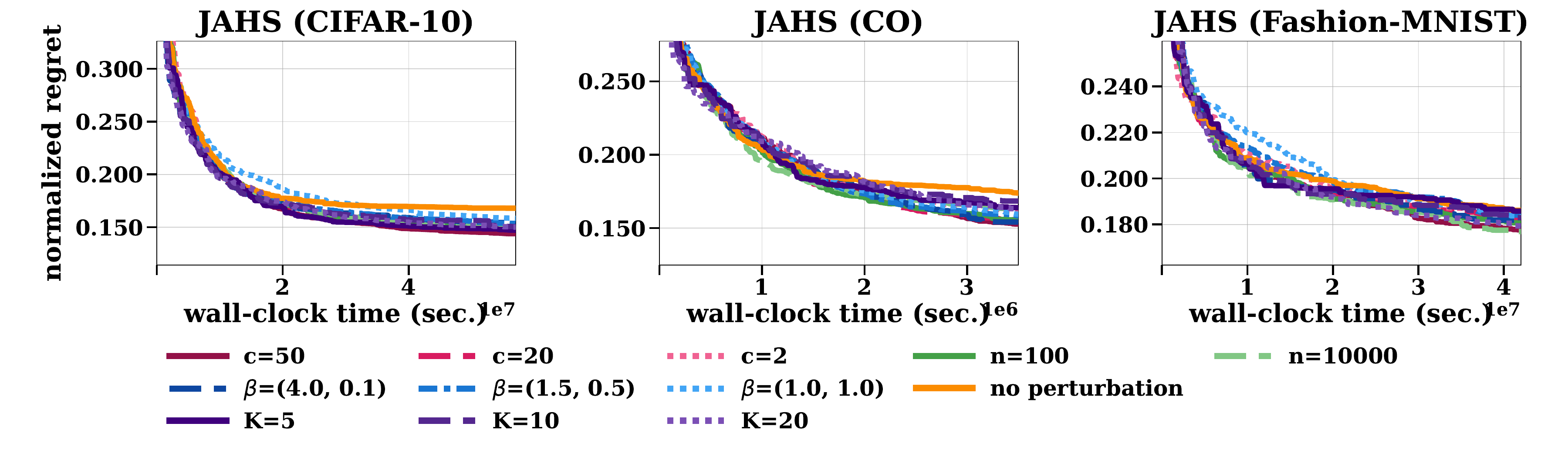}
    \caption{\textbf{DisCoMBO is robust across different hyperparameter settings.} It can be seen that DisCoMBO achieves similar performances regardless of how hyperparameters $c$, $n$, $\beta$, and $k$ are set. This demonstrates that no careful tuning of DisCoMBO's hyperparameters is necessary to achieve strong performance.}
    \label{fig:ablation_jahs}
\end{figure}

\begin{figure}
    \centering
    \includegraphics[width=1.0\linewidth]{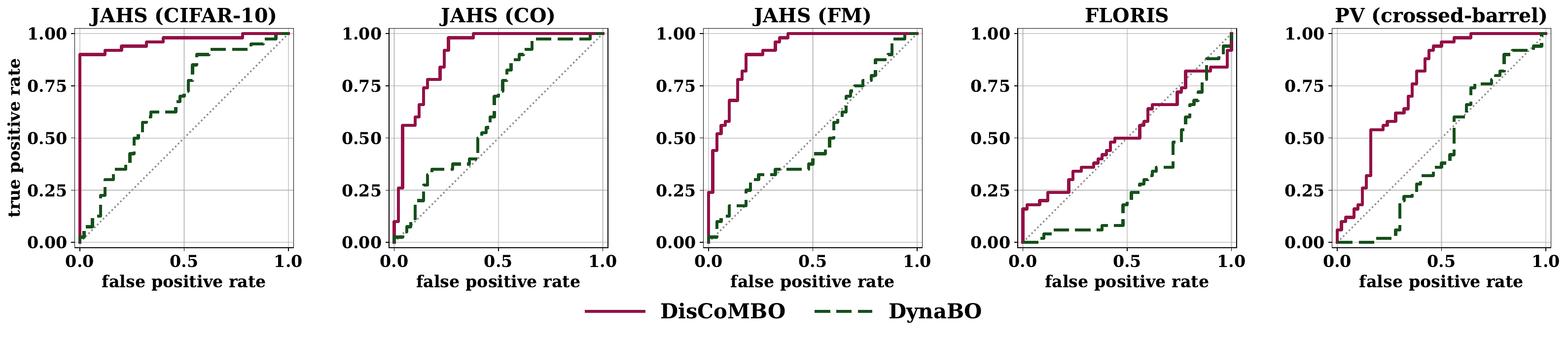}
    \caption{\textbf{DisCoMBO can discriminate between beneficial and adversarial priors.} On \textbf{4/5} datasets, our method (\solidline{DisCoMBO}) is able to identify most of the 10 good and 10 bad interventions provided to the model correctly as such. On \textbf{5/5}, DisCoMBO directly outperforms DynaBO (\dashedline{dynabo}), indicating it is well-suited for human-in-the-loop optimization scenarios even without extensive domain knowledge. }
    \label{fig:mis_prior_detection}
\end{figure}

\begin{table}[t]
\centering
\caption{\textbf{DisCoMBO is effective with broader priors.} When priors are not as specific as in our main results, DisCoMBO still clearly benefits from provided priors in terms of average speed-up. A value $> 1.0$ means the algorithm was $(x - 1.0) \cdot 100$\% faster compared to DisCoMBO with no prior provided.}
\label{tab:speedup_results}
\begin{tabular}{lcc}
\toprule
\textbf{Benchmark/task} & \textbf{DisCoMBO (no interaction)} & \textbf{DisCoMBO (broad prior)} \\
\midrule
JAHS/cifar10               & 1.0000 & 2.1603 \\
JAHS/colorectal\_histology & 1.0000 & 1.9342 \\
JAHS/fashion\_mnist        & 1.0000 & 1.7872 \\
HPO-B/6767:31              & 1.0000 & 2.7492 \\
HPO-B/6767:9914            & 1.0000 & 2.1132 \\
HPO-B/6794:3               & 1.0000 & 1.7865 \\
HPO-B/6794:31              & 1.0000 & 1.8305 \\
HPO-B/6794:9914            & 1.0000 & 1.1544 \\
HPO-B/6794:10101           & 1.0000 & 1.5094 \\
\bottomrule
\end{tabular}
\end{table}

\begin{table}[htbp]
\centering
\caption{\textbf{DisCoMBO outperforms GP+EI in 6/8 tasks.} Besides HEBO, we also compare the final regret of GP+EI against DisCoMBO. It can be seen that DisCoMBO outperforms GP+EI in 6/8 cases, demonstrating its effectiveness on various tasks. NASBench201 was excluded due to time constraints.}
\label{tab:gp_ei_discombo}
\vspace{0.5em}
\begin{tabular}{l c c}
\toprule
\textbf{Task} & \textbf{DisCoMBO} & \textbf{GP+EI} \\
\midrule
JAHS/cifar10 & \textbf{0.0845} $\pm$ 0.0040 & 0.1341 $\pm$ 0.0228 \\
JAHS/colorectal\_histology & \textbf{0.0511} $\pm$ 0.0050 & 0.0633 $\pm$ 0.0088 \\
JAHS/fashion\_mnist & 0.0526 $\pm$ 0.0021 & \textbf{0.0523} $\pm$ 0.0017 \\
\addlinespace
% HPO-B/6767:31 & \textbf{0.0475} $\pm$ 0.0197 & 0.0657 $\pm$ 0.0011 \\
HPO-B/6767:9914 & \textbf{0.0364} $\pm$ 0.0101 & 0.0464 $\pm$ 0.0044 \\
HPO-B/6794:3 & 0.0490 $\pm$ 0.0031 & \textbf{0.0436} $\pm$ 0.0000 \\
% HPO-B/6794:31 & 0.0082 $\pm$ 0.0007 & \textbf{0.0075} $\pm$ 0.0006 \\
HPO-B/6794:9914 & \textbf{0.0148} $\pm$ 0.0001 & 0.0149 $\pm$ 0.0000 \\
% HPO-B/6794:10101 & \textbf{0.2058} $\pm$ 0.0005 & 0.2059 $\pm$ 0.0003 \\
\addlinespace
PV/Crossed Barrel & \textbf{1.4159} $\pm$ 0.1622 & 1.4173 $\pm$ 0.1283 \\
Floris & \textbf{2.0565} $\pm$ 0.5485 & 2.5907 $\pm$ 0.3100 \\
\bottomrule
\end{tabular}
\end{table}

\begin{table}[htbp]
\centering
\caption{\textbf{DisCoMBO scheduling performs similar to UCB-style acquisition.} When incorporating DisCo scores as uncertainty estimate in a UCB-style construction of an acquisition function, we observed marginal differences to our proposed scheduling mechanism. }
\label{tab:ucb_vs_scheduling}
\vspace{0.5em}
\begin{tabular}{l c c c c}
\toprule
 & \textbf{JAHS/CO} & \textbf{NAS201} & \textbf{HPO-B (6794:9914)} & \textbf{PV} \\
\midrule
 norm. regret (UCB vs. Scheduling)& +0.27\% & +1.05\% & --0.21\% & --2.01\% \\
\bottomrule
\end{tabular}
\end{table}

\clearpage
\subsection{Hyperparameters of DisCoMBO}
\label{app:hp_ibo}
DisCoMBO introduces a set of hyperparameters governing the annealing schedule, sample selection, and surrogate model updates. For the annealing schedule, which controls the length of exploration and exploitation phases as well as the skew of bucket definitions, we used the periodic function$$s(t) = \frac{\beta_0+\beta_1}{2} + \frac{\beta_0-\beta_1}{2} \cos\left(\frac{\pi t}{c}\right)$$with $\beta_0 = 3$ and $\beta_1 = 0.3$, and a cycle length of $c = 20$. At each iteration, we draw $N = 1000$ candidate samples from the conditional distribution $p(\mathbf{x}\vert{}y=y^*)$ and select $K$ samples, one for each out of the $B = K$ bins. Following \citet{seng2025ihpo}, drawn samples are perturbed by $5\%$. The size of the initial dataset $J$ is set equal to $K$. Specifically, we set $K = J = 20$ for JAHS and Nas201; $K = J = 10$ for HPO-B; and $K = J = 5$ for Floris and PV. For the surrogate model updates, we fix the interval during which the surrogate remains unchanged to $L = 5$ iterations, and set the decay factor to $\gamma = 0.9$. The underlying Probabilistic Circuit (PC) surrogates and their corresponding structure learning algorithm also require specific hyperparameter choices. Structure learning is driven by K-means clustering and Randomized Dependence Coefficient (RDC) independence tests, where we set the RDC threshold for independence detection to $0.3$. To handle varying data regimes during the optimization trajectory, the minimum number of instances per leaf node is adapted dynamically based on the total number of configurations evaluated up to the current iteration. To ensure a fair comparison across all evaluated methods, we fix the maximum optimization budget to $50$, $100$, or $250$ iterations, depending on the benchmark task. For all baseline methods, we set identical hyperparameters if applicable, or set the total number of evaluations equal to the budget of DisCoMBO. Finally, we repeated each experiment across 50 different seeds to ensure statistical significance. Sensitivity analyses and ablations evaluating alternative hyperparameter configurations are provided in App.~\ref{app:further_results}.

\subsection{Hardware and Computational Cost}
\label{app:hw}
Although our experimental setup does not require much compute, we ran all our experiments on DGX-A100 machines to speed up evaluation. We used 10 CPUs for each run, thus parallelizing some sub-routines (e.g., learning of PCs). We did not use any GPUs as we queried the benchmarks employed to provide the performance of configurations. The JAHS benchmark requires a relatively large RAM ($>16GB$) to run smoothly as it loads large ensemble models. 

\subsection{Used Assets}
\label{app:used_assets}
In our experimental evaluation, we used the official implementations and assets for JAHS~\citep{jahsbench201}, NAS-Bench-201~\citep{nasbench201}, HPO-B~\citep{arango2021hpoblargescalereproduciblebenchmark}, and PV (crossed-barrel)~\citep{gongora2020crossedbarrel}. All of these are open source and open for free use under the license CC BY 4.0.
The Floris (see \url{https://github.com/NatLabRockies/floris}) simulator is licensed under the BSD 3-Clause License and is free to use.

\subsection{Released Assets}
\label{app:released_assets}
Besides code, we also release two additional surrogate benchmarks for PV (crossed-barrel) and Floris. 

\paragraph{PV} We used the data provided by~\citet{gongora2020crossedbarrel} and trained a surrogate model on that data to simulate a non-discretized version of the benchmark. Following best practices of surrogate benchmarking, we trained a Gradient Boosting Tree model with 300 trees. To ensure model fitting quality, we applied 5-fold cross-validation. We provide the code for model training in our repository and provide the learned models in \url{https://doi.org/10.6084/m9.figshare.32183937}.

\paragraph{Floris} Floris is a simulation-based benchmark. To decouple our evaluation setup from the simulator, we evaluated $10,000$ random configurations using the simulator provided at \url{https://github.com/NatLabRockies/floris} with default settings. Based on the obtained data, we trained a Gradient Boosting Tree model with 300 trees and applied 5-fold cross-validation for quality assurance. We provide the code for model training in our repository and provide the trained models at \url{https://doi.org/10.6084/m9.figshare.32183937}.

\section{On Fisher's method in DisCo}
\label{app:fisher_discussion}

To quantify the degree to which an observation $\textbf{x}$ conforms to our learned distribution $p$, our score should fulfill the following desiderata: First, it should be bounded in $[0, 1]$, where $1$ indicates conformance, related to $\mathbf{x}$ being close to a mode of $p$, and $0$ indicates surprise, related to $\mathbf{x}$ being truly outside of the critical mass of $p$. Second, surprise in one dimension should translate into an aggregated value representing this surprise accurately. If $\mathbf{x}$ is highly unusual in any dimension, we want to treat it as being nonconformant w.r.t. $p$. Third, it should be symmetric and monotone to enable coherent ranking of observations. Lastly, it should be computationally efficient, i.e., provide linear runtime inference when computing the score in a probabilistic circuit. We found that Fisher's method \citep{fisher1932statisticalmethods}, originally designed to aggregate evidence in hypothesis testing, provides our desired properties. Since its result stems from a cumulative distribution function, it is guaranteed to be bound in $[0, 1]$. Moreover, its results can be recursively combined across arbitrary depths while retaining scale-invariance. Importantly, surprise found in one dimension of $\mathbf{x}$ is correctly reflected in the aggregated value opposed to e.g. mean-based aggregates, i.e., among a growing number of dimensions that are conformant, the non-conformant magnitude determines the overall reported value. Numeric examples are provided in Tab.~\ref{tab:fishers_comparison}. Since Fisher's method is also symmetric and monotone, it allows for the order-isomorphism property to symmetric kernel-based approaches of our method. Lastly, the evaluation of the $\chi^2$ cumulative distribution function is tractable, preserving the tractability of computing DisCo in PCs. We emphasize that the resulting scores can not be treated as the result of a hypothesis test and have no formal implications in the sense of testing methodology. However, our empirical evaluations revealed the consistency of the score to quantify the conformance of observations to our learned distribution on diverse benchmarks, which is also shown in Fig.~\ref{fig:roc_misleading}.

\begin{table}[]
    \centering
    \caption{Fisher's method ranks cases similarly to raw multiplication (a–c, h, i), but accounts for the number of observations. In g), nine ones and one $10^{-2}$ give the same product as d), yet Fisher's method reports $p \approx 0.98$ instead of $\approx 0.06$ — a single weak signal among nine nulls is correctly judged uninformative. The arithmetic mean stays above $0.1$ even when nine of ten observations are tiny (h, i), washing out localized signal.}
    \begin{tabular}{llrrr}
        \toprule
         & \multicolumn{1}{c}{input} & \multicolumn{1}{c}{multiplication} & \multicolumn{1}{c}{arithmetic mean} & \multicolumn{1}{c}{Fisher $p$} \\
        \midrule
        a) & (1e-10, 1e-10) & $1.00e-20$ & $1.00e-10$ & $4.71e-19$ \\
        b) & (1e-10, 0.01) & $1.00e-12$ & $5.00e-3$ & $2.86e-11$ \\
        c) & (1e-10, 1) & $1.00e-10$ & $5.00e-1$ & $2.40e-9$ \\
        d) & (0.01, 1) & $1.00e-2$ & $5.05e-1$ & $5.61e-2$ \\
        e) & (1, 1) & $1.00e0$ & $1.00e0$ & $1.00e0$ \\
        f) & 1$\times$1e-10, 9$\times$1 & $1.00e-10$ & $9.00e-1$ & $7.93e-4$ \\
        g) & 1$\times$0.01, 9$\times$1 & $1.00e-2$ & $9.01e-1$ & $9.80e-1$ \\
        h) & 9$\times$1e-10, 1$\times$1 & $1.00e-90$ & $1.00e-1$ & $2.03e-75$ \\
        i) & 9$\times$0.01, 1$\times$1 & $1.00e-18$ & $1.09e-1$ & $1.26e-9$ \\
        \bottomrule
        \vspace{0.2cm}
    \end{tabular}
    \label{tab:fishers_comparison}
\end{table}

\section{Broader Impact}
\label{app:broader_impact}
As a general-purpose black-box optimization framework, DisCoMBO carries both positive and negative downstream societal and environmental implications depending on its domain of application.

\paragraph{Positive Impacts} The primary utility of DisCoMBO lies in its data efficiency when optimizing complex, mixed-variable systems under expert prior knowledge. By accelerating workflows in scientific discovery and renewable energy—such as optimizing wind farm layouts to maximize power generation or streamlining materials discovery—the framework can directly contribute to resource and cost reductions. Furthermore, by matching or outperforming traditional SMBO baselines with fewer function evaluations, DisCoMBO can reduce the overall computational footprint and associated carbon emissions of intensive hyperparameter tuning and neural architecture search.

\paragraph{Negative Impacts \& Dual-Use Risks} Conversely, as an application-agnostic tool, DisCoMBO introduces dual-use risks. The framework's optimization efficiency could be exploited to enhance the performance of objectionable technologies, such as mass surveillance software or automated weapon systems. Additionally, because DisCoMBO natively integrates external knowledge, malicious actors could deliberately inject biased or adversarial priors to manipulate optimization trajectories toward harmful or discriminatory outcomes in sensitive socio-technical deployments.

\paragraph{Mitigation} We mitigate these risks by releasing our framework under open-source licenses that encourage transparent peer review. We urge practitioners deploying DisCoMBO in high-stakes domains to strictly audit the alignment and validity of external priors before integration.

\clearpage

\section*{NeurIPS Paper Checklist}

\begin{enumerate}

\item {\bf Claims}
    \item[] Question: Do the main claims made in the abstract and introduction accurately reflect the paper's contributions and scope?
    \item[] Answer: \answerYes{} % Replace by \answerYes{}, \answerNo{}, or \answerNA{}.
    \item[] Justification: We provide proofs of our theoretical claims (Appendix \ref{app:proofs}) and provide an extensive empirical evaluation (Section \ref{sec:Experiments} and Appendix \ref{app:further_results}), confirming our claims.

\item {\bf Limitations}
    \item[] Question: Does the paper discuss the limitations of the work performed by the authors?
    \item[] Answer: \answerYes{} % Replace by \answerYes{}, \answerNo{}, or \answerNA{}.
    \item[] Justification: We provide a dedicated paragraph discussing limitations in Section \ref{sec:conclusion}.

\item {\bf Theory assumptions and proofs}
    \item[] Question: For each theoretical result, does the paper provide the full set of assumptions and a complete (and correct) proof?
    \item[] Answer: \answerYes{} % Replace by \answerYes{}, \answerNo{}, or \answerNA{}.
    \item[] Justification: For all propositions, proofs are provided in Appendix \ref{app:proofs}.

    \item {\bf Experimental result reproducibility}
    \item[] Question: Does the paper fully disclose all the information needed to reproduce the main experimental results of the paper to the extent that it affects the main claims and/or conclusions of the paper (regardless of whether the code and data are provided or not)?
    \item[] Answer: \answerYes{} % Replace by \answerYes{}, \answerNo{}, or \answerNA{}.
    \item[] Justification: We provide pseudo-code (Section \ref{sec:discombo}), key hyperparameters that were set (Appendix \ref{app:hp_ibo}), ablation analysis on hyperparameters (Appendix \ref{app:further_results}), and the hardware specification we used to run the experiments on (Appendix \ref{app:hw}). Further, we provide all benchmarks via a figshare link (\url{https://doi.org/10.6084/m9.figshare.32183937}) (Appendix \ref{app:experiments}).

\item {\bf Open access to data and code}
    \item[] Question: Does the paper provide open access to the data and code, with sufficient instructions to faithfully reproduce the main experimental results, as described in supplemental material?
    \item[] Answer: \answerYes{} % Replace by \answerYes{}, \answerNo{}, or \answerNA{}.
    \item[] Justification: We provide an git repository at \url{https://github.com/bewit/DisCoMBO} including instructions on how to build the Docker container and how to run experiments. All data needed for our experiments (e.g., benchmarks used) are shared via a figshare URL (\url{https://doi.org/10.6084/m9.figshare.32183937}, see also Appendix \ref{app:experiments}).

\item {\bf Experimental setting/details}
    \item[] Question: Does the paper specify all the training and test details (e.g., data splits, hyperparameters, how they were chosen, type of optimizer) necessary to understand the results?
    \item[] Answer: \answerYes{} % Replace by \answerYes{}, \answerNo{}, or \answerNA{}.
    \item[] Justification: All hyperparameters we have set are described in Appendix \ref{app:experiments}. Also, we provide ablation studies showing the effect of varying these hyperparameters (Appendix \ref{app:further_results}).

\item {\bf Experiment statistical significance}
    \item[] Question: Does the paper report error bars suitably and correctly defined or other appropriate information about the statistical significance of the experiments?
    \item[] Answer: \answerYes{} % Replace by \answerYes{}, \answerNo{}, or \answerNA{}.
    \item[] Justification: We provide critical difference plots to compare our method with our baselines aggregated over all benchmarks. This naturally comes with statistical significance information (Wilcoxon signed rank test used with $p=0.05$). Also, all our regret plots are plotted with standard error (see Section \ref{sec:Experiments}).

\item {\bf Experiments compute resources}
    \item[] Question: For each experiment, does the paper provide sufficient information on the computer resources (type of compute workers, memory, time of execution) needed to reproduce the experiments?
    \item[] Answer: \answerYes{} % Replace by \answerYes{}, \answerNo{}, or \answerNA{}.
    \item[] Justification: We provide a description of the hardware resources required to run our experiments in Appendix \ref{app:hw}.
    
\item {\bf Code of ethics}
    \item[] Question: Does the research conducted in the paper conform, in every respect, with the NeurIPS Code of Ethics \url{https://neurips.cc/public/EthicsGuidelines}?
    \item[] Answer: \answerYes{} % Replace by \answerYes{}, \answerNo{}, or \answerNA{}.
    \item[] Justification: We have carefully read the code of ethics, and we do not think that our contribution exhibits risks beyond the risks of general AI and ML applications.

\item {\bf Broader impacts}
    \item[] Question: Does the paper discuss both potential positive societal impacts and negative societal impacts of the work performed?
    \item[] Answer: \answerYes{} % Replace by \answerYes{}, \answerNo{}, or \answerNA{}.
    \item[] Justification: We provide a discussion on the broader impact of our research in Appendix \ref{app:broader_impact}.
    
\item {\bf Safeguards}
    \item[] Question: Does the paper describe safeguards that have been put in place for responsible release of data or models that have a high risk for misuse (e.g., pre-trained language models, image generators, or scraped datasets)?
    \item[] Answer: \answerNA{} % Replace by \answerYes{}, \answerNo{}, or \answerNA{}.
    \item[] Justification: We do not publish artifacts beyond surrogate benchmarks of black-box optimization tasks.

\item {\bf Licenses for existing assets}
    \item[] Question: Are the creators or original owners of assets (e.g., code, data, models), used in the paper, properly credited and are the license and terms of use explicitly mentioned and properly respected?
    \item[] Answer: \answerYes{} % Replace by \answerYes{}, \answerNo{}, or \answerNA{}.
    \item[] Justification: We cite the papers associated with each asset properly. We provide license information and URLs to the assets.

\item {\bf New assets}
    \item[] Question: Are new assets introduced in the paper well documented and is the documentation provided alongside the assets?
    \item[] Answer: \answerYes{} % Replace by \answerYes{}, \answerNo{}, or \answerNA{}.
    \item[] Justification: All new assets released are described in Appendix \ref{app:released_assets}.

\item {\bf Crowdsourcing and research with human subjects}
    \item[] Question: For crowdsourcing experiments and research with human subjects, does the paper include the full text of instructions given to participants and screenshots, if applicable, as well as details about compensation (if any)? 
    \item[] Answer: \answerNA{} % Replace by \answerYes{}, \answerNo{}, or \answerNA{}.
    \item[] Justification: Our research does not include research with human subjects.

\item {\bf Institutional review board (IRB) approvals or equivalent for research with human subjects}
    \item[] Question: Does the paper describe potential risks incurred by study participants, whether such risks were disclosed to the subjects, and whether Institutional Review Board (IRB) approvals (or an equivalent approval/review based on the requirements of your country or institution) were obtained?
    \item[] Answer: \answerNA{} % Replace by \answerYes{}, \answerNo{}, or \answerNA{}.
    \item[] Justification: See above (14.).

\item {\bf Declaration of LLM usage}
    \item[] Question: Does the paper describe the usage of LLMs if it is an important, original, or non-standard component of the core methods in this research? Note that if the LLM is used only for writing, editing, or formatting purposes and does \emph{not} impact the core methodology, scientific rigor, or originality of the research, declaration is not required.
    %this research? 
    \item[] Answer: \answerNA{} % Replace by \answerYes{}, \answerNo{}, or \answerNA{}.
    \item[] Justification: LLMs were only used for writing purposes.

\end{enumerate}

\end{document}